\pdfoutput=1

\def\MODE{1} 
\if\MODE1
\newcommand{\COND}[2]{#1}
\else
\newcommand{\COND}[2]{#2}
\fi

\documentclass[10pt]{article}

 \usepackage{palatcm}
\usepackage{graphicx} 
\usepackage{subfigure} 
\usepackage{soul}
\usepackage{color, xcolor}
\usepackage[thinlines]{easytable}
\usepackage{relsize}
\usepackage{xfrac}
\usepackage{verbatim}
\usepackage{algorithm}
\usepackage[noend]{algorithmic}
\usepackage{amsmath}
\usepackage{amssymb}
\usepackage{amsthm}
\usepackage{epsfig}
\usepackage{wrapfig}
\usepackage{url}
\usepackage[colorlinks=true,citecolor=blue,linkcolor=blue]{hyperref}
\usepackage{multirow}
\usepackage{fullpage}

\newtheorem{theo}{Theorem}

\newtheorem{dfn}{Definition}
\newtheorem{lem}{Lemma}
\newtheorem{cor}{Corollary}
\newtheorem{rem}{Remark}

\newtheorem{example}{Example}

\newcommand{\E}{ \mathbb{E} }

\newcommand{\vecx}{{\bf x}}
\newcommand{\vecy}{{\bf y}}
\newcommand{\vecg}{{\bf g}}

\usepackage[T1]{fontenc}
\usepackage{upgreek}

\newcommand{\cyc}{{\textsc{Cyclades}}}
\newcommand{\HW}{\textsc{Hogwild!}}
\newcommand{\hog}{\textsc{Hogwild!}}
\newcommand{\setS}{\mathcal{S}}

\newcommand{\setD}{\mathcal{D}}
\newcommand{\setC}{\mathcal{C}}

\usepackage[greek,english]{babel}
\usepackage[utf8x]{inputenc}
\usepackage{pifont}
\usepackage{float}

\newcommand{\WH}[1]{}
\newcommand{\ST}[1]{}
\newcommand{\XP}[1]{}

\title{ \cyc{}: Conflict-free  Asynchronous Machine Learning}

\author{
Xinghao Pan$^{\alpha,\epsilon}$, Maximilian Lam$^{\epsilon}$, Stephen Tu$^{\alpha,\epsilon}$\\
Dimitris Papailiopoulos$^{\alpha,\epsilon}$, Ce Zhang$^{s}$\\
Michael I. Jordan$^{\alpha,\epsilon, \sigma}$, Kannan Ramchandran$^{\epsilon}$, Chris Re$^{s}$, Benjamin Recht$^{\alpha,\epsilon, \sigma}$\\
$^{\alpha}$AMPLab, $^{\epsilon}$EECS at UC Berkeley, $^{\sigma}$Statistics at UC Berkeley \\
$^{s}$CS at Stanford University
} 

\date{\today}

\begin{document}

\maketitle

\begin{abstract}
We present \cyc{}, a general framework for parallelizing stochastic
optimization algorithms in a shared memory setting.  \cyc{} is asynchronous during shared model updates, and
requires no memory locking mechanisms, similar to \hog{}-type algorithms.
Unlike \hog{}, \cyc{} introduces no conflicts during the parallel execution, and
offers a black-box analysis for provable speedups across a large family of
algorithms.  Due to its inherent conflict-free nature and cache locality,  our
multi-core implementation of \cyc{} consistently outperforms \hog{}-type
algorithms on sufficiently sparse datasets, leading to up to $40\%$ speedup gains
compared to the \hog{} implementation of SGD, and up to $5\times$ gains over asynchronous implementations
of variance reduction algorithms.
\end{abstract}

\section{Introduction}

Following the seminal work of \hog{} \cite{niu2011hogwild},  many studies have demonstrated that near-linear speedups are achievable on a variety of machine learning tasks via asynchronous, lock-free implementations
\cite{recht2012factoring,
zhuang2013fast,
yun2013nomad,
liu2014asynchronous1,
duchi2013estimation,
wang2014asynchronous,
hsieh2015passcode,
mania2015perturbed}.
In all of these studies, classic algorithms are parallelized by simply running parallel and asynchronous model updates without locks.  
These lock-free, asynchronous algorithms exhibit speedups even when applied to large, non-convex problems, as demonstrated by deep learning systems such as Google's Downpour SGD~\cite{dean2012large}  and Microsoft's Project Adam~\cite{chilimbi2014project}.   

While these techniques have been remarkably successful, many of the papers cited above require delicate and tailored analyses to understand the benefits of asynchrony for the particular learning task at hand.  Moreover, in non-convex settings, we currently have little quantitative insight into how much speedup is gained from asynchrony and how much accuracy may be lost.

In this work, we present~\cyc{}, a general framework for lock-free, asynchronous machine learning algorithms that obviates the need for specialized analyses. \cyc{} runs asynchronously and \emph{maintains serial equivalence}, i.e., it produces the same outcome as the serial algorithm.  
Since it returns exactly the same output as a serial implementation, any algorithm parallelized by our framework inherits the correctness proof of the serial counterpart without modifications.  Additionally, if a particular heuristic serial algorithm is popular, but does not have a rigorous analysis, such as backpropagation on neural networks, \cyc{} still guarantees that its execution will return a serially equivalent output.

\cyc{} achieves serial equivalence by partitioning updates among cores,
in a way that ensures that there are no conflicts across partitions.  Such a partition can always
be found efficiently by leveraging a powerful result on graph phase
transitions~\cite{krivelevich2016phase}.  When applied to our setting, this
result guarantees that a sufficiently small sample of updates will have only a
\emph{logarithmic} number of conflicts.  This allows us to  evenly partition model
updates across cores, with the guarantee that all conflicts are localized within
each core.  
Given enough problem sparsity, \cyc{} guarantees a nearly linear speedup, while
inheriting all the qualitative properties of the serial counterpart of the algorithm, e.g., proofs for rates of convergence.
Enforcing a serially equivalent execution in \cyc{} comes with additional practical benefits.
Serial equivalence is helpful for hyperparameter tuning, or locating the best model produced by the asynchronous execution, since experiments are reproducible, and solutions are easily verifiable.
Moreover, a \cyc{} program is easy to debug, because bugs are repeatable and we can examine the step-wise execution to localize them.

A significant benefit of the update partitioning in \cyc{} is that it induces
considerable access locality compared to the more unstructured nature of the memory accesses during \HW{}. Cores will access the same data points and
read/write the same subset of model variables.  This has the additional benefit
of reducing false sharing across cores.
Because of these gains, \cyc{} can actually \emph{outperform}
\hog{} in practice on sufficiently sparse problems, despite appearing to require more computational overhead.  
Remarkably, because of the added locality, even a single threaded implementation of \cyc{} can actually be faster than serial SGD. 
In our SGD experiments for matrix completion and word embedding problems, \cyc{} can offer a speedup gain of up to $40\%$ compared to that of \hog{}.
Furthermore, for variance reduction techniques such as SAGA~\cite{defazio2014saga} and SVRG~\cite{johnson2013accelerating},~\cyc{} yields
better accuracy and more significant speedups, with up to $5\times$ performance gains over \hog{}-type implementations.

The remainder of our paper is organized as follows.
Section \ref{sec:prelim} establishes some preliminaries.
Details and theory of \cyc{} are presented in Section \ref{sec:cyc}.
We present our experiments in Section \ref{sec:expts}, we discuss related work in Section \ref{sec:prior},
and then conclude with Section \ref{sec:conclusion}.

\section{The Algorithmic Family of Stochastic-Updates}
\label{sec:prelim}

We study parallel asynchronous iterative algorithms on the computational model used by~\cite{niu2011hogwild}, and similar to the partially asynchronous model of  \cite{bertsekas1989parallel}: a number of cores have access to the same shared memory, and each of them can read and update components of $\vecx$ in parallel from the shared memory.

In this work, we consider a large family of randomized algorithms that we will refer to as {\it Stochastic Updates} (SU).
The main algorithmic component of SU focuses on updating small subsets of a model variable $\vecx$, that lives in shared memory, according to prefixed access patterns, as sketched by Alg. \ref{alg:protoSU}.

\begin{wrapfigure}{R}{0.5\columnwidth}
\vspace{-0.3cm}
\begin{minipage}{0.5\columnwidth}
\begin{algorithm}[H]
    \caption{Stochastic Updates pseudo-algorithm}
\begin{algorithmic}[1]
\STATE Input: ${\bf x};\: f_1,\ldots, f_n;\: u_1,\ldots, u_n;\: \mathcal{D};\:T$.
\FOR {$t=1:T$}
\STATE sample $i\sim \mathcal{D}$
\STATE ${\bf x}_{\mathcal{S}_i} =   u_i({\bf x}_{\mathcal{S}_i}, f_i)$ \hfill {\color{gray}//update global model on ${\mathcal{S}_i}$}
\ENDFOR
\STATE {\bf Output:} ${\bf x}$
\end{algorithmic}
   \label{alg:protoSU}
 \end{algorithm}
\end{minipage}
\end{wrapfigure}

In Alg.~\ref{alg:protoSU} each set $\mathcal{S}_i$ is a subset of the coordinate indices of ${\bf x}$, 
each function $f_i$ only operates on the subset $\mathcal{S}_i$ of coordinates (i.e., both its domain and co-domain are inside $\mathcal{S}_i$), 
and $u_i$ is a local update function that computes a vector with support on $\mathcal{S}_i$ using as input $\vecx_{\setS_i}$ and $f_i$.
Moreover, $T$ is the total number of iterations, and $\mathcal{D}$ is the distribution  with support $\{1,\ldots, n\}$ from which we draw  $i$.
As we explain in Appendix \ref{sec:su}, several machine learning and optimization algorithms belong to the SU algorithmic family, such as stochastic gradient descent (SGD), with or without weight decay and regularization,  variance-reduced
learning algorithms like SAGA and SVRG, and even some combinatorial graph algorithms.

\paragraph{The Updates Conflict Graph}
 A useful construction for our developments is the conflict graph between updates, which can be generated from the bipartite graph between the updates and the model variables.
We define these graphs below, and provide an illustrative sketch in Fig.~\ref{fig:graph_example}.

\begin{wrapfigure}{r}{0.4\textwidth}
\vspace{-0.5cm}
\centering
\includegraphics[width=0.37\columnwidth]{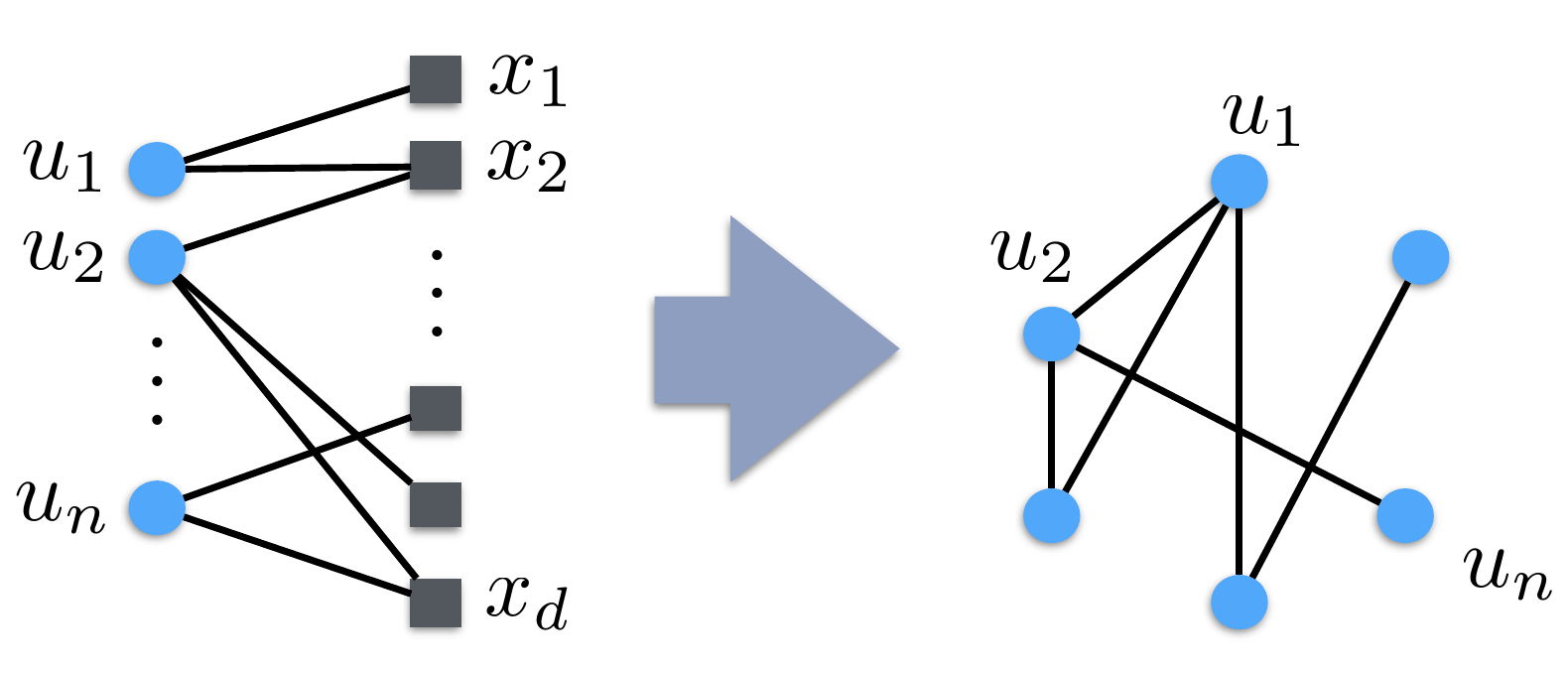}
\caption{{\small The above bipartite graph links an update $u_i$ to a variable $x_j$ when an update needs to access (read or write) the variable.
From $G_u$ we obtain the conflict graph $G_c$, whose max degree is $\Delta$.
If $G_c$ is sufficiently sparse, we expect that it is possible to parallelize updates without too many conflicts. \cyc{} exploits this intuition.}}
\label{fig:graph_example}
\vspace{-1cm}
\end{wrapfigure}

\begin{dfn}
We denote as $G_{u}$ the bipartite update-variable graph between the updates $u_1, \ldots, u_n$ and the $d$ model variables.
In $G_u$ an update $u_i$ is linked to a variable $x_j$, if $u_i$ requires to read or write $x_j$.
We let $E_{u}$ denote the number of edges in the bipartite graph, 
and also denote as $\Delta_L$ the left max vertex degree of $G_u$, and as  $\overline{\Delta}_L$ its average left degree.
\end{dfn}
\begin{dfn}
We denote by $G_c$ a conflict graph on $n$ vertices, each corresponding to an update $u_i$. 
Two vertices of $G_c$ are linked with an edge, if and only if the corresponding updates share at least one variable in the bipartite-update graph $G_u$.
We also denote as $\Delta$ the max vertex degree of $G_c$.
\end{dfn}

We stress that the conflict graph is never actually constructed, but serves as a useful concept for understanding \cyc{}.

\paragraph{Our Main Result} By exploiting the structure of the above graphs and through a light-weight and careful sampling and allocation of updates, 
\cyc{} is able to guarantee the following result for SU algorithms, which we establish in the following sections.

\begin{theo} [informal]
Let us consider an	 SU algorithm $\mathcal{A}$ defined through $n$ update rules, where the conflict max degree between the $n$ updates is $\Delta$, and the sampling distribution $\mathcal{D}$ is uniform with (or without) replacement from $\{1,\ldots,n\}$.
Moreover, assume that we wish to run $\mathcal{A}$ for $T=\Theta(n)$ iterations, and that $\frac{\Delta_L}{\overline{\Delta}_L}\le \sqrt{n}.$
Then on up to $P = \tilde{O}(\frac{n}{\Delta\cdot \Delta_L})$ cores, \cyc{} guarantees a $\widetilde{\Omega}(P)$ speedup over $\mathcal{A}$, while outputting the same solution $\vecx$ as  $\mathcal{A}$ would do after the same random set of $T$ iterations.\footnote{$\widetilde{\Omega}(\cdot)$ and $\widetilde{O}(\cdot)$ hide polylog factors.} 
\label{theo:informal}
\end{theo}

We will now provide some simple examples of how the above parameters, and guarantees translate for specific problem cases.

\begin{example} 
Many machine learning applications often seek to minimize the empirical risk
\[
\min_{\bf x} \frac{1}{n}\sum_{i=1}^n \ell_i( {\bf a}_i^T {\bf x})
\]
where ${\bf a}_i$ represents the $i$th data point, ${\bf x}$ is the model we are trying to fit, and $\ell_i$ is a loss function that tells us how good of a fit a model is with respect to data point $i$.
Several problems can be formulated in the above way, such as logistic regression, least squares, support vector machines (SVMs) for binary classification, and others.
If we attempt to solve the above problem using SGD (with or without regularization), or via variance reduction techniques like SVRG and SAGA, then (as we show in  Appendix~\ref{sec:su}) the sparsity of the $u_i$ updates is determined by the gradient of a single sampled data point $i$.
For the aforementioned problems,  this will be proportional to $\left(\frac{d}{du}\ell_i(u)\big|_{u={\bf a}_i^T{\bf x}}\right){\bf a}_i$, hence the sparsity of the update is defined by the non-zero support of datapoint ${\bf a}_i$.
In the induced bipartite update-variable graph of this problem, we have $\Delta_L = \max_i ||{\bf a}_i||_0$, and the maximum conflict degree $\Delta$ is the maximum number of data points ${\bf a}_i$ that share at least one of the $d$ features.
As a toy example, let $\frac{n}{d} = \Theta(1)$ and let the non-zero support of ${\bf a}_i$ be of size $n^\delta$ and uniformly distributed. Then, one can show that with overwhelmingly high probability $\Delta = \widetilde{O}(n^{1/2+\delta})$ and hence \cyc{} achieves an $\widetilde{\Omega}(P)$ speedup on up to $P = \widetilde{O}(n^{1/2-2\delta} )$ cores.

\end{example}

\begin{example} 
Consider the following generic minimization problem
$$\min_{{\bf x}_1,\ldots,{\bf x}_{m_1}}\min_{{\bf y}_1,\ldots,{\bf y}_{m_2}}\sum_{i=1}^{m_1}\sum_{j=1}^{m_2} \phi_{i,j}({\bf x}_i, {\bf y}_j) $$
where $\phi_{i,j}$ is a convex function of a scalar. 
The above generic formulation captures several problems like matrix completion and matrix factorization \cite{recht2013parallel} (where $ \phi_{i,j} =  (A_{i,j}-{\bf x}_i^T {\bf y}_j)^2$), word embeddings \cite{arora2015rand-walk} (where $ \phi_{i,j} =A_{i,j} (\log(A_{i,j}) - \|{\bf x}_i + {\bf x}_j\|_2^2 - C)^2$), graph $k$-way cuts \cite{niu2011hogwild} (where $\phi_{i,j}=A_{i,j}\|{\bf x}_i-{\bf x}_j\|_1$), and others.
Let $m_1=m_2=m$ for simplicity, and assume that we aim to minimize the above by sampling a single function $\phi_{i,j}$ and then updating ${\bf x}_i$ and ${\bf y}_j$ using SGD. 
Here, the number of update functions is proportional to $n = m^2$, and for the above setup each gradient update with respect to the sampled function $ \phi_{i,j}({\bf x}_i,{\bf y}_j)$ is only interacting with the variables ${\bf x}_i$ and ${\bf y}_j$, i.e., only {\it two} variable vectors out of the $2m$ many (i.e., $\Delta_L = 2$). Moreover, the previous imply a conflict degree of at most $\Delta = 2m$. 
In this case, \cyc{} can provably guarantee an $\widetilde{\Omega}(P)$ speedup for up to $P=O(m)$ cores. 
\end{example}

In our experiments we test \cyc{} on several problems including least squares, classification with logistic models, matrix factorization, and word embeddings, and several algorithms including SGD, SVRG, and SAGA.
We show that in most cases it can significantly outperform the \HW{} implementation of these algorithms, if the data is sparse.

\begin{rem}
We would like to note, that there are several cases where there might be a few outlier updates with extremely high conflict degree.
In Appendix~\ref{app:outliers}, we prove that if there are no more than $O(n^\delta)$ vertices of high conflict degree $\Delta_o$, and the rest of the vertices have max degree at most $\Delta$, then 
the result of Theorem~\ref{theo:informal} still holds in expectation.
\end{rem}

In the following section, we establish the technical results behind \cyc{} and  provide the details behind our parallelization framework.

\section{\cyc{}: Shattering Dependencies}
\label{sec:cyc}

\cyc{} consists of three computational components as shown in
Figure ~\ref{fig:cyc}.

\begin{wrapfigure}{r}{0.5\textwidth}
  \vspace{-0.3cm}
  \centering
  \includegraphics[width=0.45\textwidth]{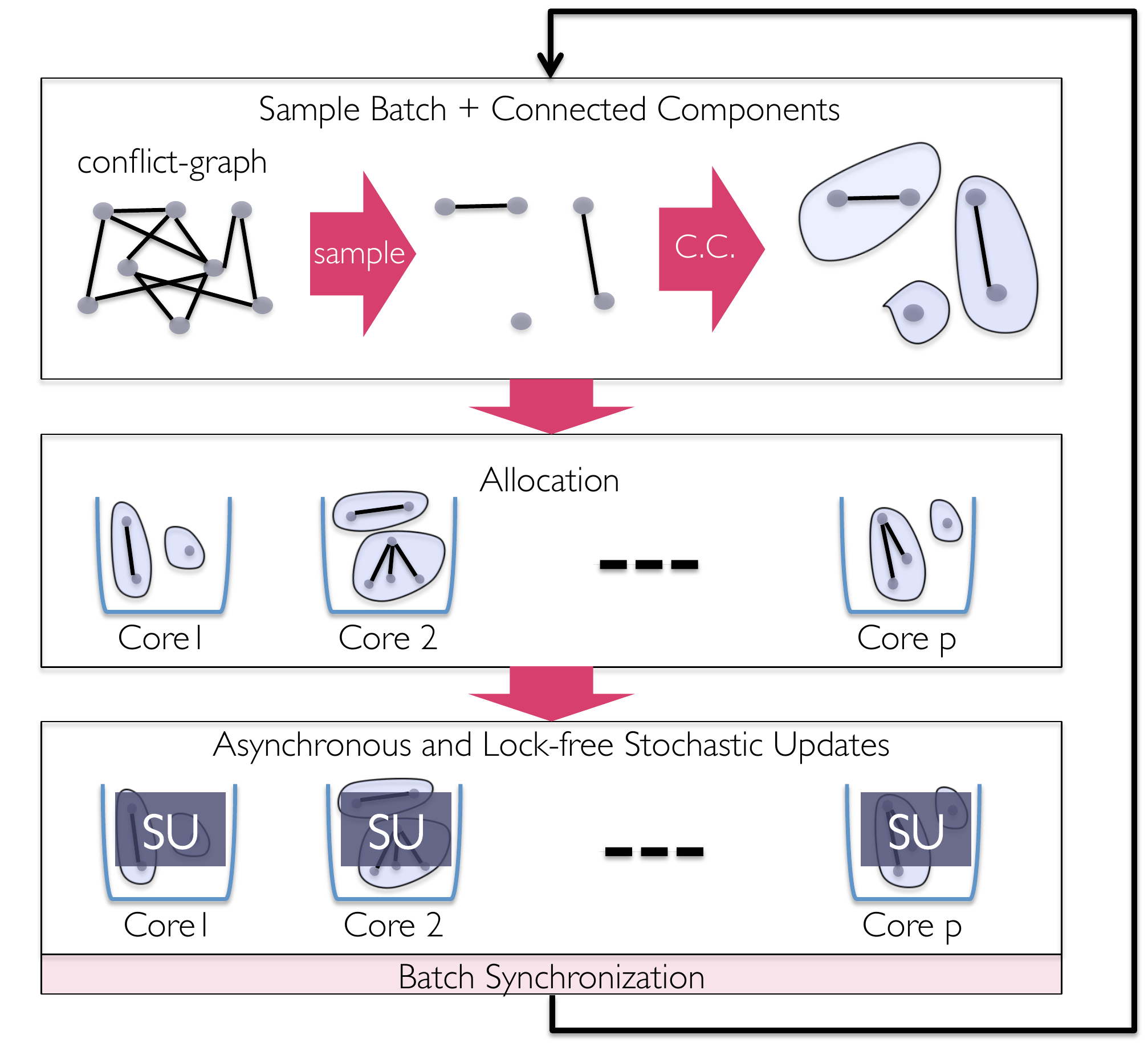}
  \caption{{\small \cyc{} carefully samples updates, then finds conflict-groups, and allocates them across cores.
Then, each core asynchronously updates the shared model, without incurring any read/write conflicts. 
This is possible by processing all the conflicting updates within the same core. After the processing of a batch is completed, the above is repeated, for as many iterations as required.}}
  \label{fig:cyc}
  \vspace{-0.7cm}
\end{wrapfigure}

It starts by sampling (according
to a distribution $\setD$) a number of $B$ updates from the graph shown in
Fig.~\ref{fig:graph_example}, and assigns a label to each of them (a processing order).
We note that in practice the sampling is done on the bipartite graph, which avoids the need to actually construct the conflict graph.
After sampling, it computes the connected
components of the sampled subgraph induced by the $B$ sampled updates, to determine the conflict groups. 
Once the conflicts groups are formed, it allocates them across $P$ cores.
Finally, each core processes locally the conflict groups of updates that it has been assigned, following the order that each
update has been labeled with.  The above process is then repeated, for as many iterations as needed.

The key component of \cyc{} is to carry out the sampling in such a way that we have as many connected components as possible, and all of them of small size, provably.
In the next subsections, we explain how each part is carried out, and provide theoretical guarantees for each of them individually, which we combine at the end of this section for our main theorem.

\paragraph{Frugal sampling shatters conflicts.}

A key technical aspect that we exploit in \cyc{} is that appropriate sampling and allocation of updates can lead to near optimal parallelization of sparse SU algorithms. To do that we expand upon the following result established in \cite{krivelevich2016phase}.
\begin{theo} 
Let $G$ be a graph on $n$ vertices, with maximum vertex degree $\Delta$.
Let us sample each vertex independently with probability 
$p = \frac{1-\epsilon}{\Delta}$ and define as $G'$ the induced subgraph on the sampled vertices.
Then, the largest connected component of $G'$ has size at most 
$\frac{4}{\epsilon^2} \log n$,
with high probability.
\label{theo:Gsample1}
\end{theo}

The above result pays homage to the giant component phase transition phenomena in random Erdos-Renyi graphs.
What is surprising is that a similar phase transition can apply for any given graph!

\paragraph{Adapting to ML-friendly sampling procedures.}

In practice, for most SU algorithms of interest, the sampling distribution of updates is either with or without replacement from the $n$ updates. 
As it turns out, morphing Theorem~\ref{theo:Gsample1} into a with-/without-replacement result is not straightforward.
We defer the analysis needed to Appendix \ref{app:replacement}, and present our main theorem about graph sampling here.

\begin{theo}
Let $G$ be a graph on $n$ vertices, with maximum vertex degree $\Delta$.
Let us sample $B = (1-\epsilon)\frac{n}{\Delta}$ vertices with or without replacement, and define as $G'$ the induced subgraph on the sampled vertices.
Then, the largest connected component of $G'$ has size at most $O(\frac{ \log n}{\epsilon^2})$,
with high probability.
\label{theo:ourGsample}
\end{theo}

The key idea from the above theorem is that if one samples no more than $B = (1-\epsilon)\frac{n}{\Delta}$ vertices, then there will be at least $O\left(\sfrac{\epsilon^2 B}{\log n} \right)$ conflict groups to allocate across cores, all of size at most $O\left(\log n /\epsilon^2\right)$. 
Moreover,  since there are no conflicts between different conflict-groups, the processing of updates per any single group will never interact with the variables corresponding to the updates of another conflict group.

The next step of \cyc{} is to form and allocate the connected components (CCs) across cores, and do so efficiently. 
We address this in the following subsection.
In the following, for simplicity we carry our analysis for the with-replacement sampling case, but it can be readily extended to the without-replacement sampling case.

\paragraph{Identifying Groups of Conflict via CCs}

In \cyc{}, we sample batches of updates of size $B=(1-\epsilon)\frac{n}{\Delta}$ multiple times, and for each batch we need to identify the conflict groups across the updates.  
Let us refer to 
$G_u^i$ as the subgraph induced by the $i$th sampled batch of updates on the
update-variable bipartite graph $G_u$.
In the following we always assume that we sample at most $n_b = c\cdot \frac{\Delta}{1-\epsilon}$ batches, where $c\ge 1$ is a constant that does not depend on $n$.
This number of batches results in a constant number of passes over the dataset.

Identifying the conflict groups in $G_u^i$ can be done with a connected components (CC) algorithm.
The main question we need to address is what is the best way to parallelize this graph partitioning part.
There are two avenues that we can take for this, depending on the number of cores $P$ at our disposal.
We can either parallelize the computation of the CCs of a single batch (i.e., compute the CCs of $G_u^i$ on $P$ cores), or we can compute in parallel the CCs of all $n_b$ batches, by allocating the sampled graphs $G_u^i$ to cores, so that each of them can compute the CCs of its allocated subgraphs. 
Depending on the number of available cores, one technique can be better than the other.
In Appendix \ref{sec:cc} we provide the details of this part, and prove the following result:

\begin{lem}
Let the number of cores by bounded as $P = O(\frac{n}{\Delta \Delta_{\text{L}} })$, and let $ \frac{\Delta_{\text{L}}}{\overline{\Delta}_{\text{L}}} \le \sqrt{n}$.
Then, the overall computation of CCs for $n_b = c\cdot \frac{\Delta}{1-\epsilon}$ batches, each of size $B = (1-\epsilon)\frac{n}{\Delta}$, costs no more than
$O(\frac{E_u \log^2 n}{P})$.
\end{lem}

\paragraph{Allocating Updates to Cores}

Once we compute the CCs (i.e., the conflicts groups of the sampled updates), we have to allocate them across cores.
Once a core has been assigned with CCs, it will process the updates included in these CCs, according to the order that each update has been labeled with.
Due to Theorem~\ref{theo:ourGsample}, each connected component will contain at most $O(\frac{\log n}{\epsilon^2})$ updates.
Assuming that the cost of the $j$-th update in the batch is $w_j$, the cost of a single connected component $\setC$ will be $w_{\setC} = \sum_{j\in \setC} w_j$. 
To proceed with characterizing the maximum load among the $P$ cores, we assume that the cost of a single update $u_i$, for $i\in\{1,\ldots, n\}$, is proportional to the out-degree of that update ---according to the update-variable graph $G_u$--- times a constant cost which we shall refer to as $\kappa$. Hence, $w_j = O(d_{L,j} \cdot \kappa)$, where $d_{L,j}$ is the  degree of the $j$-th left vertex of $G_u$.
In Appendix \ref{sec:allocation} we establish that a near-uniform allocation of CCs according to their weights leads to the following guarantee.

\begin{lem}
Let the number of cores by bounded as $P = O(\frac{n}{\Delta \Delta_{\text{L}} })$, and let $ \frac{\Delta_{\text{L}}}{\overline{\Delta}_{\text{L}}} \le \sqrt{n}$.
Then, computing the stochastic updates across all $n_b= c\cdot \frac{\Delta}{1-\epsilon}$ batches can be performed in time
$O(\frac{E\log^2 n}{P} \cdot \kappa )$, with high probability, where $\kappa$ is the per edge cost for computing one of the $n$ updates defined on $G_u$.
\end{lem}

\paragraph{Stitching the pieces together}

Now that we have described the sampling, conflict computation, and allocation strategies, we are ready to put all the pieces together and detail \cyc{} in full.
Let us  assume that we sample a total number of $n_b = c\cdot \frac{\Delta}{1-\epsilon}$ batches of size $B=(1-\epsilon)\frac{n}{\Delta}$, and that each update is sampled uniformly at random.
For the $i$-th batch let us denote as $\setC_1^i, \ldots \setC_{m_i}^i$ the connected components on the induced subgraph $G_u^i$.
Due to Theorem~\ref{theo:ourGsample}, each connected component $\setC$ contains a number of at most $O(\frac{\log n}{\epsilon^2})$ updates, and each update carries an ID (the order of which it would have been sampled by the serial algorithm).
Using the above notation, we give the pseudocode for \cyc{} in Alg. \ref{alg:SU}.

\begin{wrapfigure}{R}{0.5\columnwidth}
\vspace{-0.3cm}
\begin{minipage}{0.5\columnwidth}
\begin{algorithm}[H]
 \caption{\cyc{}}
   {\small
\begin{algorithmic}[1]
\STATE {\bf Input:} $G_u, T, B$. 
\STATE Sample $n_b = T/B$ subgraphs $G_u^1, \ldots, G_u^{n_b}$ from $G_u$
\STATE Cores compute in parallel CCs for sampled subgraphs
\FOR{batch $i=1:n_b$}
\STATE Allocation of $\setC_1^i, \ldots \setC_{m_i}^i$ to $P$ cores
\vspace{0.1cm}
\hrule
\vspace{0.1cm}
\FOR{each core {\bf in parallel}}
\FOR{each allocated component $\mathcal{C}$}
\FOR{each update $j$ (in order) from $\mathcal{C}$}
\STATE ${\bf x}_{\mathcal{S}_j}  = u_j({\bf x}_{\mathcal{S}_j}, f_j)$ 
\ENDFOR
\ENDFOR
\ENDFOR
\ENDFOR
\vspace{0.1cm}
\hrule
\vspace{0.1cm}
\STATE {\bf Output:} {\bf x}
\end{algorithmic}
}
   \label{alg:SU}
 \end{algorithm}
\end{minipage}
\end{wrapfigure}

Note that the inner loop that is parallelized (i.e., the SU processing loop in lines 6 -- 9), can be performed asynchronously; cores do not have to synchronize, and do not need to lock any memory variables, as they are all accessing non-overlapping subset of $\vecx$. 
This also provides for better cache coherence.
Moreover, each core potentially accesses the same coordinates several times, leading to good cache locality.
These improved cache locality and coherence properties experimentally lead to substantial performance gains as we see in the next section.

We can now combine the results of the previous subsection to obtain our main theorem for \cyc{}.
\begin{theo}
Let us assume any given update-variable graph $G_u$ with average, and max left degree $ \overline{\Delta}_L$ and $ {\Delta}_L$, such that $ \frac{\Delta_{\text{L}}}{\overline{\Delta}_{\text{L}}} \le \sqrt{n}$, and with induced max conflict degree $\Delta$.
Then, \cyc{} on $P  = O( \frac{n}{\Delta \cdot \Delta_L})$ cores, with batch sizes $B = (1-\epsilon)\frac{n}{\Delta}$ can execute $T = c\cdot n$ updates, for any constant $c\ge 1$, selected uniformly at random with replacement, in time 
$$\mathcal{O}\left( \frac{E_u \cdot \kappa}{P}\cdot  \log^2 n \right),$$ 
with high probability.
\label{theo:main}
\end{theo}

Observe that \cyc{} bypasses the need to establish convergence guarantees for the parallel algorithm. Hence, it could be the case for many applications of interest that although we might not be able to analyze how ``well'' the serial SU algorithm might perform in terms of the accuracy of the solution, \cyc{} can provide black box guarantees for speedup, since our analysis is completely oblivious to the qualitative performance of the serial algorithm.  
This is in contrast to recent studies similar to \cite{de2015taming}, where the authors provide speedup guarantees via a convergence-to-optimal proof for an asynchronous SGD on a nonconvex problem. Unfortunately these proofs can become complicated especially on a wider range of nonconvex objectives.

In the following section we show that \cyc{} is not only useful theoretically, but can consistently outperform \HW{} on sufficiently sparse datasets.

\section{Evaluation}
\label{sec:expts}

\subsection{Implementation and Setup} 
We implemented \cyc{} in \texttt{C++} and tested it on a variety of problems and datasets described below.
We tested a number of stochastic updates algorithms, and compared against their \hog{} (i.e., asynchronous and lock-free) implementations --- in some cases, there are no theoretical foundations for these \hog{} implementations, even if they work reasonably well in practice.
Since \cyc{} is intended to be a general approach for parallelization of stochastic updates algorithms, we do not compare against algorithms designed and tailored for specific applications, nor do we expect \cyc{} to outperform every such highly-tuned, well-designed, specific algorithms.

Our experiments were conducted on a machine with 72 CPUs (Intel(R) Xeon(R) CPU E7-8870 v3, 2.10 GHz) on 4 NUMA nodes, each with 18 CPUs, and 1TB of memory.
We ran both \cyc{} and \hog{} with 1, 4, 8, 16 and 18 threads pinned to CPUs on
a single NUMA node (i.e., the maximum physical number of cores possible, for a single node), so that we can avoid well-known cache coherence and scaling issues across different nodes \cite{zhang2014dimmwitted}. 
 We note that distributing threads across NUMA nodes
significantly increased running times for both \cyc{} and \hog{}, but was
relatively worse for \hog{}.  We believe this is due to the poorer locality of
\hog{}, which results in more cross-node communication.
In this paper, we exclusively focus our study and experiments on parallelization within a single NUMA node, and leave cross-NUMA node parallelization for future work, while referring the interested reader to a recent study of the various tradeoffs of ML algorithms on NUMA aware architectures \cite{zhang2014dimmwitted}.

In our experiments, we measure overall running times which include the overheads for computing connected components and allocating work in \cyc{}.
Separately, we also measure running times for performing the stochastic updates by excluding the \cyc{} coordination overheads.
We also compute the objective value at the end of each epoch (i.e., one full pass over the data).
We measure the speedups for each algorithm as
\[
\frac{\text{time of the parallel algorithm to reach $\epsilon$ objective}}{\text{time of the serial algorithm to reach $\epsilon$ objective}}
\]
where $\epsilon$ was chosen to be the smallest objective value that is achievable by all parallel algorithms on every choice of number of threads.
That is, $\epsilon = \max_{\mathcal{A},T}\min_{e} f(X_{\mathcal{A},T,e})$
where $X_{\mathcal{A},T,e}$ is the model learned by algorithm $\mathcal{A}$ on $T$ threads after $e$ epochs.
The serial algorithm used for comparison is \hog{} running serially on one thread.

\begin{table}[t]
\begin{center}
\begin{small}
\begin{tabular}{|c||c|c|c|c|l|}
\hline
\multirow{3}{*}{Dataset} & \multirow{3}{*}{\# datapoints} & \multirow{3}{*}{\# features} & Density (average & \multirow{3}{*}{Comments}\\
&  &  & number of features &\\
& & & per datapoint)& \\
\hline\hline
\multirow{2}{*}{NH2010} & \multirow{2}{*}{48,838} & \multirow{2}{*}{48,838} & \multirow{2}{*}{4.8026} & Topological graph of 49 \\
& & & &Census Blocks in New Hampshire.\\\hline
\multirow{2}{*}{DBLP} & \multirow{2}{*}{5,425,964} & \multirow{2}{*}{5,425,964} & \multirow{2}{*}{3.1880} & Authorship network of 1.4M authors\\
& & & & and 4M publications, with 8.65M edges.\\\hline
\multirow{2}{*}{MovieLens}  &\multirow{2}{*}{$\sim$10M} & \multirow{2}{*}{82,250} & \multirow{2}{*}{200} & 10M movie ratings from 71,568  users\\
& & & & for 10,682 movies.\\\hline
EN-Wiki & 20,207,156 & 213,272 & 200 & Subset of English Wikipedia dump.\\\hline
\end{tabular}
\end{small}
\caption{\small 
Details of datasets used in our experiments.
}
\label{tab:datasets}
\end{center}
\end{table}

In Table \ref{tab:datasets} we list some details of the datasets that we use in our experiments.
The stepsizes and batch sizes used for each problem are listed in Table \ref{tab:exptparams}, along with dataset and problem details.
In general, we chose the stepsizes to maximize convergence without diverging.
Batch sizes were picked to optimize performance for \cyc{}.

\begin{table}[t]
\begin{center}
\begin{small}
\begin{tabular}{|c|c|c||c|c|c|c|c|}
\hline
\multirow{3}{*}{Problem} & \multirow{3}{*}{Algorithm} & \multirow{3}{*}{Dataset} & \hog{} & \cyc{} & Batch & Average \# & Average size\\
& & & Stepsize & Stepsize & Size & of connected & of connected\\
& & & & & & components & components\\
\hline\hline
\multirow{2}{*}{Least squares} & \multirow{2}{*}{SAGA} & NH2010 & $1 \times 10^{-14}$ & $3 \times 10^{-14}$ & 1,000 & 792.98 & 1.257 \\\cline{3-8}
& & DBLP & $1 \times 10^{-5}$ & $3 \times 10^{-4}$ & 10,000 & 9410.34 & 1.062 \\\hline
\multirow{2}{*}{Graph eigen} & \multirow{2}{*}{SVRG} & NH2010 & $1 \times 10^{-5}$ & $1 \times 10^{-1}$ & 1,000 & 792.98 & 1.257 \\\cline{3-8}
& & DBLP & $1 \times 10^{-7}$ & $1 \times 10^{-2}$ & 10,000 & 9410.34 & 1.062 \\\hline
\multirow{2}{*}{Matrix comp} & SGD & \multirow{2}{*}{MovieLens} & \multicolumn{2}{|c|}{\multirow{2}{*}{$5\times 10^{-5}$}} & \multirow{2}{*}{5,000} & \multirow{2}{*}{1663.73} & \multirow{2}{*}{3.004} \\\cline{2-2}
& Weighted SGD & &\multicolumn{2}{|c|}{} & & &\\\hline
Word embed & SGD & EN-Wiki & \multicolumn{2}{|c|}{$1 \times 10^{-10}$} & 4,250 & 2571.51 & 1.653 \\\hline
\end{tabular}
\end{small}
\caption{\small Stepsizes and batch sizes for the various learning tasks in our evaluation.
We selected stepsizes that maximize convergence without diverging.
We also chose batch sizes to maximize performance of \cyc{}.
We further list the average size of connected components and the average number of connected components in each batch.
Typically there are many connected components with small average size, which leads to good load balancing for \cyc{}.
}
\label{tab:exptparams}
\end{center}
\vskip -0.1in
\end{table}

\subsection{Learning tasks and algorithmic setup}

\COND{}{\vspace{-0.4cm}}
\paragraph{Least squares via SAGA}
The first problem we consider is least squares: 
$$\min_{\bf x} \frac{1}{n} \|{\bf A}{\bf x}-{\bf b}\|_2^2 = \min_{\bf x} \frac{1}{n}\sum_{i=1}^n ({\bf a}_i^T{\bf  x} - b_i)^2$$
which we will solve using the SAGA algorithm \cite{defazio2014saga}, an incrimental gradient algorithm with faster than SGD rates on convex, or strongly convex functions.
In SAGA, we initialize ${\bf g}_i = \nabla f_i ({\bf x}_0)$ and iterate the following two steps
\begin{align*}
&{\bf x}_{k+1} = {\bf x}_k - \gamma \cdot \left(\nabla f_{s_k} ({\bf x}_k) - {\bf g}_{s_k} + \frac{1}{n} \sum_{i=1}^n {\bf g}_i\right)\\
&{\bf g}_{s_k} = \nabla f_{s_k} ({\bf x}_k).
\end{align*}
where $f_i({\bf x}) =  ({\bf a}_i^T{\bf  x} - b_i)^2$ and
$\nabla f_i({\bf x}) = 2 \left({\bf a}_i^T{\bf x} -b_i\right) {\bf a}_i.$
In the above iteration it is useful to observe that the updates can be performed in a sparse and "lazy" way. That is for any updates where the sampled gradients $\nabla f_{s_k}$ have non-overlapping support, we can still run them in parallel, and apply the vector of gradient sums at the end of a batch "lazily". 
We explain the details of the lazy updates in Appendix \ref{app:lazy}.
This requires computing the number of skipped gradient sum updates, say they were $\tau_j$ of them for each lazily updated coordinate $j$, which may be negative in \hog{} due to re-ordering of updates.
We thresholded $\tau_j$ when needed in the \hog{} implementation, as this produced better convergence for \hog{}.
Unlike other experiments, we used different stepsizes $\gamma$ for \cyc{} and \hog{}, as \hog{} would often diverge with larger stepsizes.
The stepsizes chosen for each were the largest such that the algorithms did not diverge.
We used the  DBLP and NH2010 datasets for this experiment, and set ${\bf A}$ as the adjacency matrix of each graph.
For NH2010, the values of ${\bf b}$ were set to population living in the Census Block.
For DBLP we used synthetic values: we set ${\bf b} = {\bf A}\tilde{{\bf x}} + 0.1 \tilde{{\bf z}}$, where $\tilde{{\bf x}}$ and $\tilde{{\bf z}}$ were generated randomly.
The SAGA algorithm was run for up to 500 epochs for each dataset.

\paragraph{Graph eigenvector via SVRG}
Given an adjacency matrix ${\bf A}$, the top eigenvector of ${\bf A}^T{\bf A}$ is useful in several applications such as spectral clustering, principle component analysis, and others.
In a recent work, \cite{jin2015robust} proposes an algorithm for computing the top eigenvector of ${\bf A}^T{\bf A}$ by running intermediate SVRG steps to approximate the shift-and-invert iteration.
Specifically, at each step SVRG is used to solve 
$$\min \frac{1}{2}{\bf x}^T (\lambda {\bf I} - {\bf A}^T{\bf A}) {\bf x} - {\bf b}^T{\bf x} = \min \sum_{i=1}^n \left( \frac{1}{2}{\bf x}^T \left(\frac{\lambda}{n} {\bf I} - {\bf a}_i {\bf a}_i^T\right) {\bf x} - \frac{1}{n}{\bf b}^T{\bf x}\right).$$
According to  \cite{jin2015robust}, if we initialize ${\bf y} = {\bf x}_0$ and assume $\|{\bf a}_i \|=1$, we have to iterate the following updates
\begin{align*}
&{\bf x}_{k+1} = {\bf x}_k - \gamma \cdot n \cdot \left( \nabla f_{s_k}({\bf x}_k)- \nabla f_{s_k}({\bf y})\right) + \gamma\cdot \nabla f ({\bf y})\\
\end{align*}
where after every $T$ iterations we update  ${\bf y} = {\bf x}_k$, and the stochastic gradients are of the form 
$\nabla f_i({\bf x}) = \left(\frac{\lambda}{n} {\bf I} -  {\bf a}_i {\bf a}_i^T\right){\bf x} - \frac{1}{n}{\bf b}.$

We apply \cyc{} to SVRG with dense linear gradients (see App. \ref{app:lazy}) for parallelizing this problem, which uses lazy updates to avoid dense operations on the entire model ${\bf x}$.
This requires computing the number of skipped updates, $\tau_j$, for each lazily updated coordinate, which may be negative in \hog{} due to re-ordering of updates.
In our \hog{} implementation, we thresholded the bookkeeping variable $\tau_j$ (described in App. \ref{app:lazy}), as we found that this produced faster convergence.
The rows of ${\bf A}$ are normalized by their $\ell_2$-norm, so that we may apply the SVRG algorithm of \cite{jin2015robust} with uniform sampling.
Two graph datasets were used in this experiment.
The first, DBLP \cite{konect2015dblp}, is an authorship network consisting of 1.4M authors and 4M publications, with 8.65M edges.
The second, NH2010 \cite{uflsparse2014nh2010}, is a weighted topological graph of 49 Census Blocks in New Hampshire, with an edge between adjacent blocks, for a total of 234K edges.
We ran SVRG for 50 and 100 epochs for NH2010 and DBLP respectively.

\paragraph{Matrix completion via SGD}
In the matrix completion problem, we are given a partially observed $n \times m$ matrix ${\bf M}$, and wish to factorize it as ${\bf M} \approx {\bf U}{\bf V}$ where ${\bf U}$ and ${\bf V}$ are low rank matrices with dimensions $n \times r$ and $r\times m$ respectively.
This may be achieved by optimizing 
$$\min_{{\bf U},{\bf V}} \sum_{(i,j) \in \Omega} (M_{i,j} - {\bf U}_{i,\cdot} {\bf V}_{\cdot,j})^2$$ 
where $\Omega$ is the set of observed entries, which can be approximated by SGD on the observed samples.
The objective can also be regularized as:
$$\min_{{\bf U},{\bf V}} \sum_{(i,j) \in \Omega} (M_{i,j} - {\bf U}_{i,\cdot} {\bf V}_{\cdot,j})^2 + \frac{\lambda}{2}(\|{\bf U}\|_F^2 + \|{\bf V}\|_F^2)
=\min_{{\bf U},{\bf V}} \sum_{(i,j) \in \Omega} \left((M_{i,j} - {\bf U}_{i,\cdot} {\bf V}_{\cdot,j})^2 + \frac{1}{|\Omega|}\frac{\lambda}{2}(\|{\bf U}\|_F^2 + \|{\bf V}\|_F^2)\right)
.$$
The regularized objective can be optimized by weighted SGD, which samples $(i,j) \in \Omega$ and updates
\begin{align*}
{\bf U}_{i',\cdot} &\leftarrow \begin{cases}
(1-\gamma\lambda) {\bf U}_{i,\cdot} - \gamma \cdot |\Omega|\cdot 2({\bf U}_{i,\cdot} {\bf V}_{\cdot,j}-M_{i,j}) ({\bf V}_{\cdot,j})^T & \text{if } i = i'\\
(1-\gamma\lambda) {\bf U}_{i',\cdot} & \text{otherwise}
\end{cases}
\end{align*}
and analogously for ${\bf V}_{\cdot,j}$
In our experiments, we chose a rank of $r=100$, and ran SGD and weighted SGD for 200 epochs.
We used the MovieLens 10M dataset \cite{grouplens2009movielens10M} containing 10M ratings for 10,000 movies by 72,000 users.

\COND{}{\vspace{-0.4cm}}
\paragraph{Word embedding via SGD}
Semantic word embeddings aim to represent the meaning of a word $w$ via a vector ${\bf v}_w \in \mathbb{R}^r$.
In a recent work by \cite{arora2015rand-walk}, the authors propose using a generative model, and solving for the MLE which is equivalent to:
$$\min_{\{{\bf v}_w\},C} \sum_{w, w'} A_{w,w'} (\log(A_{w,w'}) - \|{\bf v}_w + {\bf v}_{w'}\|_2^2 - C)^2,$$
where $A_{w,w'}$ is the number of times words $w$ and $w'$ co-occur within $\tau$ words in the corpus. 
In our experiments we set $\tau=10$ following the suggested recipe of the aforementioned paper.
We can approximate the solution to the above problem by SGD: we can repeatedly sample entries $A_{w,w'}$ from ${\bf A}$ and update the corresponding vectors ${\bf v}_w, {\bf v}_{w'}$.
In this case the update is of the form as:
\begin{align*}
{\bf v}_w &= {\bf v}_{w} + 4\gamma A_{w,w'} (\log(A_{w,w'}) - \|{\bf v}_w + {\bf v}_{w'}\|_2^2 - C) ({\bf v}_w + {\bf v}_{w'})\\
\end{align*}
and identically for ${\bf v}_{w'}$
Then, at the end of each full pass over the data, we update the constant $C$ by its locally optimal value, which can be calculated in closed form:
$$C \leftarrow \frac{\sum_{w, w'} A_{w,w'} (\log(A_{w,w'}) - \|{\bf v}_w + {\bf v}_{w'}\|_2^2)}{\sum_{w, w'} A_{w,w'}}.$$
In our experiments, we optimized for a word embedding of dimension $r=100$, and tested on a $80$MB subset of the English Wikipedia dump available at \cite{mahoney2006enwiki}.
The dataset contains $213$K words and ${\bf A}$ has 20M non-zero entries.
For our experiments, we run SGD for 200 epochs.

\subsection{Speedup and Convergence Results}
In this subsection, we present the bulk of our experimental findings.
Our extended and complete set of results can be found in Appendix \ref{app:exptresults}.

\COND{
\begin{figure}[H]
  \centering
    \includegraphics[width=0.7\textwidth]{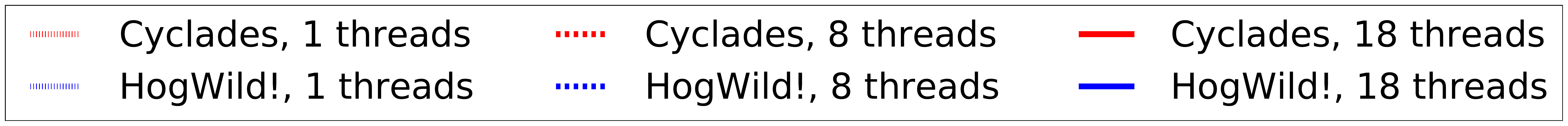}\\
    \subfigure[Least Squares, DBLP, SAGA]{\label{fig:ls_DBLP_may_saga:converge} \includegraphics[width=0.23\textwidth]{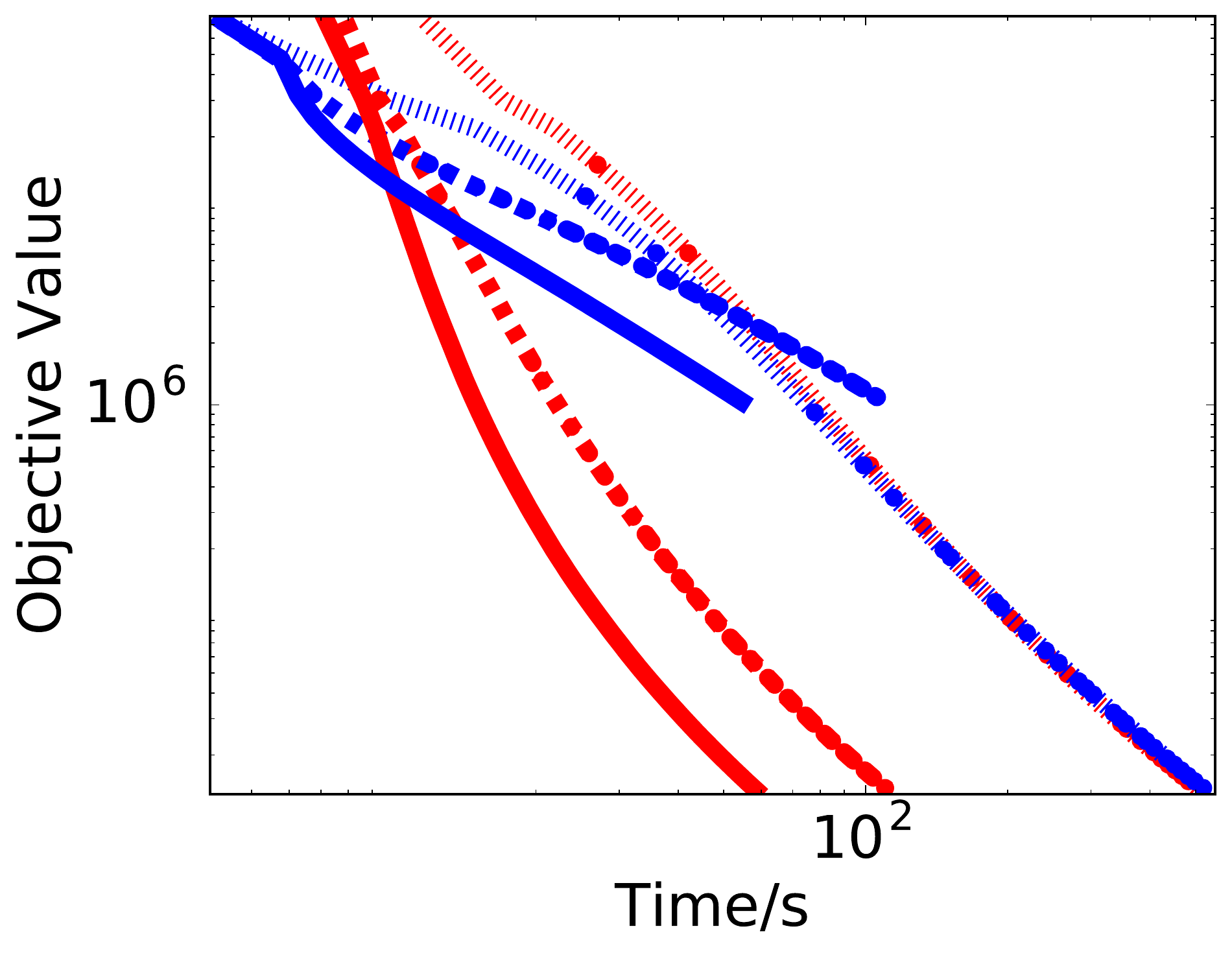}}
    \subfigure[Graph Eig., NH2010, SVRG]{\label{fig:ge_nh2010_may_svrg:converge} \includegraphics[width=0.23\textwidth]{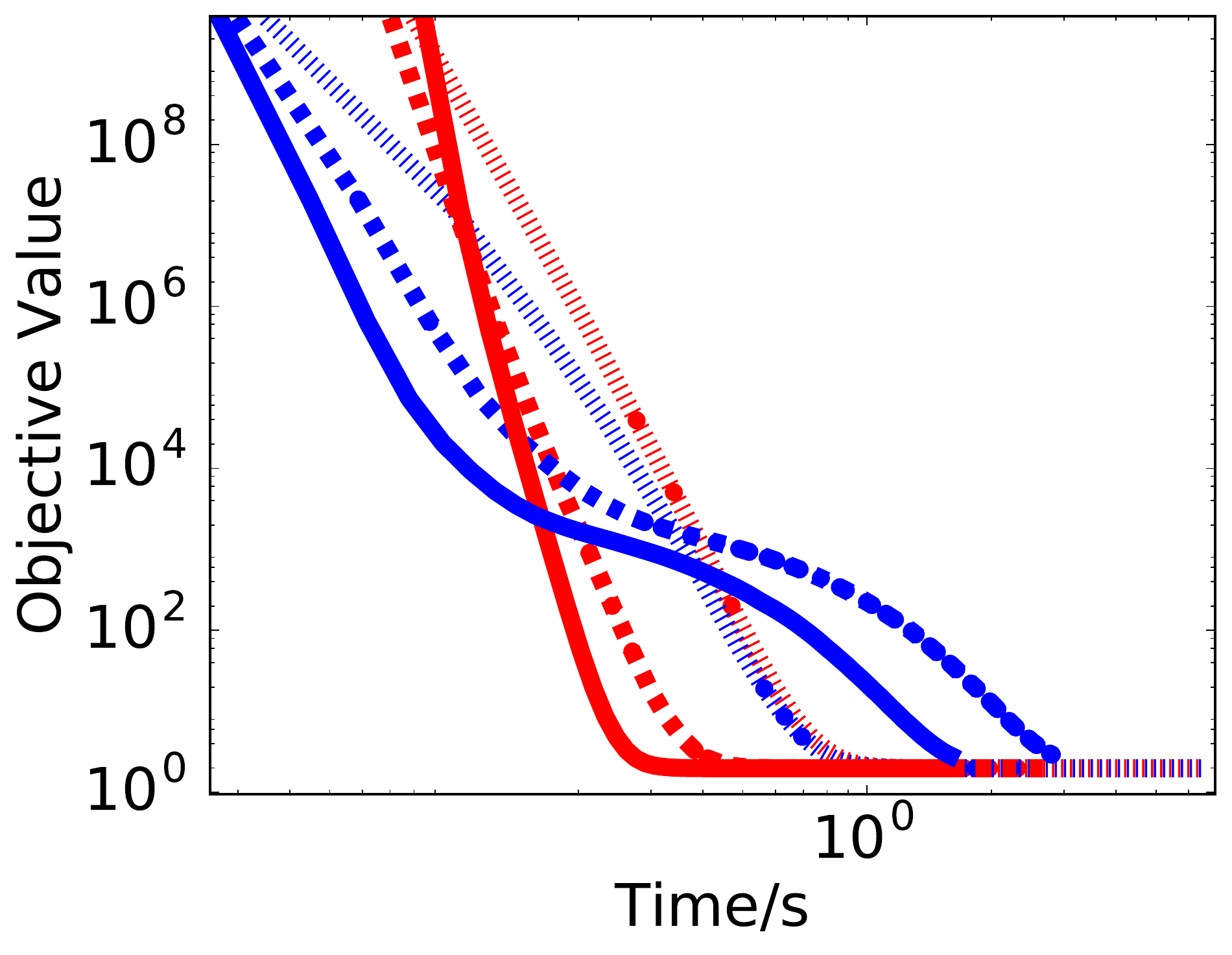}}
        \subfigure[Mat. Comp., 10M,  $\ell_2$-SGD]{\label{fig:mc_10m_may_rsg:converge} \includegraphics[width=0.23\textwidth]{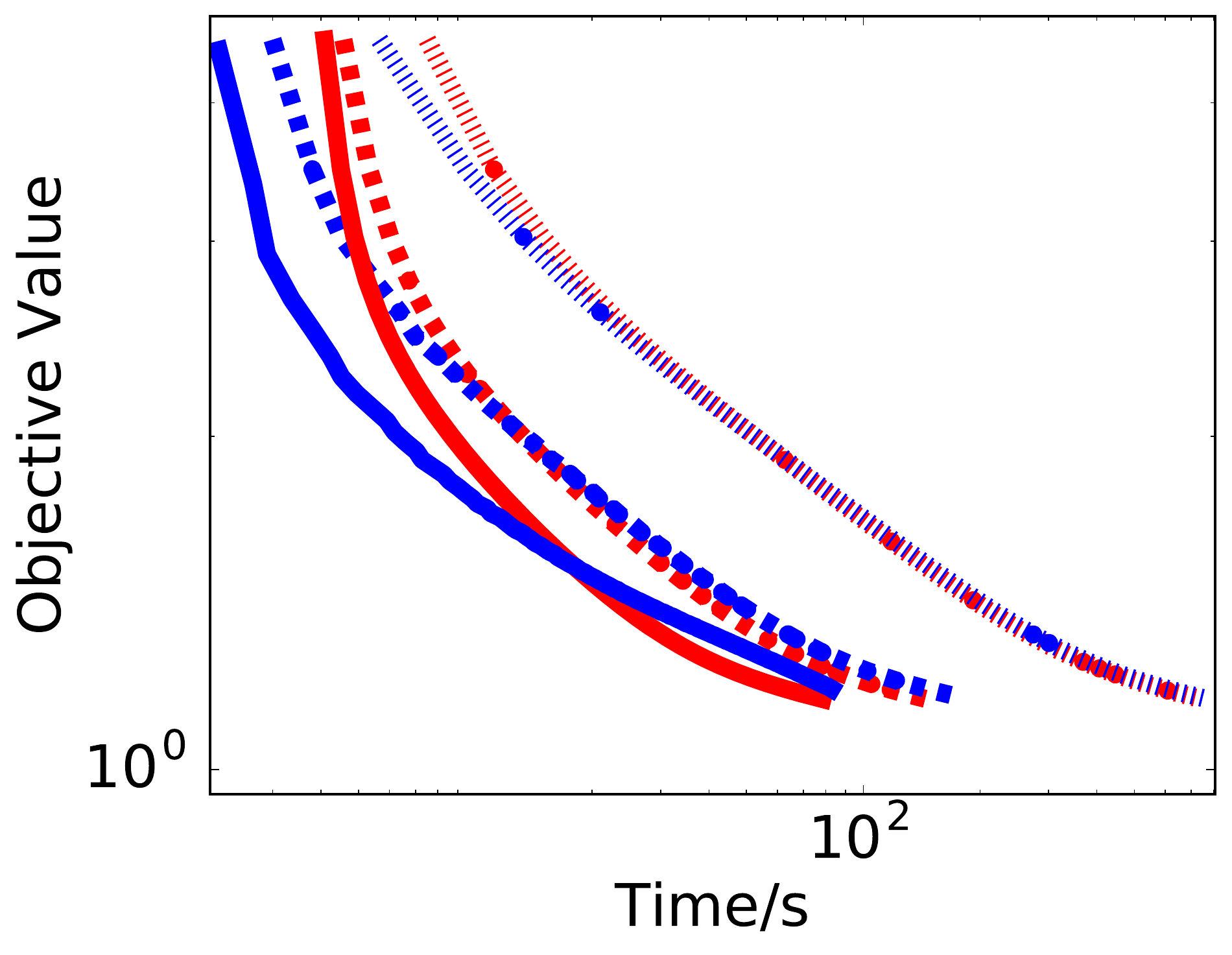}}
    \subfigure[Word2Vec, EN-Wiki, SGD]{\label{fig:we_ewk_may_sgd:converge} \includegraphics[width=0.23\textwidth]{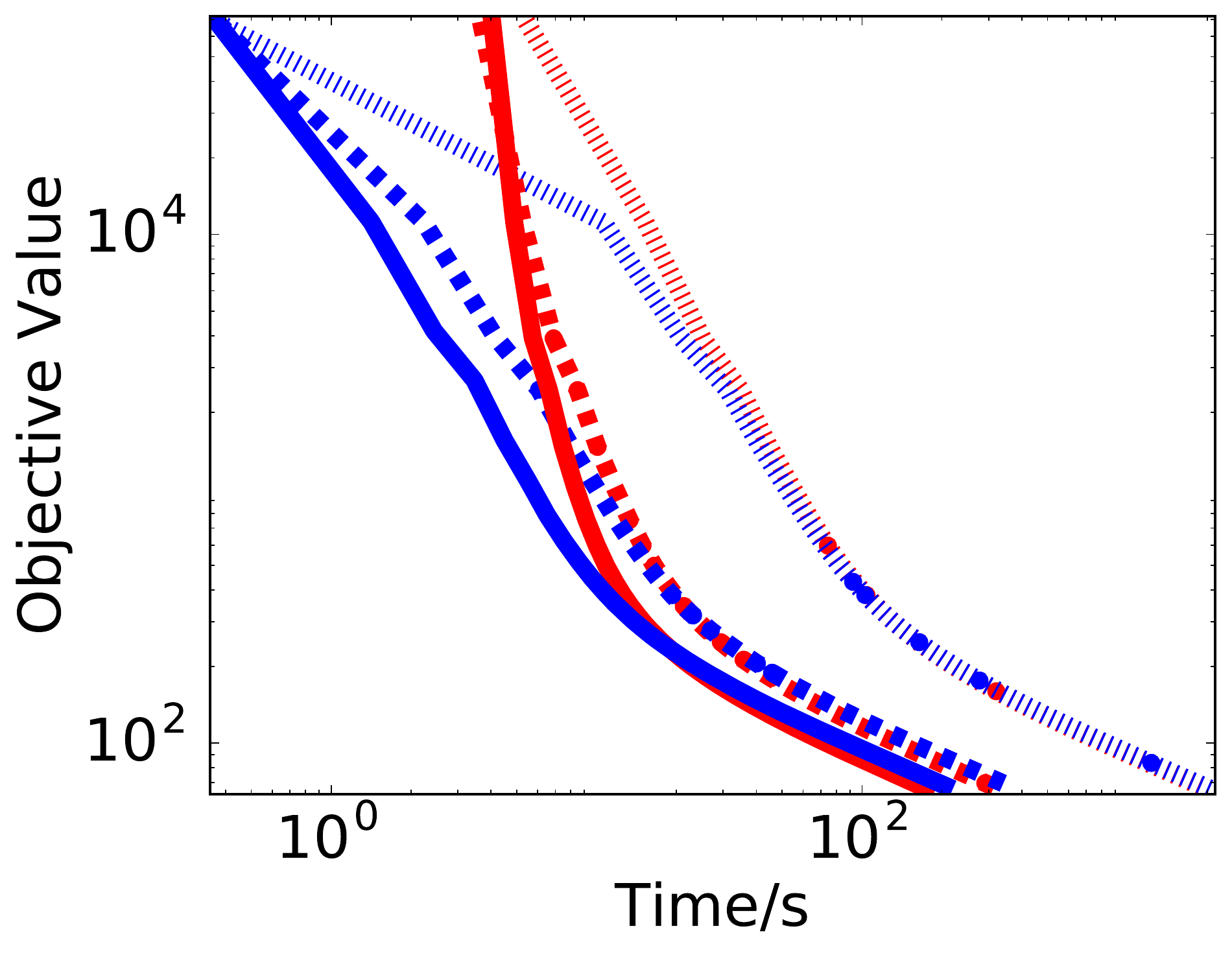}}

  \caption{\small Convergence of \cyc{} and \hog{} in terms of overall running time with 1, 8, 16, 18 threads.
  \cyc{} is initially slower, but ultimately reaches convergence faster than \hog{}.
  }
  \label{fig:expt_converge}
\end{figure}
\begin{figure}[H]
  \centering
    \subfigure[Least squares, DBLP, SAGA]{\label{fig:ls_DBLP_may_saga:speedups} \includegraphics[width=0.23\textwidth]{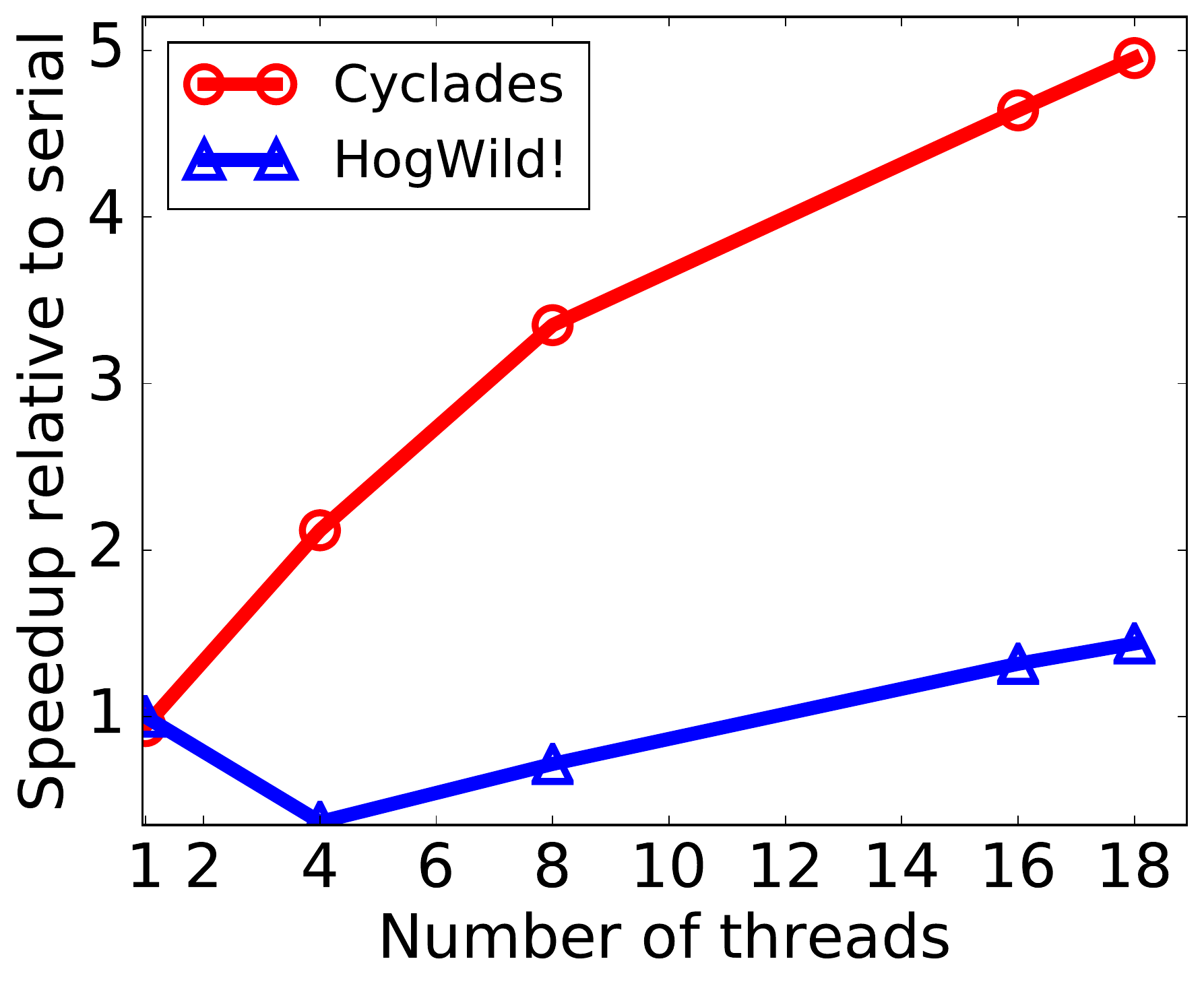}}
      \subfigure[Graph Eig., NH2010, SVRG]{\label{fig:ge_nh2010_may_svrg:speedups} \includegraphics[width=0.23\textwidth]{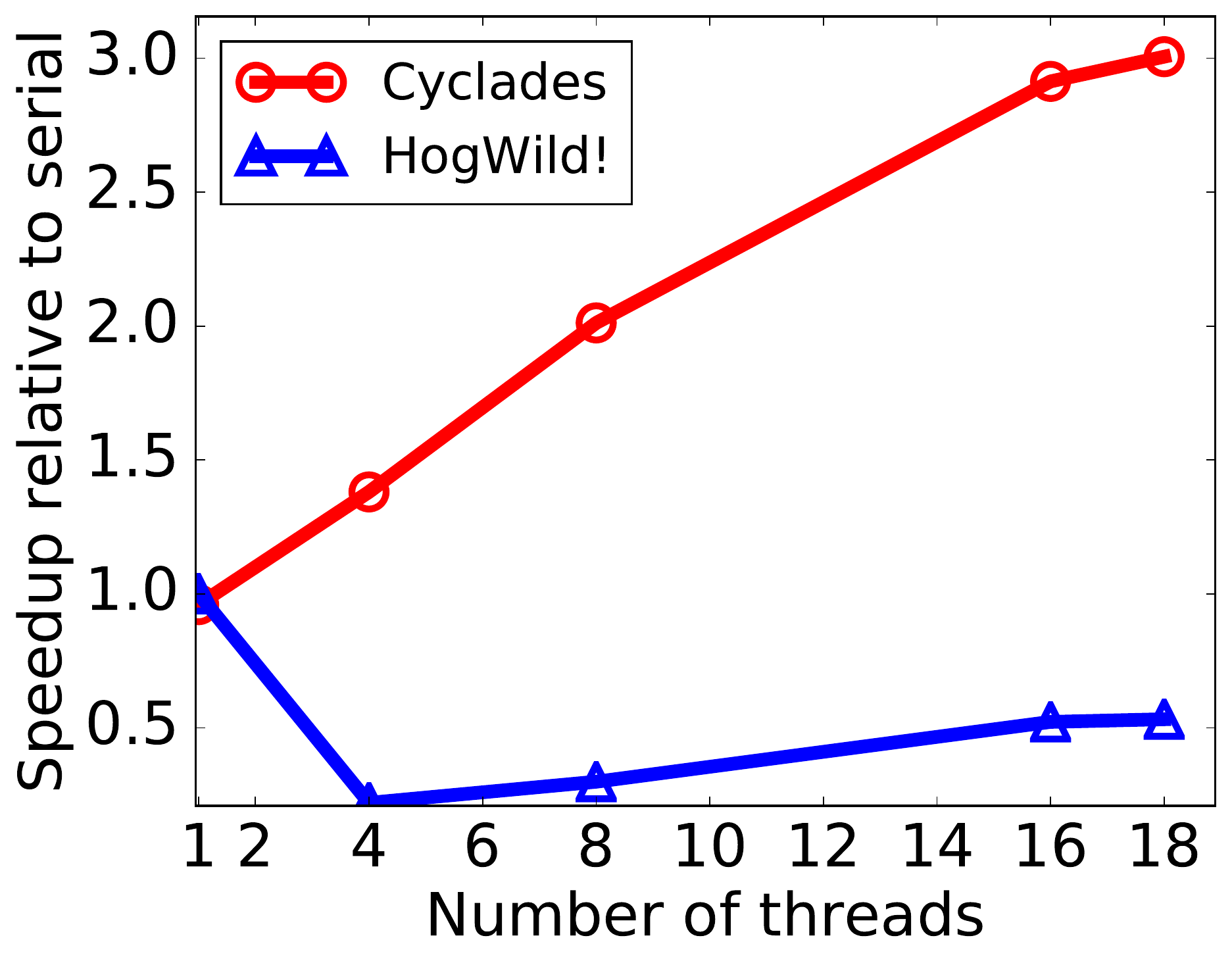}}
    \subfigure[Mat. Comp., 10M, $\ell_2$-SGD]{\label{fig:mc_10m_may_rsg:speedups} \includegraphics[width=0.23\textwidth]{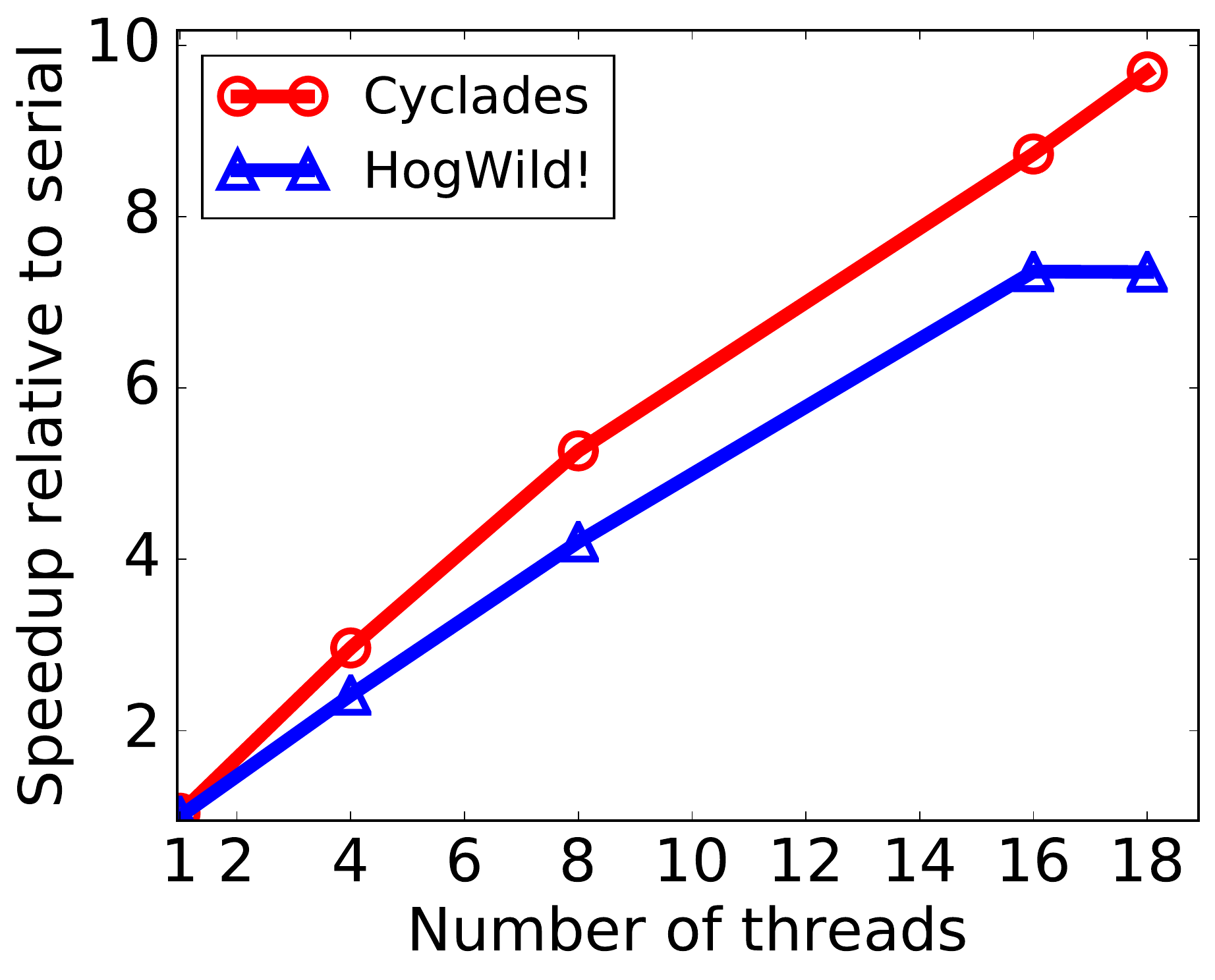}}
    \subfigure[Word2Vec, EN-Wiki, SGD]{\label{fig:we_ewk_may_sgd:speedups} \includegraphics[width=0.23\textwidth]{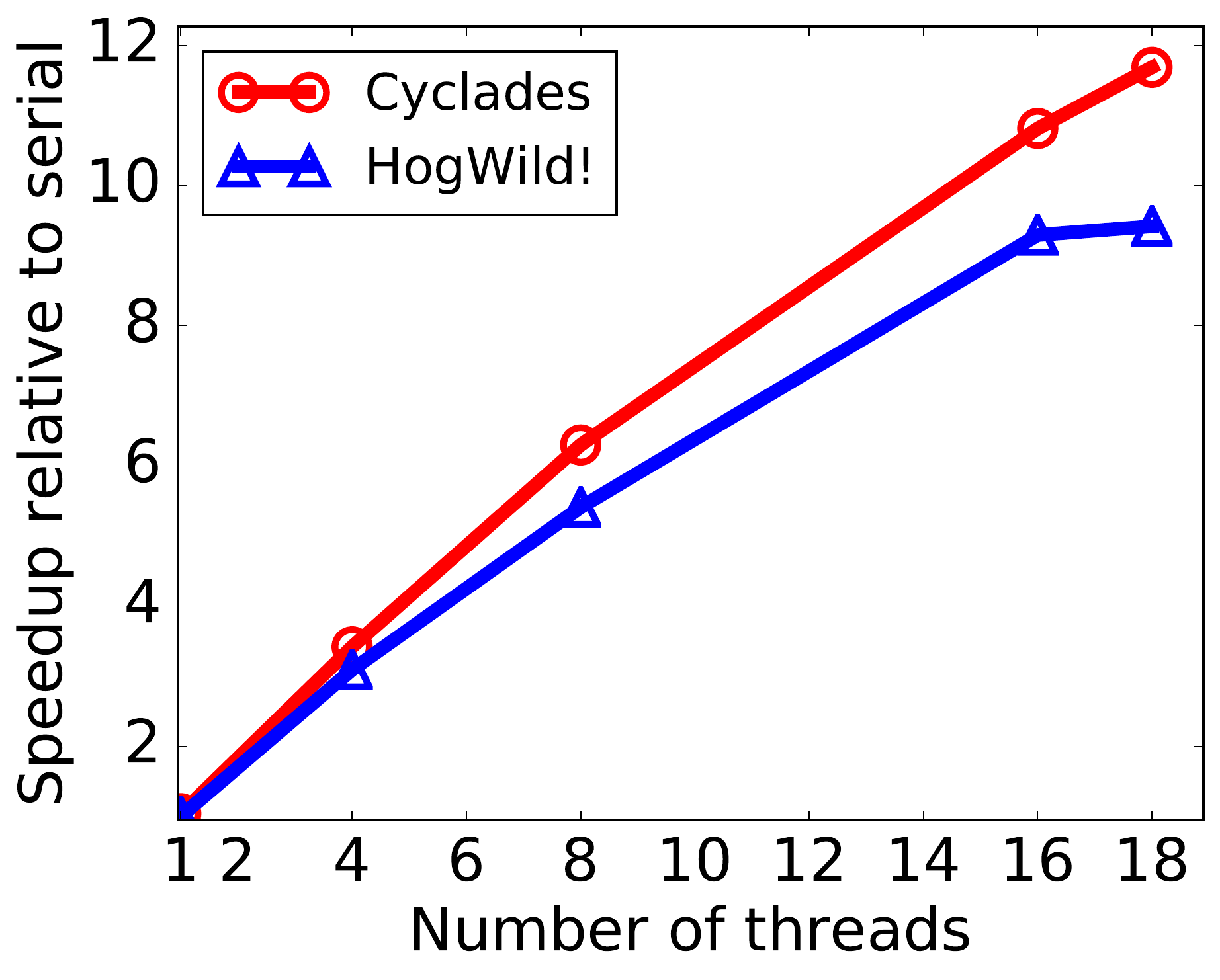}}
  \caption{\small Speedup of \cyc{} and \hog{} versus number of threads. 
  On multiple threads, \cyc{} always reaches $\epsilon$ objective faster than \hog{}.
  In some cases \cyc{} is faster than \hog{} even on 1 thread, due to better cache locality.
  In Figs. \ref{fig:ls_DBLP_may_saga:speedups} and \ref{fig:ge_nh2010_may_svrg:speedups}, \cyc{} exhibits significant gains,
  since \hog{} suffers from asynchrony noise for which we had to use comparatively smaller stepsizes to prevent divergence.
  }
  \label{fig:expt_speedup}
\vspace{-0.2cm}
\end{figure}

}{}

\COND{}{\vspace{-0.4cm}}
\paragraph{Least squares}

When running SAGA for least squares, we found that \hog{} was divergent with the large stepsizes that we were using for \cyc{} (Fig. \ref{fig:mc_10m_may_rsg:hdiverge}).
Thus, in the multi-thread setting, we were only able to use smaller stepsizes for \hog{}, which resulted in slower convergence than \cyc{}, as seen in Fig. \ref{fig:ls_DBLP_may_saga:converge}. 
The effects of a smaller stepsize for \hog{} are also manisfested in terms of speedups in Fig. \ref{fig:ls_DBLP_may_saga:speedups}, since \hog{} takes a longer time to converge to an $\epsilon$ objective value.

\begin{wrapfigure}{r}{0.5\textwidth}
  \centerline{\includegraphics[width=0.35\textwidth]{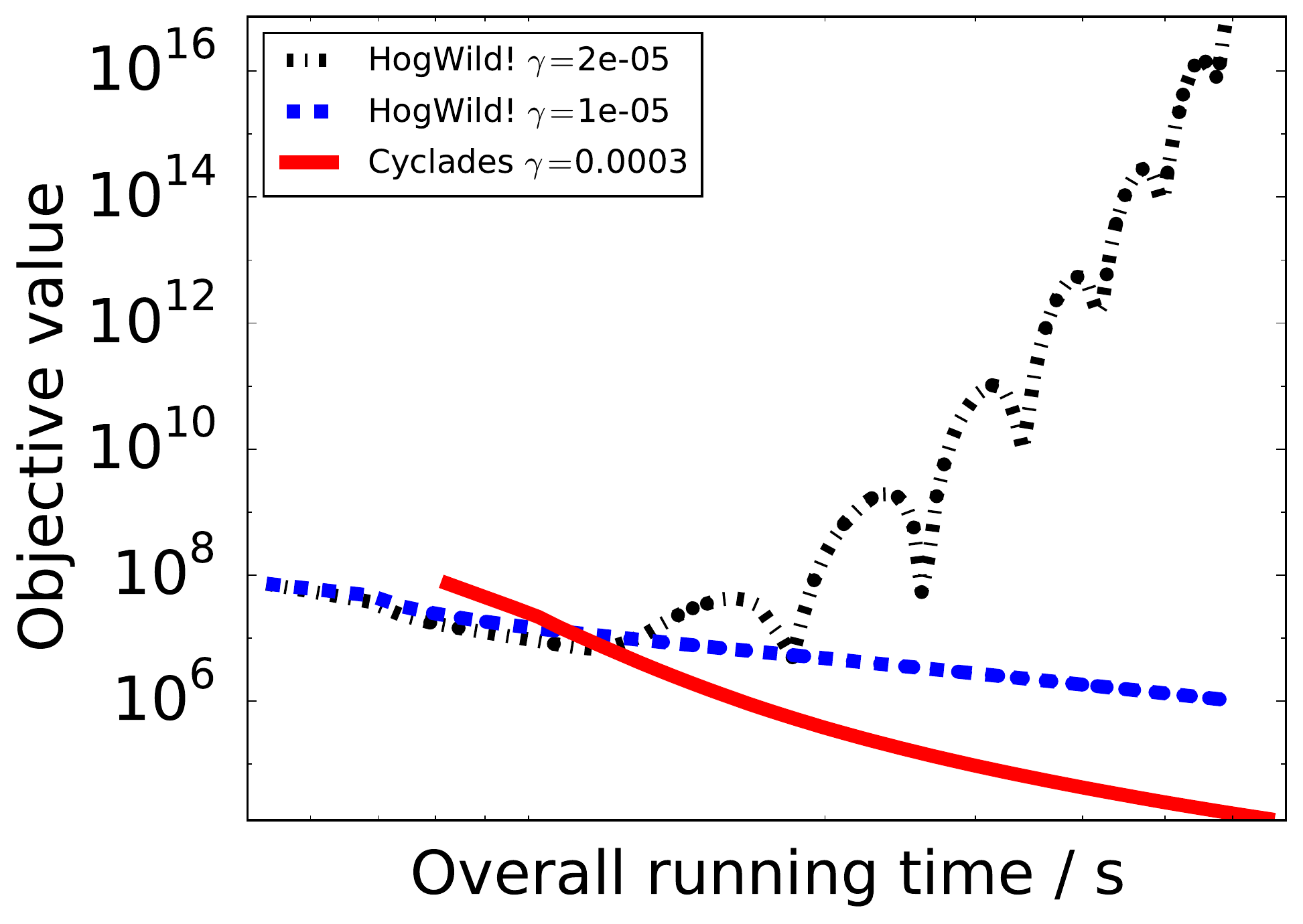}}
  \caption{\small 
    Convergence of \cyc{} and \hog{} on least squares using SAGA, with 16 threads, on DBLP dataset.
    \hog{} diverges with $\gamma > 10^{-5}$;
    thus, we were only able to use a smaller step size $\gamma = 10^{-5}$ for \hog{} on multiple threads.
    For \hog{} on 1 thread (and \cyc{} on any number of threads), we could use a larger stepsize of $\gamma=3 \times 10^{-4}$.}
  \label{fig:mc_10m_may_rsg:hdiverge} 
\vspace{-0.5cm}
\end{wrapfigure}

\COND{}{\vspace{-0.4cm}}
\paragraph{Graph eigenvector}
The convergence of SVRG for graph eigenvectors is shown in Fig. \ref{fig:ge_nh2010_may_svrg:converge}.
\cyc{} starts off slower than \hog{}, but always produces results equivalent to the convergence on a single thread.
Conversely, \hog{} does not exhibit the same behavior on multiple threads as it does serially---in fact, the error due to asynchrony causes \hog{} to converge slower on multiple threads.
This effect is clearly seen on Figs. \ref{fig:ge_nh2010_may_svrg:speedups}, where \hog{}
 fails to converge faster than the serial counterpart,
and \cyc{} attains a significantly better speedup on 16 threads.

\COND{}{\vspace{-0.4cm}}
\paragraph{Matrix completion and word embeddings}
Figures \ref{fig:mc_10m_may_rsg:converge} and \ref{fig:we_ewk_may_sgd:converge} show the convergence for the matrix completion and word embeddings problems.
\cyc{} is initially slower than \hog{} due to the overhead of computing connected components.
However, due to better cache locality and convergence properties, \cyc{} is able to reach a lower objective value in less time than \hog{}.
In fact, we observe that \cyc{} is faster than \hog{} when both are run serially, demonstrating that the gains from (temporal) cache locality outweigh the coordination overhead of \cyc{}.
These results are reflected in the speedups of \cyc{} and \hog{} (Figs. \ref{fig:mc_10m_may_rsg:speedups} and \ref{fig:we_ewk_may_sgd:speedups}).
\cyc{} consistently achieves a better speedup ($9-10\times$ on $18$ threads) compared to that of \hog{} ($7-9\times$  on $18$ threads).

\subsection{Runtime Breakdown}

\paragraph{Partitioning and allocation costs}
The cost of partitioning and allocation for \cyc{} is given in Table \ref{tab:cccost}, relatively to the time that \hog{} takes to complete one epoch of stochastic updates (i.e., a single pass over the dataset).
For matrix completion and the graph eigenvector problem, on 18 threads, \cyc{} takes the equivalent of 4-6 epochs of \hog{} to complete its partitioning, as the problem is either very sparse or the updates are expensive.
For solving least squares using SAGA and word embeddings using SGD, the cost of partitioning is equivalent to 11-14 epochs of \hog{} on 18 threads.
However, we point out that partitioning and allocation is a one-time cost which becomes cheaper with more stochastic update epochs.
Additionally, we note that this cost can become amortized quickly due to the extra experiments one has to run for hyperparameter tuning, since the graph partitioning is identical across different stepsizes one might want to test.
\begin{table}[H]
\begin{center}
\begin{small}
\begin{tabular}{@{}|@{}c@{}|@{}|c|c|c|c|c|c|c|}
\hline
\#     & Least Squares  & Least Squares & Graph Eig. & Graph Eig. & Mat. Comp. & Mat. Comp. & Word2Vec \\
threads      & SAGA & SAGA & SVRG & SVRG & SGD & Weighted SGD & SGD \\
 & NH2010 & DBLP& NH2010 & DBLP & MovieLens & MovieLens & EN-Wiki \\
\hline\hline
1 &   1.9155 &   2.2245&   0.9039 &   0.9862 &    0.7567 &   0.5507 &   0.5299 \\
4 &   4.1461 &   4.6099&   1.6244 &   2.8327 &    1.8832 &   1.4509 &   1.1509 \\
8 &   6.1157 &   7.6151&   2.0694 &   4.3836 &    3.2306 &   2.5697 &   1.9372 \\
16 &  11.7033 &  13.1351&   3.2559 &   6.2161 &   5.5284 &   4.6015 &   3.5561 \\
18 &  11.5580 &  14.1792&   4.7639 &   6.7627 &   6.1663 &   5.5270 &   3.9362 \\
\hline
\end{tabular}
\end{small}
\caption{\small 
  Cost of partitioning and allocation.
  The table shows the ratio of the time that \cyc{} consumes for partition and allocation over the time that \hog{} takes for 1 full pass over the dataset.
  On 18 threads, \cyc{} takes between 4-14 \hog{} epochs to perform partitioning.
  Note however, this computational effort is only required once per dataset.
}
\label{tab:cccost}
\end{center}
\vskip -0.1in
\end{table}

\COND{}{\vspace{-0.4cm}}
\paragraph{Stochastic updates running time}
When performing stochastic updates, \cyc{} has better cache locality and coherence, but requires synchronization after each batch.
Table \ref{tab:sutime} shows the time for each method to complete a single pass over the dataset, only with respect to stochastic updates (i.e., here we factor out the partitioning time).
In most cases, \cyc{} is faster than \hog{}.
In the cases where \cyc{} is not faster, the overheads of synchronizing outweigh the gains from better cache locality and coherency.
However, in some of these cases, synchronization can help by preventing errors due to asynchrony that lead to worse convergence, thus allowing \cyc{} to use larger stepsizes and maximize convergence speed.

\begin{table}[H]
\begin{center}
\begin{small}
\begin{tabular}{@{}|@{}c@{}|cc|cc|cc|cc|cc|cc|c@{}c@{}|@{}@{}}
\hline
\#     & \multicolumn{2}{|c|}{Mat. Comp.} & \multicolumn{2}{|c|}{Mat. Comp.} & \multicolumn{2}{|c|}{Word2Vec} & \multicolumn{2}{|c|}{Graph Eig.} & \multicolumn{2}{|c|}{Graph Eig.} & \multicolumn{2}{|c|}{Least Squares } & \multicolumn{2}{|c|}{Least Squares }\\
threads      & \multicolumn{2}{|c|}{SGD} & \multicolumn{2}{|c|}{$\ell_2$-SGD} & \multicolumn{2}{|c|}{SGD} & \multicolumn{2}{|c|}{SVRG} & \multicolumn{2}{|c|}{SVRG} & \multicolumn{2}{|c|}{SAGA} & \multicolumn{2}{|c|}{SAGA} \\
 & \multicolumn{2}{|c|}{MovieLens} & \multicolumn{2}{|c|}{MovieLens} & \multicolumn{2}{|c|}{EN-Wiki} & \multicolumn{2}{|c|}{NH2010} & \multicolumn{2}{|c|}{DBLP} & \multicolumn{2}{|c|}{NH2010} & \multicolumn{2}{|c|}{DBLP}\\
\cline{2-15}
& Cyc & Hog & Cyc & Hog & Cyc & Hog & Cyc & Hog & Cyc & Hog & Cyc & Hog & Cyc & Hog \\
\hline
1 &    \textbf{2.76} &   2.87 &   \textbf{3.69} &   3.84 &   \textbf{9.85} &  10.25 &   0.07 &   0.07 &           11.54 &  \textbf{11.50} &   0.04 &   0.04 &   \textbf{5.01} &           5.25 \\
4 &    \textbf{1.00} &   1.17 &   \textbf{1.27} &   1.51 &   \textbf{2.98} &   3.35 &   0.04 &   0.04 &   \textbf{4.60} &           4.81  &   0.03 &   0.03 &   \textbf{1.93} &           1.96 \\
8 &    \textbf{0.57} &   0.68 &   \textbf{0.71} &   0.86 &   \textbf{1.61} &   1.89 &   0.03 &   0.03 &   \textbf{2.86} &           3.03  &   0.01 &   0.01 &           1.04  &   \textbf{1.03}\\
16 &   \textbf{0.35} &   0.40 &   \textbf{0.42} &   0.48 &   \textbf{0.93} &   1.11 &   0.02 &   0.02 &   \textbf{2.03} &           2.15  &   0.01 &   0.01 &           0.59  &   \textbf{0.55}\\
18 &   \textbf{0.32} &   0.36 &   \textbf{0.37} &   0.40 &   \textbf{0.86} &   1.03 &   0.02 &   0.02 &   \textbf{1.92} &           2.01  &   0.01 &   0.01 &           0.52  &   \textbf{0.51}\\
\hline
\end{tabular}
\end{small}
\caption{\small Time, in seconds, to complete one epoch (i.e. full pass of stochastic updates over the data) by \cyc{} and \hog{}.
Lower times are highlighted in boldface.
\cyc{} is usually faster than \hog{}, due to its better cache locality and coherence properties.
}
\label{tab:sutime}
\end{center}
\vskip -0.1in
\end{table}

\subsection{Diminishing stepsizes}
In the previous experiments we used constant stepsizes.
Here, we investigate the behavior of \cyc{} and \hog{} in the regime where we decrease the stepsize after each epoch.
In particular, we ran the matrix completion experiments with SGD (with and without regularization), where we multiplicatively updated the stepsize by 0.95 after each epoch.
The convergence and speedup plots are given in Figure \ref{fig:matcom_reducestepsize}.
\cyc{} is able to achieve a speedup of up to $6-7\times$ on $16-18$ threads.
On the other hand, \hog{} is performing worse comparatively to its performance with constant stepsizes (Figure \ref{fig:mc_10m_may_rsg:speedups}).
The difference is more significant on regularized SGD, where we have to perform lazy updates (Appendix \ref{app:lazy}), and \hog{} fails to achieve the same optimum as \cyc{} with multiple threads.
Thus, on 18 threads, \hog{} obtains a maximum speedup of $3\times$, whereas \cyc{} attains a speedup of $6.5\times$.

\begin{figure}[h]
   \hspace{0.1cm} \includegraphics[width=0.5\textwidth]{images/matplotlib_plots/time_loss_legend_48_3}\\
    \subfigure[Convergence, SGD]{\includegraphics[width=0.245\textwidth]{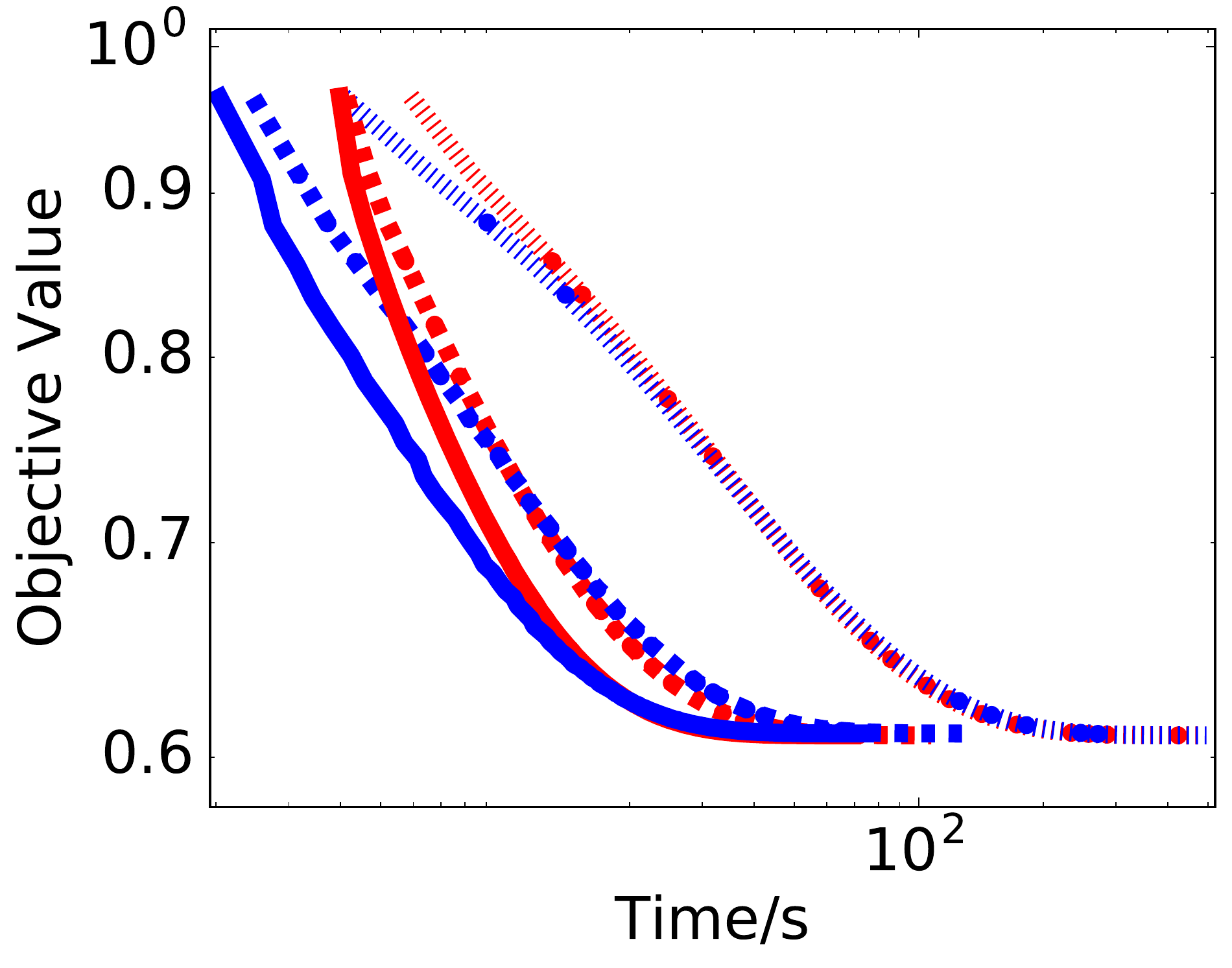}}
    \subfigure[Convergence, $\ell_2$-SGD]{\includegraphics[width=0.245\textwidth]{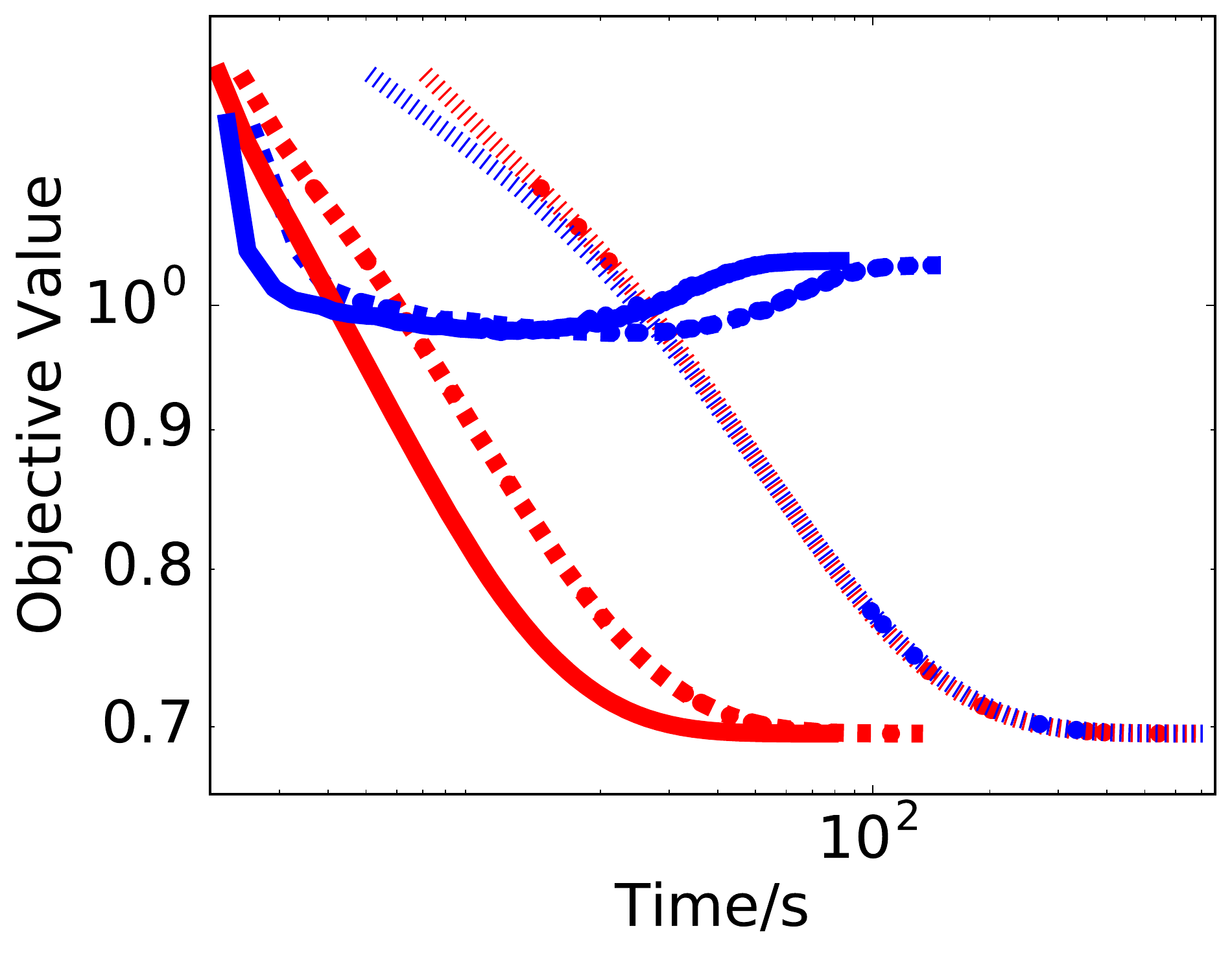}}
    \subfigure[Speedup, SGD]{\includegraphics[width=0.235\textwidth]{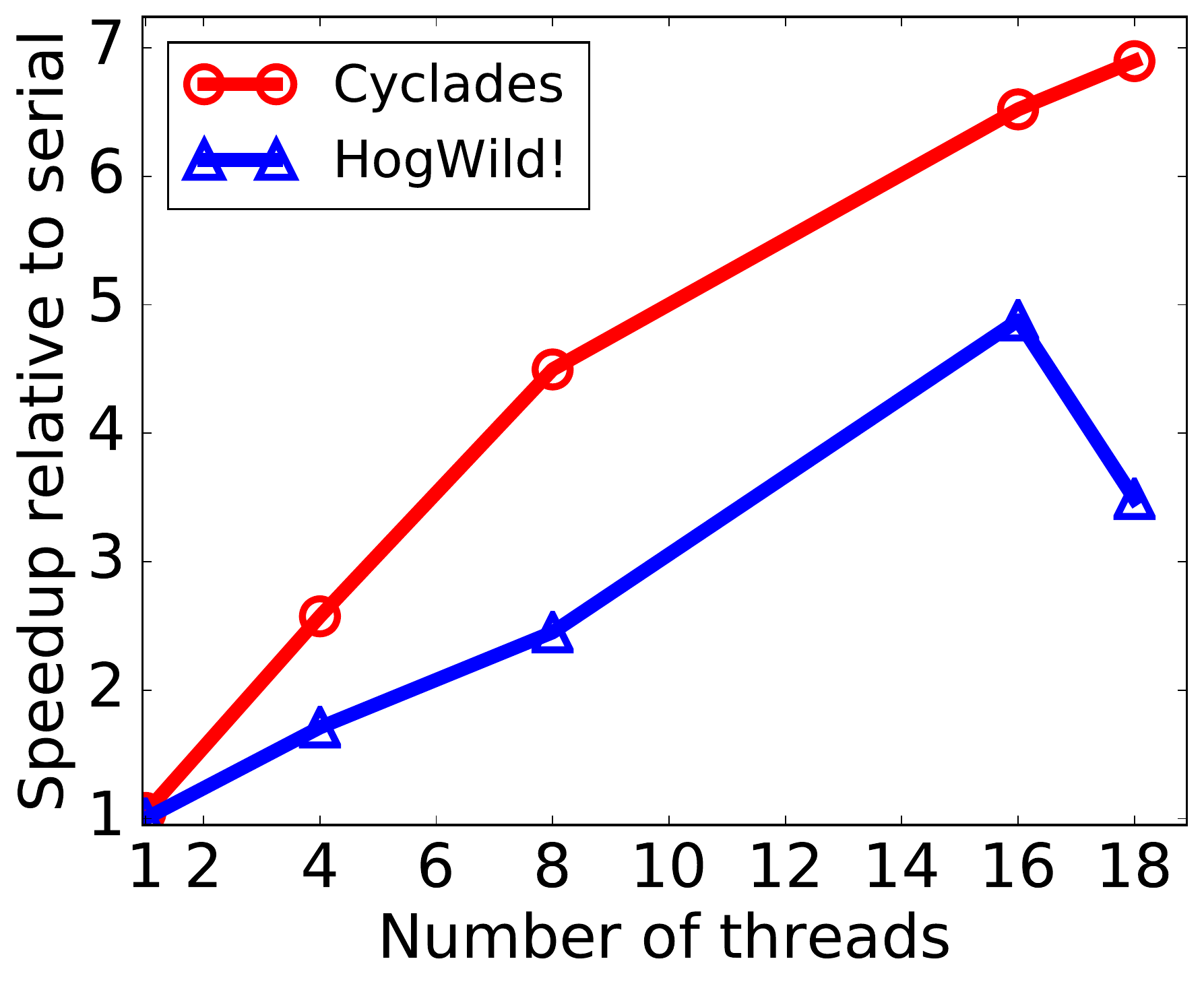}}
    \subfigure[Speedup, $\ell_2$-SGD]{\includegraphics[width=0.235\textwidth]{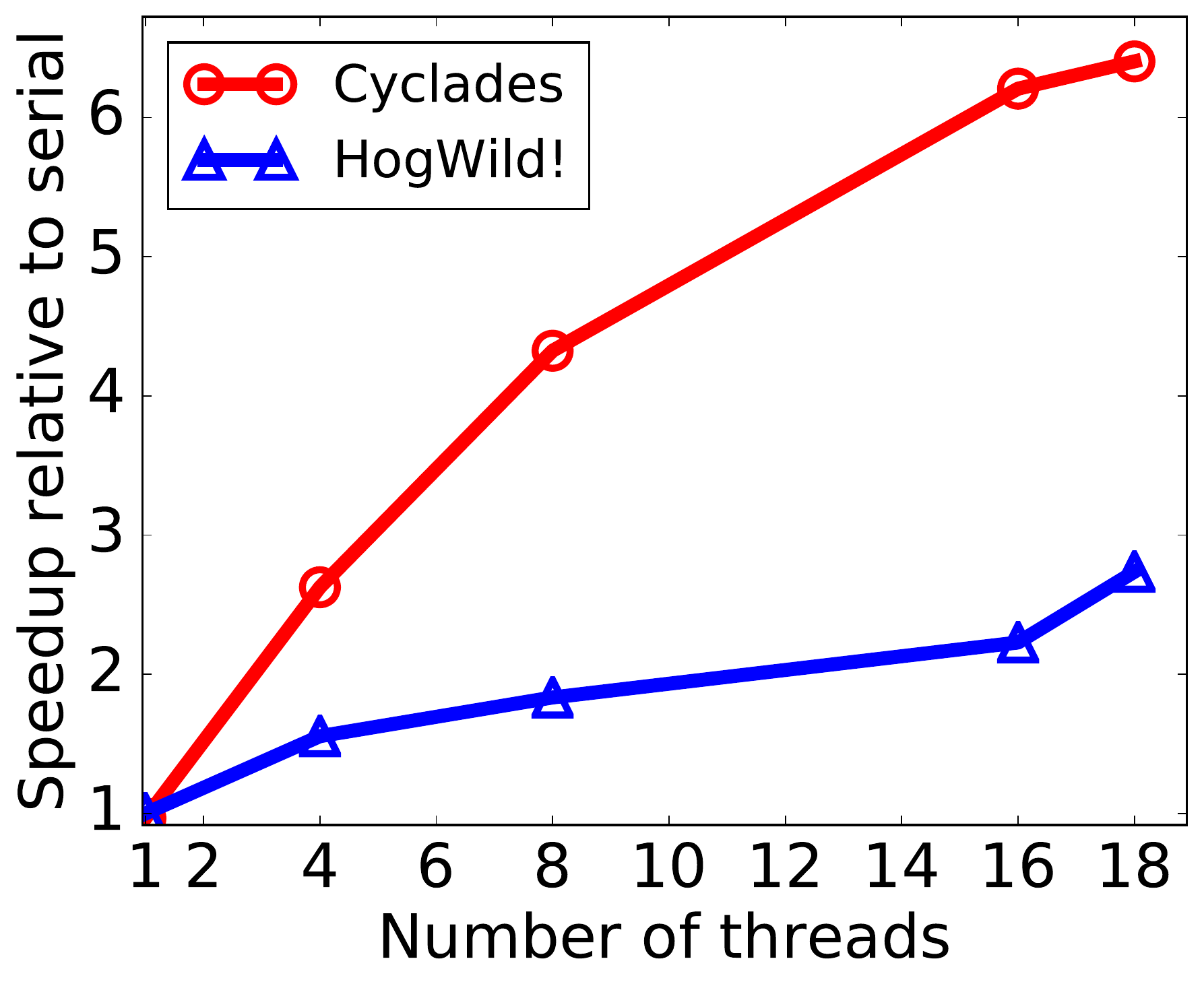}}
  \caption{\small  
 Convergence and speedups for SGD and weighted SGD with diminishing stepsizes for the matrix completion on the MovieLens dataset.
  In this case, \cyc{} outperforms \hog{} by achieving up to 6-7x speedup, when \hog{} achieves at most 5x speedup for 16-18 threads.
  For the weighted SGD algorithm, we used lazy updates (Appendix \ref{app:lazy}), in which case \hog{} on multiple threads gets to a worse optimum.
  }
  \label{fig:matcom_reducestepsize}
\end{figure}

\subsection{Binary Classification and Dense Coordinates}

\begin{wrapfigure}{r}{0.5\textwidth}
\begin{small}
\begin{center}
\begin{tabular}{@{}|@{}c|c|c@{}|@{}}
\hline
Filtering $\%$ & \#  filtered  features &  \#  remaining  features  \\
\hline
0.048\% &  1,555 & 3,228,887\\
0.047\% &  1,535 & 3,228,907\\
0.034\% &  1,120 & 3,229,322\\
0.028\% &    915 & 3,229,527\\
0.016\% &    548 & 3,229,894\\
\hline
\end{tabular}
\caption{\small Filtering of features in URL dataset.
with a total of 3,230,442 features before filtering.
The maximum percentage of features filtered is less than 0.05\%.
}
\label{table:filter}
\end{center}
\end{small}
\vspace{-0.5cm}
\end{wrapfigure}
In addition to the above experiments, here we explore settings where \cyc{} is expected to perform poorly due to the inherent density of updates (i.e., for data sets with dense features).
In particular, we test \cyc{} on a classification problem for text based data, where a few features appear in most data points.
Specifically, we run classification for the URL dataset \cite{ma2009identifying} contains $\sim$ 2.4M URLs, labeled as either benign or malicious, and 3.2M features, including bag-of-words representation of tokens in the URL.

\begin{wrapfigure}{l}{0.45\textwidth}
  \centerline{\includegraphics[width=0.35\textwidth]{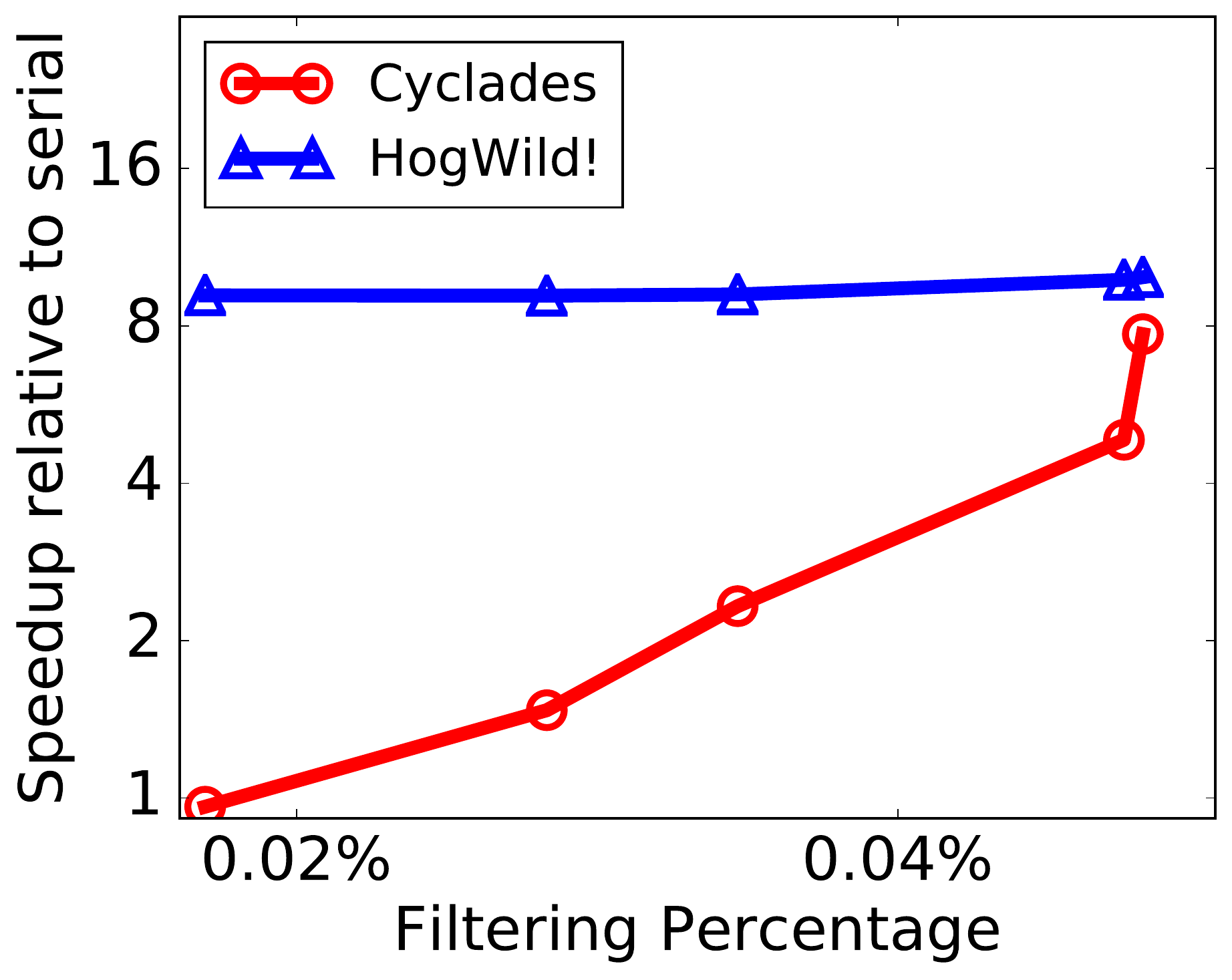}}
  \caption{\small Speedups of \cyc{} and \hog{} on 16 threads, for different percentage of dense features filtered.
  When only a very small number of features are filtered, \cyc{} is almost serial.
  However, as we increase the percentage from $0.016\%$ to $0.048\%$, the speedup of \cyc{} improves and almost catches up with \hog{}.
  }
  \label{fig:urlspeedup_all}
\vspace{-0.7cm}
\end{wrapfigure}

For this classification task, we used a logistic regression model, trained using SGD.
By its power-law nature, the dataset consists of a small number of extremely dense features which occur in nearly all updates.
Since \cyc{} explicitly avoids all conflicts, for these dense cases it will have a schedule of SGD updates that leads to poor speedups.
However, we observe that most conflicts are caused by a small percentage of the densest features.
If these features are removed from the dataset, \cyc{} is able to obtain much better speedups.
To that end, we ran \cyc{} and \hog{} after filtering the densest $0.016\%$ to $0.048\%$ of features.
The number of features that are filtered are shown in Table \ref{table:filter}.

The speedups that are obtained by \cyc{} and \hog{} on 16 threads for different filtering percentages are shown in Figure \ref{fig:urlspeedup_all}.
Full results of the experiment are presented in Figure~\ref{fig:url_all} and in App. \ref{app:exptresults}.
\cyc{} fails to get much speedup when nearly all the features are used,
however, as more dense features are removed, \cyc{} obtains a better speedup, almost equalling \hog{}'s speedup when $0.048\%$ of the densest features are filtered.

\begin{figure*}[t]
  \centering
    \includegraphics[width=0.7\textwidth]{images/matplotlib_plots/time_loss_legend_48_3}\\
    \subfigure[Convergence, 0.016\%]{\includegraphics[width=0.23\textwidth]{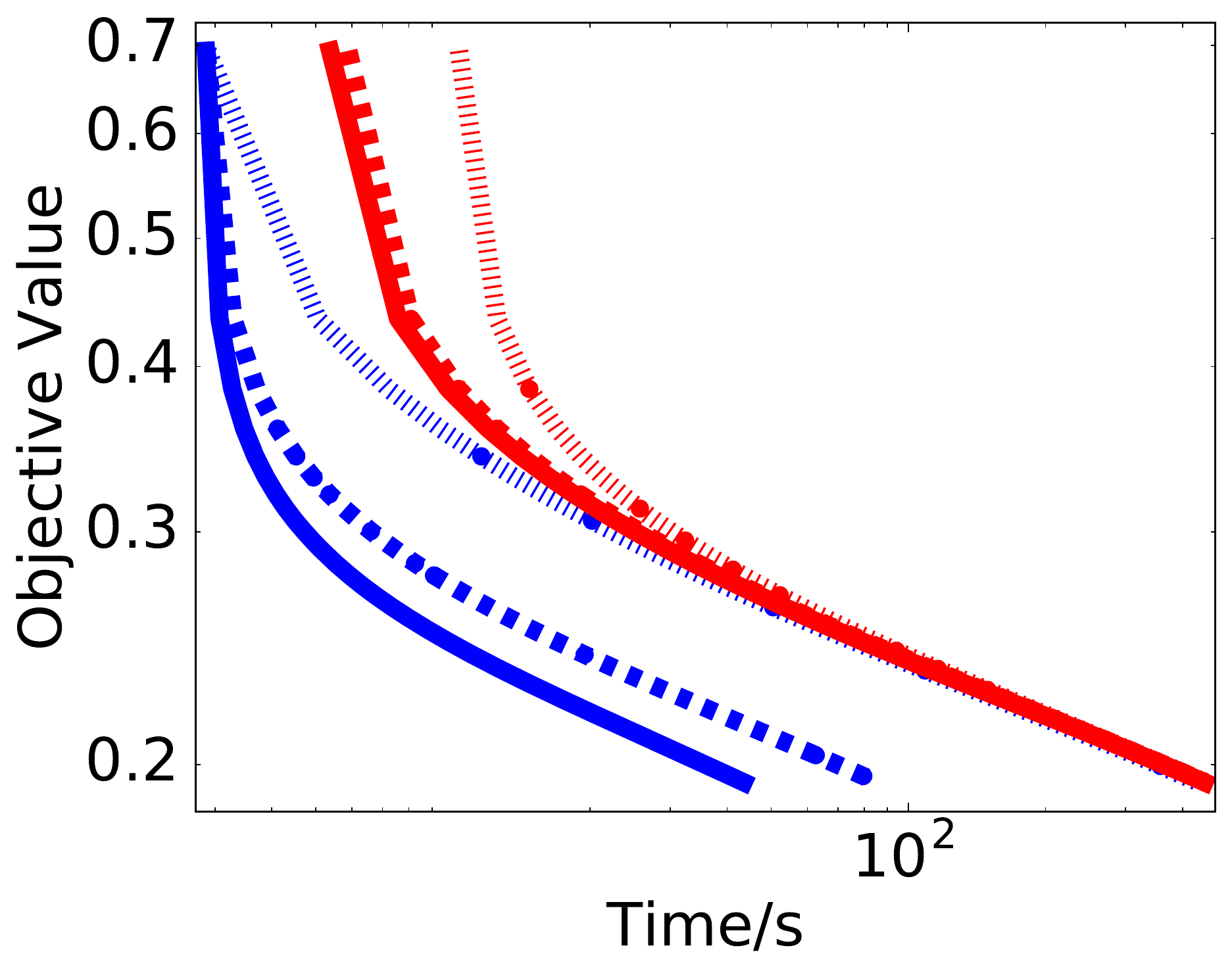}}
    \subfigure[Convergence, 0.034\%]{\includegraphics[width=0.23\textwidth]{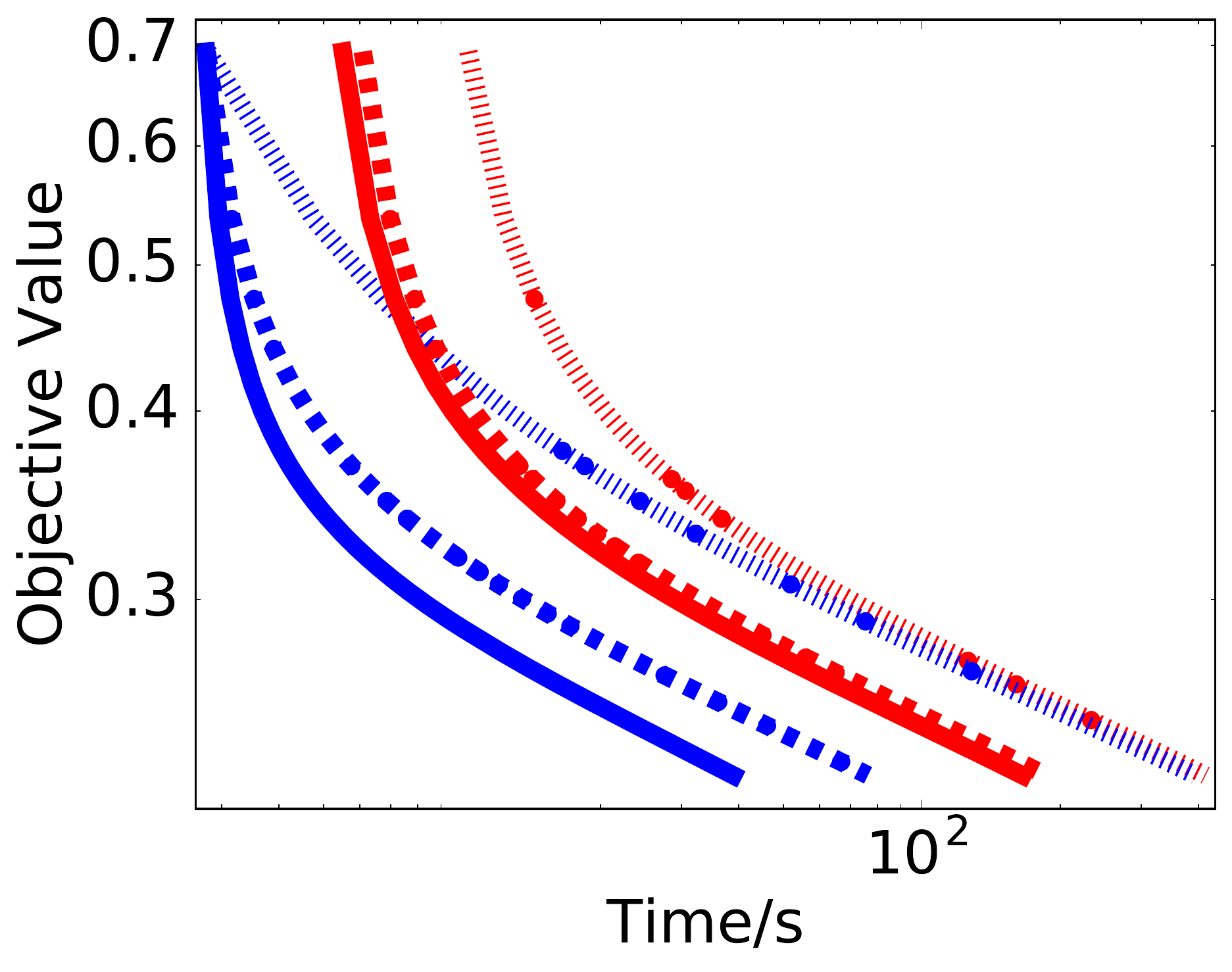}}
    \subfigure[Convergence, 0.047\%]{\includegraphics[width=0.23\textwidth]{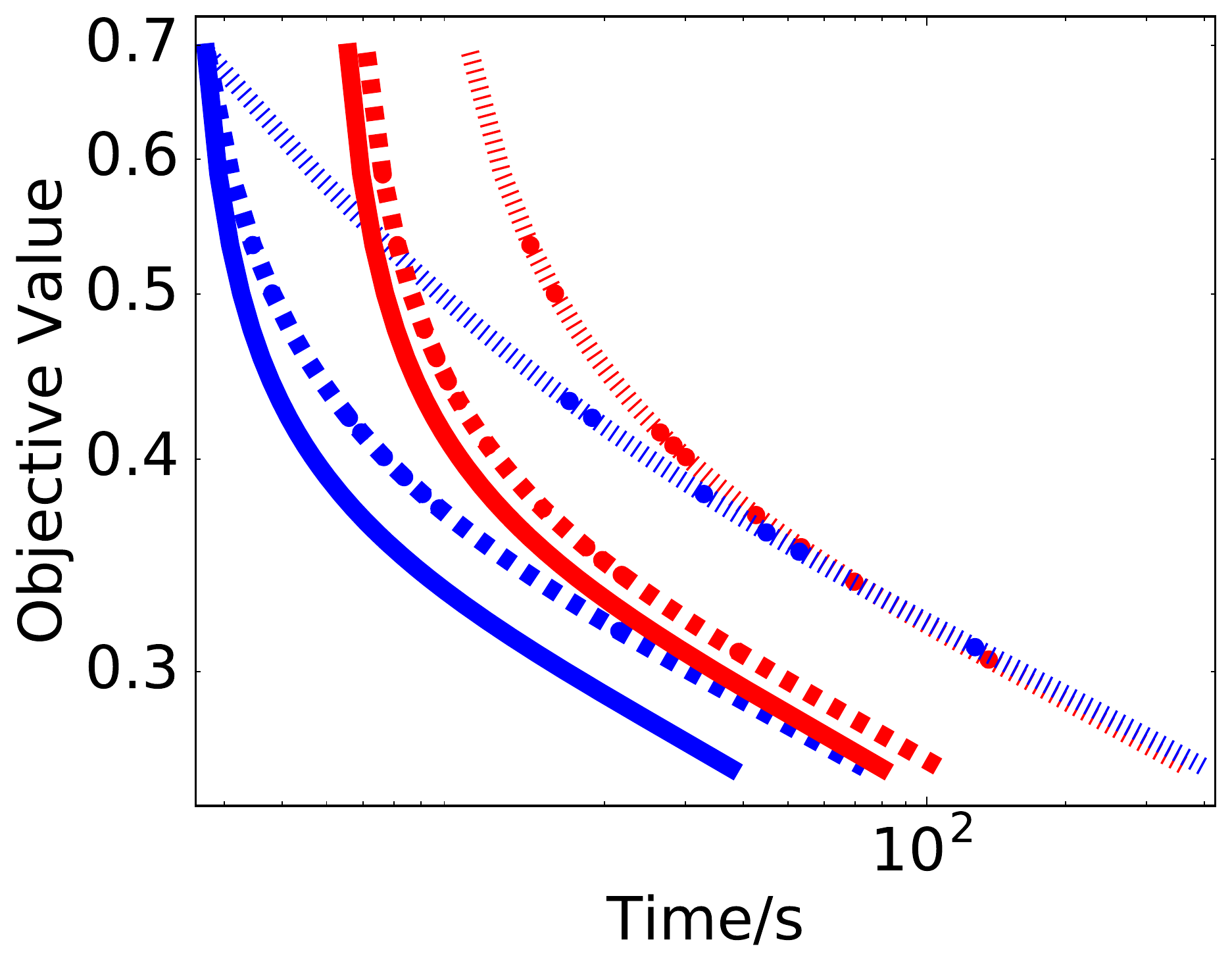}}
    \subfigure[Convergence, 0.048\%]{\includegraphics[width=0.23\textwidth]{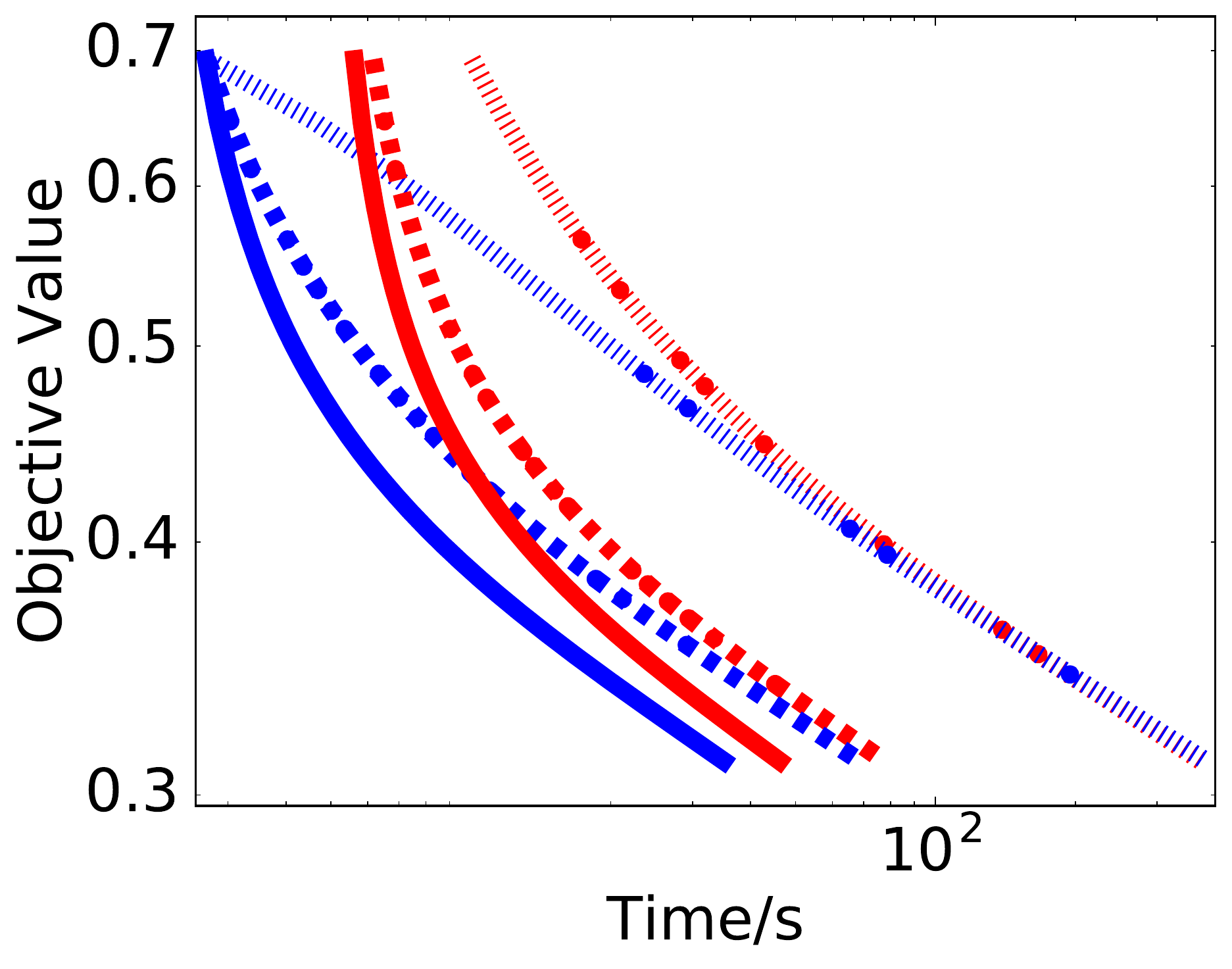}}\\
    \subfigure[Speedup, 0.016\%]{\includegraphics[width=0.23\textwidth]{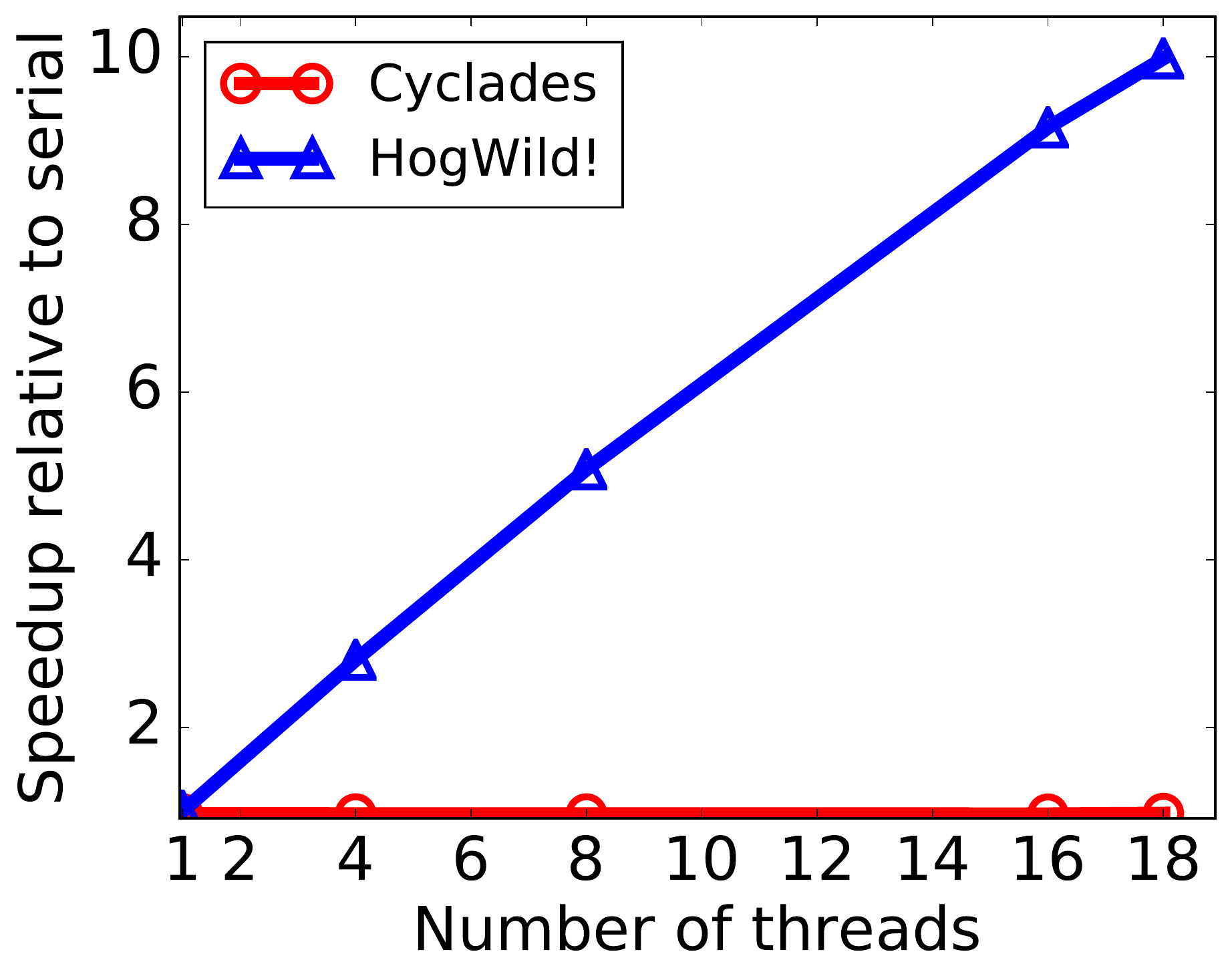}}
    \subfigure[Speedup, 0.034\%]{\includegraphics[width=0.23\textwidth]{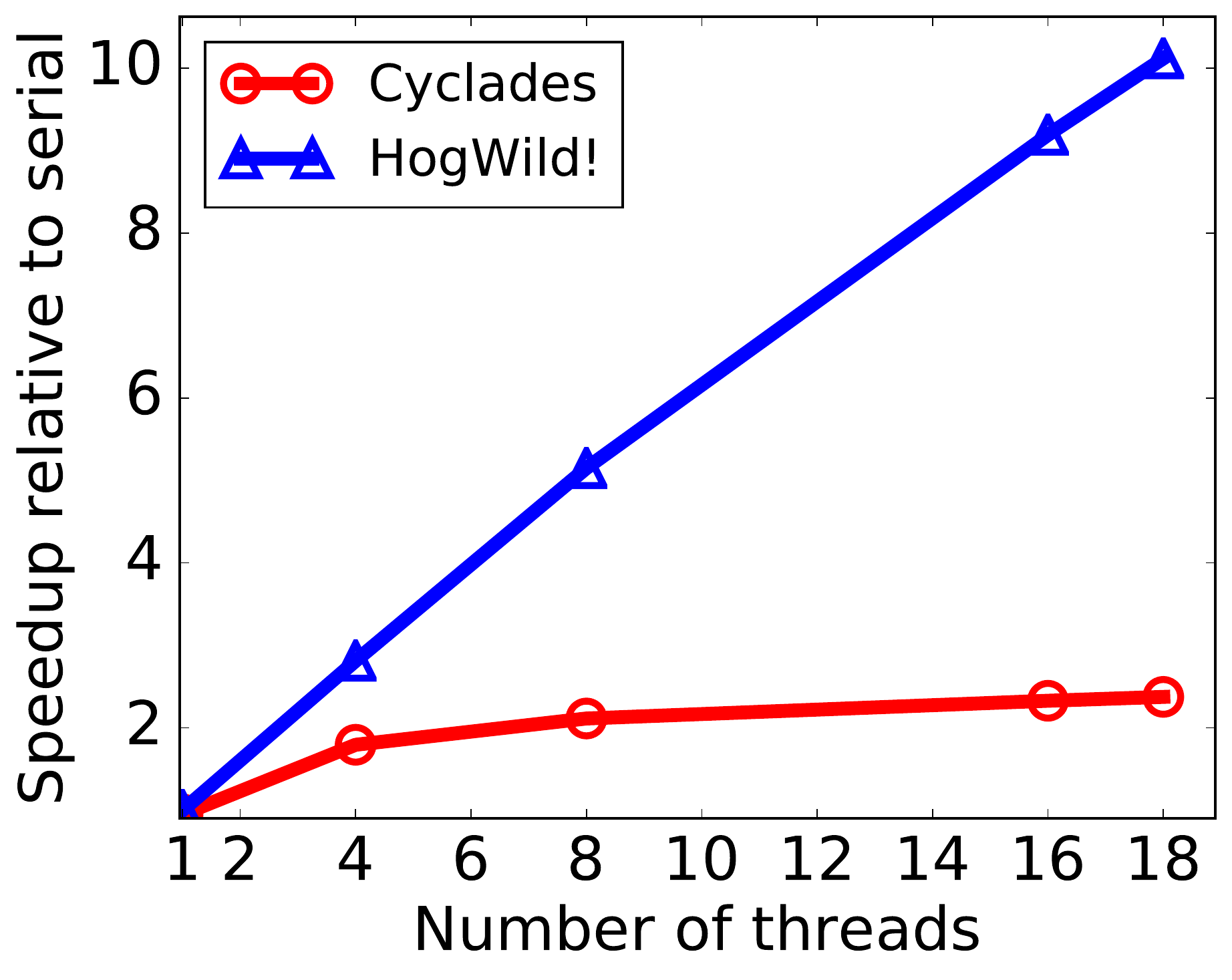}}
    \subfigure[Speedup, 0.047\%]{\includegraphics[width=0.23\textwidth]{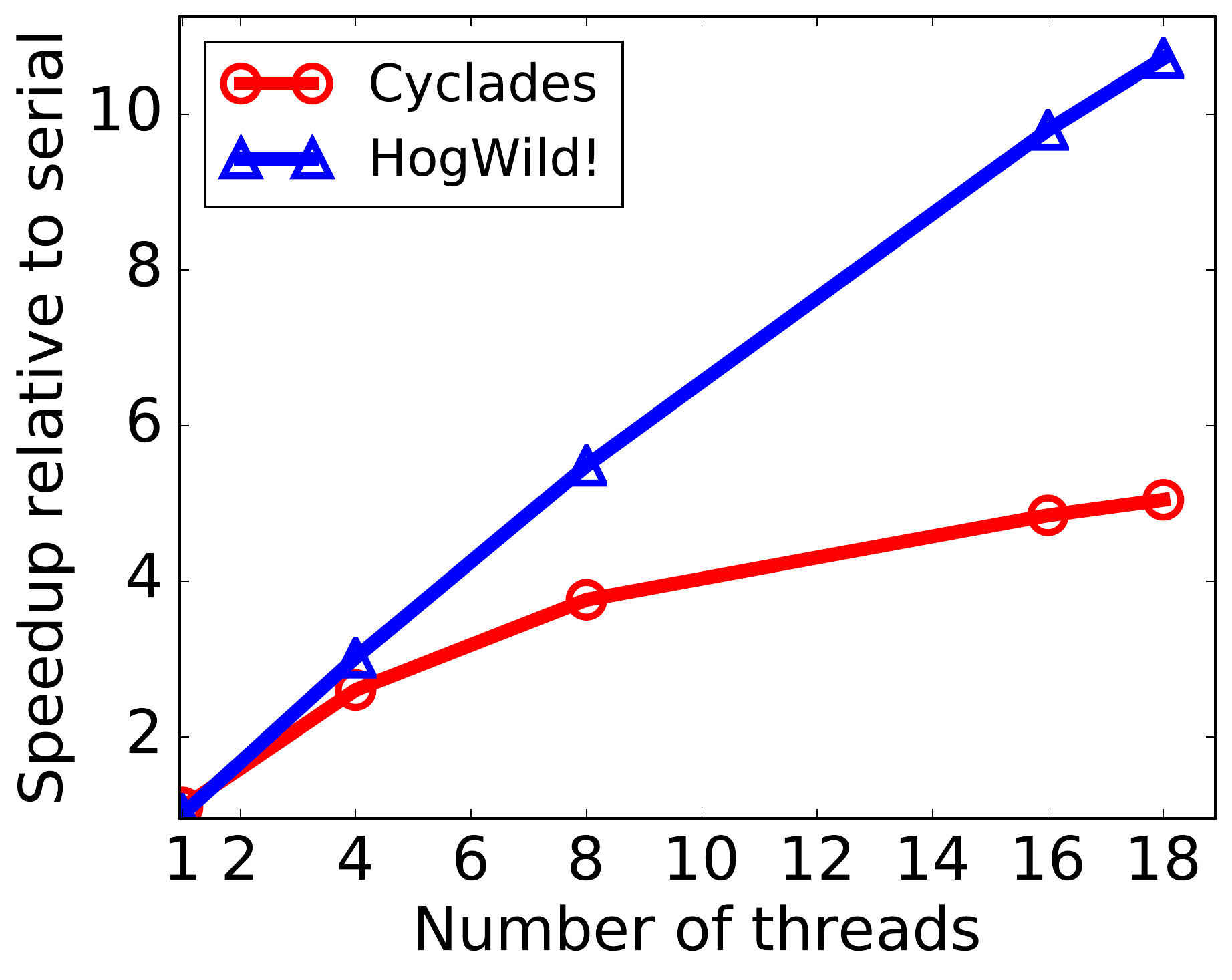}}
    \subfigure[Speedup, 0.048\%]{\includegraphics[width=0.23\textwidth]{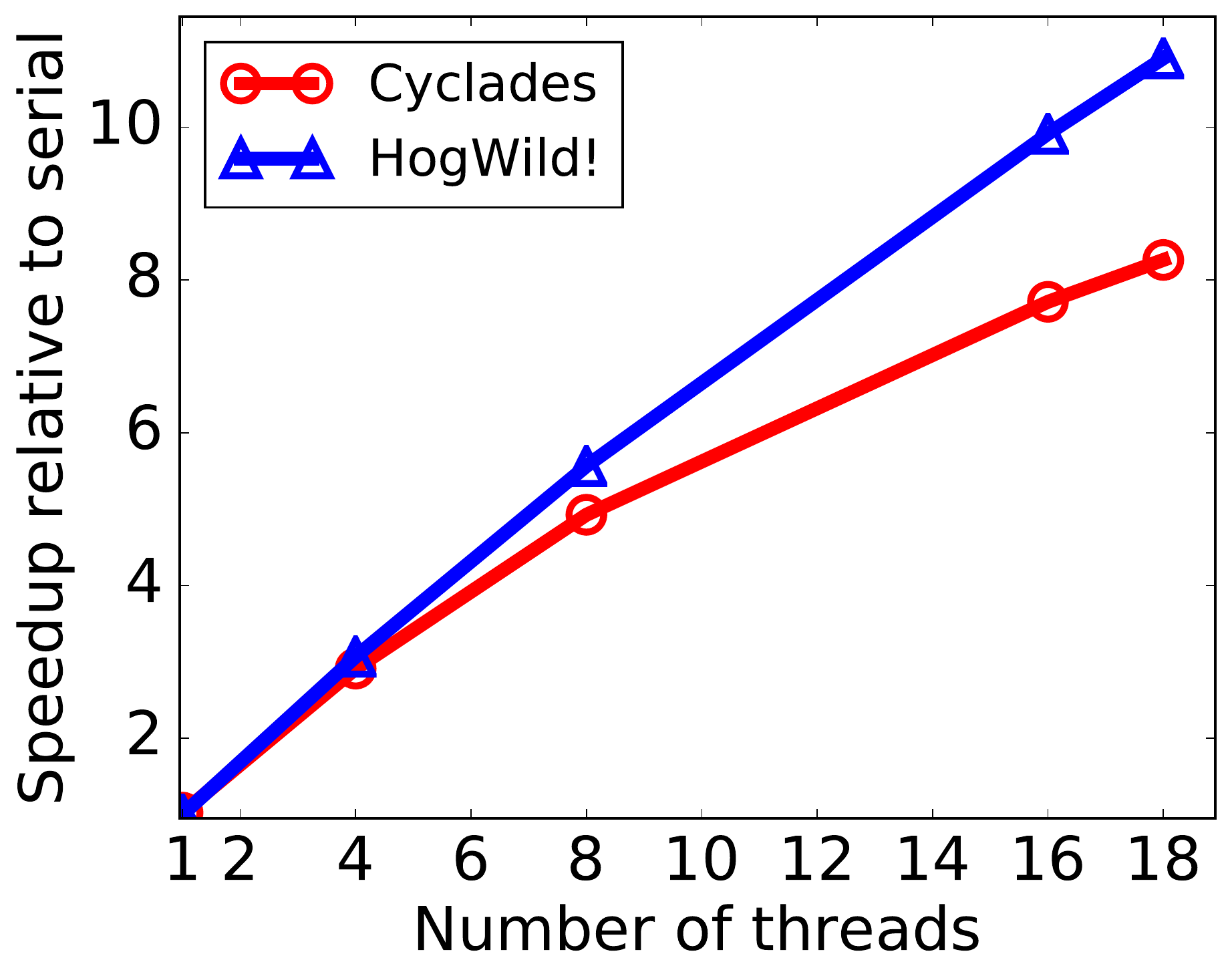}}
  \caption{Convergence and speedups of \cyc{} and \hog{} on 1, 4, 8, 16, 18 threads, for different percentage of dense features filtered.
  }
  \label{fig:url_all}
\vspace{-0.2cm}
\end{figure*}

\section{Related work}
\label{sec:prior}

Parallel stochastic optimization has been studied under various guises, with literature stretching back at least to the late 60s
\cite{chazan1969chaotic}.
The end of Moore's Law coupled with recent advances in parallel and distributed computing technologies have triggered renewed interest
\cite{zinkevich2009slow,
zinkevich2010parallelized,
gemulla2011large,
agarwal2011distributed,
richtarik2012parallel,
jaggi2014communication}
in the theory and practice of this field.
Much of this contemporary work is built upon the foundational work of Bertsekas, Tsitsiklis et al.
\cite{bertsekas1989parallel,
tsitsiklis1986distributed}.

\cite{niu2011hogwild} introduced \HW{}, a completely lock-free and asynchronous parallel stochastic gradient descent (SGD), in shared-memory multicore systems.
Inspired by \HW{}'s success at achieving nearly linear speedups for a variety of machine learning tasks, several authors developed other lock-free and asynchronous optimization algorithms, such as parallel stochastic coordinate descent \cite{liu2014asynchronous1,liu2015asynchronous}.
Additional work in first order optimization and beyond
\cite{duchi2013estimation,
wang2014asynchronous,
hong2014distributed,
hsieh2015passcode,
feyzmahdavian2015asynchronous},
extending to parallel iterative linear solvers \cite{liu2014asynchronous2,avron2014revisiting}, has further demonstrated that linear speedups are generically possible in the asynchronous shared-memory setting.
Moreover, \cite{reddi2015variance} proposes an analysis for asynchronous parallel, but dense, SVRG, under assumptions similar to those found in \cite{niu2011hogwild}.
The authors of
\cite{de2015taming} offer a new analysis for the ``coordinate-wise'' update version of \HW{} using martingales, with similar assumptions to \cite{niu2011hogwild}, that however can be applied to some non-convex problems. 
Furthermore, \cite{lian2015asynchronous} presents an analysis for stochastic gradient methods on smooth, potentially nonconvex functions.
Finally, \cite{peng2015arock} introduces a new framework for analyzing coordinate-wise fixed point stochastic iterations.

Recently, \cite{mania2015perturbed} provided a new analysis for asynchronous stochastic optimization by interpreting the effects of asynchrony as noise in the iterates.
This perspective of asycnhrony as noise in the iterates was used in \cite{pan2015parallel} to analyze a combinatorial graph clustering algorithm.

Parallel optimization algorithms have also been developed for specific subclasses of problems, including L1-regularized loss minimization \cite{bradley2011parallel} and matrix factorization \cite{recht2013parallel}.

Other machine learning algorithms that have been parallelized using coordination and concurrency control, including non-parametric clustering \cite{pan2013optimistic}, submodular maximization \cite{pan2014parallel}, and correlation clustering \cite{pan2015parallel}.

Sparse, graph-based parallel computation are supported by systems like GraphLab \cite{low2014graphlab} and PowerGraph \cite{gonzalez2012powergraph}.
These frameworks require computation to be written in a specific programming model with associative, commutative operations.
GraphLab and PowerGraph support serializable execution via locking mechanisms
This is in contrast to our partition-and-allocate coordination which allows us to provide guarantees on speedup.

Parameter servers \cite{ho2013more, li2014scaling} are frameworks supporting distributed, asynchronous computation.
Convergence is only guaranteed for some specific algorithms, namely SGD, but not in general as execution is not serializable.
The focus of parameter servers is in the distributed setting, where the challenges are different than in the multicore setting.

\section{Conclusion}
\label{sec:conclusion}
We presented \cyc{}, a general framework for lock-free parallelization of stochastic optimization algorithms, while maintaining serial equivalence.
Our framework can be used to parallelize a large family of stochastic updates algorithms in a conflict-free manner, thereby ensuring the parallelized algorithm produces the same result as its serial counterpart.
Theoretical properties, such as convergence rates, are therefore preserved by the \cyc{}-parallelized algorithm, and
we provide a single unified theoretical analysis that guarantees near linear speedups.

By eliminating conflicts across processors within each batch of updates, \cyc{} is able to avoid all asynchrony errors and conflicts, and leads to better cache locality and cache coherence than \hog{}.
These features of \cyc{} translate to near linear speedups in practice, where it can outperform \HW{}-type of implementations by up to a factor of $5\times$, in terms of speedups

In the future, we intend to explore hybrids of \cyc{} with \hog{}, pushing the boundaries of what is possible in a shared-memory setting.
We are also considering solutions for scaling \emph{out} in a distributed setting, where the cost of communication is significantly higher.

\section*{Acknowledgements}
BR is generously supported by ONR awards  N00014-14-1-0024, N00014-15-1-2620,
and N00014-13-1-0129, and NSF awards CCF-1148243 and CCF-1217058.
XP's work is also supported by a DSO National Laboratories Postgraduate Scholarship.
This research
is supported in part by NSF CISE Expeditions Award CCF-1139158, LBNL Award
7076018, and DARPA XData Award FA8750-12-2-0331, and gifts from Amazon Web
Services, Google, SAP, The Thomas and Stacey Siebel Foundation, Adatao, Adobe,
Apple, Inc., Blue Goji, Bosch, C3Energy, Cisco, Cray, Cloudera, EMC2, Ericsson,
Facebook, Guavus, HP, Huawei, Informatica, Intel, Microsoft, NetApp, Pivotal,
Samsung, Schlumberger, Splunk, Virdata and VMware.

\bibliography{cyclades}
\bibliographystyle{alpha}

\appendix
\section{Algorithms in the Stochastic Updates family}
\label{sec:su}

Here we show that several algorithms belong to the Stochastic Updates (SU) family.
These include the well-known stochastic gradient descent, iterative linear solvers, stochastic PCA and others, as well as combinations of weight decay updates, variance reduction methods, and more.
Interestingly, even some combinatorial graph algorithms fit under the SU umbrella, such as the maximal independent set, greedy correlation clustering, and others.
We visit some of these algorithms below.

\paragraph{Stochastic Gradient Descent (SGD)}

Given $n$ functions $f_1, \ldots, f_n$, one often wishes to minimize the average of these functions:
$$\min_{\vecx} \frac{1}{n} \sum_{i=1}^n f_i(\vecx).$$
A popular algorithm to do so ---even in the case of non-convex losses--- is the stochastic gradient descent:
\begin{equation*}
\vecx_{k+1} = \vecx_{k} - \gamma_k\cdot \nabla f_{i_k}(\vecx_k).\nonumber
\end{equation*}
In this case, the distribution $\mathcal{D}$ for each sample $i_k$ is usually a with or without replacement uniform sampling among the $n$ functions.
For this algorithm the conflict graph between the $n$ possible different updates is completely determined by the support of the gradient vector $\nabla f_{i_k}(\vecx_k)$.

\paragraph{Weight decay and regularization}
Similar to SGD, in some cases we might wish to regularize the objective with an $\ell_2$ term and solve instead the following optimization:
$$\min_{\vecx} \frac{1}{n} \sum_{i=1}^n f_i(\vecx) + \frac{\eta}{2} ||\vecx||_2^2.$$ 
In this case, the update is a weighted version of the ``non-regularized" SGD:
\begin{equation*}
\vecx_{k+1} = (1 - \gamma_k \eta)\cdot \vecx_{k} - \gamma_k\cdot \nabla f_{i_k}(\vecx_k). \nonumber
\end{equation*}
The above stochastic update algorithm can be also be written in the SU language.
Although here for each single update the entire model has to be updated with the new weights, we show below that with a simple technique it can be equivalently expressed so that each update is sparse and the support is again determined by the gradient vector $\nabla f_{i_k}(\vecx_k)$.

\paragraph{First order techniques with variance reduction}
Variance reduction is a technique that is usually employed for (strongly) convex problems, where we wish to minimize the variance of SGD in order to achieve better  rates.
A popular way to do variance reduction is either through SVRG or SAGA, where a ``memory'' factor is used in the computation of each gradient update rule.
For SAGA we have
\begin{align*}
&{\bf x}_{k+1} = {\bf x}_k - \gamma \cdot \left(\nabla f_{s_k} ({\bf x}_k) - {\bf g}_{s_k} + \frac{1}{n} \sum_{i=1}^n {\bf g}_i\right)\\
&{\bf g}_{s_k} = \nabla f_{s_k} ({\bf x}_k).
\end{align*}
For SVRG the update rule is 
\begin{equation*}
\vecx_{k+1} = \vecx_{k} - \gamma_k \cdot (\nabla f_{i_k}(\vecx_k)-\nabla f_{i_k}(\vecy)+\vecg) \nonumber
\end{equation*}
where ${\bf g} = \nabla f(\vecy)$, and ${\bf y}$ is updated every $T$ iterations of the previous form to be equal to the last $\vecx_k$ iterate.
Again, although at first sight the above updates seem to be dense, we show below how we can equivalently rewrite them so that the update-conflict graph is completely determined by the support of the gradient vector $\nabla f_{i_k}(\vecx_k)$.

\paragraph{Combinatorial graph algorithms}
Interestingly, even some combinatorial graph problems can be phrased in the SU language: greedy correlation clustering  (The Pivot Algorithm~\cite{ailon2008aggregating}) and the maximal independent set (these are in fact identical algorithms).
In the case of correlation clustering, we are given a graph $G$ vertices joined with either positive or negative edges.
Here the objective is to create a number of clusters so that the number of vertex pairs that are sharing negative edges within clusters, plus the number of pairs that are sharing positive edges across clusters, is minimized.
For these cases, there exists a very simple algorithm that obtains a $3$ approximation for the above objective: randomly sample a vertex, create a cluster with that vertex and its neighborhood, remove that cluster from the graph, and repeat.
The above procedure is amenable to the following update rule:
\ST{Can we put some time indices below?}
\begin{align*}
[\vecx_{k+1}]_{N(v)} = \min([\vecx_{k}]_{N(v)}, v ) \nonumber
\end{align*}
where $\vecx$ is intialized to $\infty$, and at each iteration we sample $v$ uniformly among those with label $x_v=\infty$, and $N(v)$ denotes the neighborhood of a vertex in $G$.
Interestingly, we can directly apply the same guarantees of the main \cyc{} theorem here.
An optimized implementation of \cyc{} for correlation clustering was developed in \cite{pan2015parallel}.

To reiterate, all of the above algorithms, various combinations of them, and further extensions can be written in the language of SU, as presented in Alg. ~\ref{alg:SU}.

\subsection{Lazy Updates}
\label{app:lazy}

For the cases of weight decay/regularization, and variance reduction, we can reinterpret their inherently dense updates in an equivalent sparse form.
Let us consider the following generic form of updates:
\begin{align}
x_j \leftarrow (1-\mu_j) x_j - \nu_j + h_{ij}({\bf x}_{\mathcal{S}_i}) \label{eq:lazy}
\end{align}
where $h_{ij}({\bf x}_{\mathcal{S}_i}) = 0$ for all $j \not\in \mathcal{S}_i$.
Each stochastic update therefore reads from the set $\mathcal{S}_i$ but writes to every coordinate.
However, it is possible to make updates lazily only when they are required.
Observe that if $\tau_j$ updates are made, each of which have $h_{ij}({\bf x}_{\mathcal{S}_i}) = 0$, then we could rewrite these $\tau_j$ updates in closed form as
\begin{align}
x_j
&= (1-\mu_j)^{\tau_j} x_j - \nu_j \sum_{k=1}^{\tau_j} (1-\mu_j)^k\label{eq:lazy_catchup1}\\
&= (1-\mu_j)^{\tau_j} x_j - \frac{\nu_j}{\mu_j}(1-\mu_j) \left(1-(1-\mu_j)^{\tau_j}\right).\label{eq:lazy_catchup2}
\end{align}
This allows the stochastic updates to only write to coordinates in $\mathcal{S}_i$ and defer writes to other coordinates.
This procedure is described in Algorithm \ref{alg:lazy}.
With \cyc{} it is easy to keep track of $\tau_j$, since we know the serially equivalent order of each stochastic update.
On the other hand, it is unclear how a \hog{} approach would behave with additional noise in $\tau_j$ due to asynchrony.
In fact, \hog{} could possibly result in negative values of $\tau_j$, and in practice, we find that it is often useful to threshold $\tau_j$ by $\max(0, \tau_j)$.

\begin{wrapfigure}{R}{0.5\columnwidth}
\vspace{-0.3cm}
\begin{minipage}{0.46\columnwidth}
\begin{algorithm}[H]
   \caption{Lazy Stochastic Updates pseudo-algorithm}
\begin{algorithmic}[1]
\STATE Input: ${\bf x};\: f_1,\ldots, f_n;\: u_1,\ldots, u_n;\: g_1,\ldots, g_n;\: \mathcal{D};\:T$.
\STATE Initialize $\rho(j) = 0$.
\FOR {$t=1:T$}
  \STATE sample $i\sim \mathcal{D}$
  \STATE ${\bf x}_{\mathcal{S}_i}$ =  read coordinates $\setS_i$ from $\vecx$
  \FOR {$j \in \mathcal{S}_i$}
    \STATE $\tau_j = t - \rho(j) - 1$.
    \STATE $x_j \leftarrow (1-\mu_j)^{\tau_j} x_j - \nu_j \sum_{k=1}^{\tau_j} (1-\mu_j)^k$.
    \STATE $x_j \leftarrow (1-\mu_j) x_j - \nu_j + h_{ij}({\bf x}_{\mathcal{S}_i})$.
    \STATE $\rho(j) \leftarrow t$.
  \ENDFOR
\ENDFOR
\STATE {\bf Output:} ${\bf x}$
\end{algorithmic}
   \label{alg:lazy}
 \end{algorithm}
\end{minipage}
 \end{wrapfigure}

\paragraph{Weight decay and regularization}
The weighted decay SGD update is a special case of Eq \ref{eq:lazy}, with $\mu_j = \eta\gamma$ and $\nu_j = 0$.
Eq \ref{eq:lazy_catchup2} becomes $x_j \leftarrow (1-\eta\gamma)^{\tau_j} x_j$.

\paragraph{Variance reduction with sparse gradients}
Suppose $\nabla f_i({\bf x})$ is sparse, such that $[\nabla f_i({\bf x})]_j = 0$ for all $\bf x$ and $j \not\in \mathcal{S}_i$.
Then we can perform  SVRG and SAGA using lazy updates, with $\mu_j = 0$.
Just-in-time updates for SAG (a similar algorithm to SAGA) were introduced in \cite{schmidt2013minimizing}.
For SAGA, the update Eq \ref{eq:lazy_catchup1} becomes
\[
x_j \leftarrow x_j - \gamma\tau_j g_j
\]
where $g_j = \left[\frac{1}{n}\sum_{i=1}^n \mathbf{y}_{k,i}\right]_j$ is the $j$th coordinate of the average gradient.
For SVRG, we instead use $g_j = [\nabla f(\mathbf{y})]_j$.

\paragraph{SVRG with dense linear gradients}
Suppose instead that the gradient is dense, but has linear form $[\nabla f_i(\mathbf{x})]_j = \lambda_j x_j - \kappa_j + \tilde{h}_{ij}(\mathbf{x}_{\mathcal{S}_i})$, where $\tilde{h}_{ij}(\mathbf{x}_{\mathcal{S}_i}) = 0$ for $j \not\in \mathcal{S}_i$.
The SVRG stochastic update on the $j$th coordinate is then
\begin{align*}
x_j
&\leftarrow x_j - \gamma \left(
\lambda_j x_j - \kappa_j + \tilde{h}_{ij}(\mathbf{x}_{\mathcal{S}_i})
- \lambda_j y_j + \kappa_j - \tilde{h}_{ij}(\mathbf{y}_{\mathcal{S}_i})
+ g_j
\right)\\
&= (1-\gamma\lambda_j)x_j - \gamma g_j - \gamma \left(\tilde{h}_{ij}(\mathbf{x}_{\mathcal{S}_i}) - \tilde{h}_{ij}(\mathbf{y}_{\mathcal{S}_i})\right)
\end{align*}
where $g_j = [\nabla f(\mathbf{y})]_j$ as above.
This fits into our framework with $\mu_j = \gamma\lambda_j$, $\nu_j = \gamma g_j$, and $h_{ij}(\mathbf{x}_{\mathcal{S}_i}) = - \gamma \left(\tilde{h}_{ij}(\mathbf{x}_{\mathcal{S}_i}) - \tilde{h}_{ij}(\mathbf{y}_{\mathcal{S}_i})\right)$.
\section{With and Without Replacement Proofs}
\label{app:replacement}
In this Appendix, we show how the sampling and shattering Theorem~\ref{theo:Gsample1} can be restated for sampling with, or without replacement to establish Theorem~\ref{theo:ourGsample}.

Let us define three sequences of binary random variables $\{X_i\}_{i=1}^n, \{Y_i\}_{i=1}^n,$ and $\{Z_i\}_{i=1}^n$.
$\{X_i\}_{i=1}^n$ consists of $n$ i.i.d. Bernoulli random variables, each with probability $p$.
In the second sequence $ \{Y_i\}_{i=1}^n$, a random subset of $B$ random variables is set to $1$ without replacement.
Finally, in the third sequence $\{Z_i\}_{i=1}^n$, we draw $B$ variables with replacement, and we set them to $1$.
Here, $B$ is integer that satisfies the following bounds
$$(n+1)\cdot p-1\le B < (n+1)\cdot p.$$

Now, consider any  function $f$, that has a ``monotonicity" property:
$$f(x_1,\ldots,x_i,\ldots,  x_n) \ge f(x_1,\ldots,0,\ldots,  x_n), \text{ for all } i=1,\ldots,n.$$
Let us now define
\begin{align*}
\rho_X &= \Pr \left( f(X_1,\ldots, X_n) > C \right)\\
\rho_Y &= \Pr \left( f(Y_1,\ldots, Y_n) > C \right)\\
\rho_Z &= \Pr \left( f(Z_1,\ldots, Z_n) > C \right)
\end{align*}
for some number $C$, and
let us further assume that we have an upper bound on the above probability 
$$\rho_X\le \delta.$$ 
Our goal is to bound $\rho_Y$ and $\rho_Z$.
By expanding $\rho_X$ using the law of total probability we have
\begin{align}
\rho_X =\sum_{b=0}^n \Pr \left( f(X_1,\ldots, X_n ) > C \left| \sum_{i=1}^n X_i = b\right.\right) \cdot \Pr\left( \sum_{i=1}^n X_i = b\right)\nonumber=\sum_{b=0}^n q_b \cdot \Pr\left( \sum_{i=1}^n X_i = b\right)
\end{align}
where $q_b=\Pr \left( f(X_1,\ldots, X_n ) > C \left| \sum_{i=1}^n X_i = b\right.\right)$, denotes the probability that $f(X_1,\ldots, X_n ) > C$ given that a uniformly random subset of $b$ variables was set to $1$.
Moreover, we have 
\begin{align}
\rho_Y& =\sum_{b=0}^{n} \Pr \left( f(Y_1,\ldots, Y_n) > C \left| \sum_{i=1}^n Y_i = b\right.\right) \cdot \Pr\left( \sum_{i=1}^n Y_i =b\right )\nonumber\\
&\overset{(i)}{=}\sum_{b=0}^{n} q_b \cdot \Pr\left( \sum_{i=1}^n Y_i =b\right )\nonumber\\
&\overset{(ii)}{=} q_{B} \cdot 1
\end{align}
where $(i)$ comes form the fact that $\Pr \left( f(Y_1,\ldots, Y_n) > C \left| \sum_{i=1}^n Y_i = b\right.\right) $ is the same as the probability that that $f(X_1,\ldots, X_n ) > C$ given that a uniformly random subset of $b$ variables where set to $1$, and 
$(ii)$ comes from the fact that since we sample without replacement in $Y$, we have that $\sum_{i}^n Y_i = B $ always.

In the expansion of $\rho_X$, we can keep the $b=B$ term, and lower bound the probability to obtain:
\begin{align}
\rho_X &=\sum_{b=0}^n q_b \cdot \Pr\left( \sum_{i=1}^n X_i = b\right)\nonumber \\
&\ge q_{B} \cdot \Pr\left( \sum_{i=1}^n X_i = B\right) = \rho_Y \cdot \Pr\left( \sum_{i=1}^n X_i = B\right)
\label{eq:rho_X_LB}
\end{align}
since all terms in the sum are non-negative numbers.
Moreover, since $X_i$s are Bernoulli random variables, their sum $\sum_{i=1}^n X_i$ is Binomially distributed with parameters $n$ and $p$.
We know that the maximum of the Binomial pmf with parameters $n$ and $p$ occurs at $\Pr\left(\sum_i X_i = B\right)$ where $B$ is the integer that satisfies the upper bound mentioned above: 
$(n+1)\cdot p-1\le B < (n+1)\cdot p$.
Furthermore, the maximum value of the Binomial pmf, with parameters $n$ and $p$, cannot be less than the corresponding probability of a uniform element:
\begin{equation}
\Pr\left( \sum_{i=1}^n X_i =B\right) \ge \frac{1}{n}.
\label{eq:maxBinomial}
\end{equation}
If we combine \eqref{eq:rho_X_LB} and \eqref{eq:maxBinomial} we get
\begin{equation}
\rho_X\ge \rho_Y / n \Leftrightarrow \rho_Y \le  n\cdot \delta. 
\end{equation}
The above establish a relation between the without replacement sampling sequence $\{Y_i\}_{i=1}^n$, and the i.i.d. uniform sampling sequence $\{X_i\}_{i=1}^n$.

Then, for the last sequence $\{Z_i\}_{i=1}^n$ we have 
\begin{align}
\rho_Z& =\sum_{b=0}^{n} \Pr \left( f(Z_1,\ldots, Z_n) > C \left| \sum_{i=1}^n Z_i = b\right.\right) \cdot \Pr\left( \sum_{i=1}^n Z_i =b\right )\nonumber\\
& \overset{(i)}{=}\sum_{b=1}^{B} q_b \cdot \Pr\left( \sum_{i=1}^n Z_i =b\right )\\
& \overset{(ii)}{\le}  \left(\max_{1\le b\le B} q_b\right) \cdot \sum_{b=1}^{B}\Pr\left( \sum_{i=1}^n Z_i =b\right )\nonumber\\
& \overset{(iii)}{=} q_{B} = \rho_Y \le  n\cdot \delta,\nonumber
\end{align}
where $(i)$ comes from the fact that $\Pr\left( \sum_{i=1}^n Z_i =b\right )$ is zero for $b=0$ and $b>B$,
$(ii)$ comes by applying H\"{o}lder's Inequality, 
and $(iii)$ holds since $f$ is assumed to have the monotonicity property:
$$
f(x_1,\ldots,x_i,\ldots,  x_n) \ge f(x_1,\ldots,0,\ldots,  x_n),
$$
for any sequence of variables $x_1,\ldots, x_n$.
Hence, for any $b_1\ge b_2$
\begin{equation}
\Pr \left( f(Z_1,\ldots, Z_n) > C \left| \sum_{i=1}^n Z_i = b_1\right.\right) \ge\Pr \left( f(Z_1,\ldots, Z_n) > C \left| \sum_{i=1}^n Z_i = b_2\right.\right).
\end{equation}

In conclusion, we have upper bounded $\rho_Z$ and $\rho_Y$ by
\begin{align}
\rho_Z \leq \rho_Y \leq n \cdot \rho_X \leq n \cdot \delta. \label{eq:replacementbound}
\end{align}

\textbf{Application to Theorem \ref{theo:ourGsample}:}
For our purposes, the above bound Eq. \eqref{eq:replacementbound} allows us to assert Theorem \ref{theo:ourGsample} for with replacement, without replacement, and i.i.d. sampling, with different constants.
Specifically, for any graph $G$, the size of the largest connected component in the sampled subgraph can be expressed as a function $f_G(x_1, \dots, x_n)$, where each $x_i$ is an indicator for whether the $i$th vertex was chosen in the sampling process.
Note that $f_G$ is a monotone function, i.e., $f_G(x_1, \dots, x_i, \dots, x_n) \geq f_G(x_1, \dots, 0, \dots, x_n)$ since adding vertices to the sampled subgraph may only increase (or keep constant) the size of the largest connected component. 
We note that the high probability statement of Theorem~\ref{theo:Gsample1}, can be restated so that the constants in front of the size of the connected components accomodate for a statement that is true with probability $1-1/n^\zeta$, for any constant $\zeta>1$.
This is required to take care of the extra $n$ factor that appears in the bound of Eq.~\ref{eq:replacementbound}, and to obtain Theorem~\ref{theo:ourGsample}.

\section{Parallel Connected Components Computation}
\label{sec:cc}

As we will see in the following, the cost of computing CCs in parallel will depend on the number of cores so that uniform allocation across them is possible, and the number of edges that are induced by the sampled updates on the bipartite update-variable graph $G_u$ is bounded.
As a reminder we denote as $G_u^i$ the bipartite subgraphs of the update-variable graph $G_u$, that is induced by the sampled updates of the $i$-th batch.
Let us denote as $E_u^i$ the number of edges in $G_u^i$.

Following the sampling recipe of our main text (i.e., sampling each update per batch uniformly and with replacement), let us assume here that we are sampling $c\cdot n$ updates in total, for some constant $c\ge1$.
Assuming that the size of each batch is $B = (1-\epsilon)\frac{n}{\Delta}$, the total number of sampled batches will be $n_b = \frac{c}{1-\epsilon} \Delta$.
The total number of edges in the induced sampled bipartite graphs is a random variable that we denote as
$$ Z = \sum_{i=1}^{n_b} E_u^i.$$
Observe that $\E Z = c\cdot E_u$.
Using a simple Hoeffding concentration bound we can see that 
\begin{align*}
 \Pr\left\{ |Z- c E_u| >  (1+\delta) c\cdot E_u \right\} &\le 2 e^{-\frac{2 c^2\cdot (1+\delta)^2 E_u^2}{c\cdot n \Delta_{\text{L}}^2}}\le 2 e^{-2 c\cdot (1+\delta)^2 \cdot \frac{ n\overline{\Delta}^2_{\text{L}}}{\Delta_{\text{L}}^2}}
\end{align*}
where $\Delta_{\text{L}}$ is the max left degree of the bipartite graph $G_u$ and $\overline{\Delta}_{\text{L}}$ is its average left degree.
Now assuming that 
$$ \frac{\Delta_{\text{L}}}{\overline{\Delta}_{\text{L}}} \le \sqrt{n}$$
we obtain
$$ \Pr\left\{ |Z- c E_u| > \log n \cdot c\cdot  E_u \right\} \le  2 e^{- c\cdot \log^2n}.$$
Hence, we get the following simple lemma:
\begin{lem}
Let $ \frac{\Delta_{\text{L}}}{\overline{\Delta}_{\text{L}}} \le \sqrt{n}$.
Then, the total number of edges $Z = \sum_{i=1}^{n_b} E_u^i$ across the $n_b=\frac{c}{1-\epsilon} \Delta$ sampled subgraphs $G_u^1,\ldots, G_u^{n_b}$ is less than $ O( E_u \log n)$ with probability $1-n^{-\Omega(\log n)}$.
\end{lem}

Now that we have a bound on the number of edges in the sampled subgraphs, we can derive the complexity bounds for computing CCs in parallel.
We will break the discussion into the not-too-many- and many-core regime.

\paragraph{The not-too-many cores regime.} 
In this case, we sample $n_b$ subgraphs, allocate them across $P$ cores, and let each core compute CCs on its allocated subgraphs using BFS or DFS.
Since each batch is of size $B = (1-\epsilon)\frac{n}{\Delta}$, we need $n_b = \lfloor \sfrac{c \cdot n}{B} \rfloor= c\cdot  \lfloor \frac{\Delta}{1-\epsilon} \rfloor$ batches to cover $c\cdot n$ updates in total.
If the number of cores is 
$$ P = O\left( \frac{\overline{\Delta}_{\text{L}}}{\Delta_{\text{L}}} \cdot \Delta \right),$$
then the max cost of a single CC computation on a subgraph (which is $O(B\Delta_{\text{L}})$) is smaller than the average cost across $P$ cores, which is $O(Z/P)$.
This implies that a uniform allocation is possible, so that $P$ cores can share the computational effort of computing CCs.
Hence, we can get the following lemma:

\begin{lem}
Let the number of cores be $ P = O\left( \frac{\overline{\Delta}_{\text{L}}}{\Delta_{\text{L}}} \cdot \Delta \right)$, and let us sample $O(\Delta)$ batches, where each batch is of size $O(\frac{n}{\Delta})$.
Then, each core will not spend more than $O\left(\frac{E_u \log n}{P}\right)$ time in computing CCs, with high probability.
\end{lem}

\paragraph{The many-cores regime.} 
When $P >> \frac{\overline{\Delta}_{\text{L}}}{\Delta_{\text{L}}}$ the uniform balancing of the above method will break, leaving no room for further parallelization.
In that case, we can use a very simple ``push-label'' CC algorithm, whose cost on $P$ cores and arbitrary graphs with $E$ edges and max degree $\Delta$ is $O(\max\{\frac{E}{P}, \Delta\}\cdot C_{\max})$, where $C_{\max}$ is the size of the longest-shortest path, or the diameter of the graph \cite{kang2009pegasus}.
This parallel CC algorithm is given below, where each core is allocated with a number of vertices

\begin{wrapfigure}{R}{0.5\columnwidth}
\vspace{-0.3cm}
\begin{minipage}{0.46\columnwidth}
\begin{center}
\begin{algorithm}[H]
 \caption{\cyc{}}
 {\small
   \caption{\texttt{push-label}}
\begin{algorithmic}[1]
\STATE Initialize shared $\mathsf{cc}(v)$ variables to vertexIDs
\FOR{$i = 1:$ length of longest shortest path}
\FOR{ $v$ in the allocated vertex set}
\FOR{ all $u$ that are neighbors of $v$}
\STATE Read $cc(v)$ from shared memory
\IF{$cc(u)>cc(v)$}
\STATE Update shared $cc(u) \leftarrow \min(\mathsf{cc}(u),\mathsf{cc}(v))$
\ENDIF
\ENDFOR
\ENDFOR
\ENDFOR
\end{algorithmic}
   \label{alg:CC}
   }
 \end{algorithm}
 \end{center}
\end{minipage}
 \end{wrapfigure}

The above simple algorithm can be significantly slow for graphs where the longest-shortest path is large.
Observe, that in the sampled subgraphs $G_u^i$ the size of the shortest-longest path is always bounded by the size of the largest connected component.
By Theorem~\ref{theo:ourGsample} that is bounded by $O\left(\frac{\log n}{\epsilon^2}\right)$. 
Hence, we obtain the following lemma.

\begin{lem}
For any number of cores $P = O(\frac{n}{\Delta\cdot \Delta_{\text{L}}} )$, computing the connected component of a single sampled graph $G_u^i$ can be performed in time
$\mathcal{O}(\frac{E^i_u \log n}{P})$, with high probability.
\end{lem}

Since, we are interested in the overall running time for $n_b$ batches of the CC algorithm, we can see that the above lemma simply boils down to the following:
\begin{cor}
For any number of cores $P = O(\frac{n}{\Delta\cdot \Delta_{\text{L}}} )$, computing the connected component for all sampled graph $G_u^1, \ldots, G_u^{n_b}$ can be performed in time
$\mathcal{O}(\frac{E\log^2 n}{P})$.
\end{cor}
\begin{rem}
In practice it seems to be that parallelizing the CC computation using the
not-too-many core regime policy is significantly more scalable.
\end{rem}

\section{Allocating the Conflict Groups}
\label{sec:allocation}

After we have sampled a single batch (i.e., a subgraph $G_u^i$), and computed the CCs for it, we have to allocate the connected components of that sampled subgraph across cores.
Observe that each connected component will contain at most $\log n$ updates, each ordered according to the a serial predetermined order.
Once a core has been assigned all the CCs, it will process all the updates included in the CCs according to the order that each update has been labeled with.

Now assuming that the cost of the $i$-th update is $w_i$, the cost of a single connected component $\setC$ will be $w_{\setC} = \sum_{i\in \setC} w_i$. 
We can now allocate the CCs accross cores so that the maximum core load is minimized, using the following $4/3$-approximation algorithm (i.e., an allocation with max load that is at most $4/3$ times the maximum between the max weight, and the sum of weights divided by $P$):
\begin{wrapfigure}{R}{0.5\columnwidth}
\vspace{-0.3cm}
\begin{minipage}{0.465\columnwidth}
\begin{algorithm}[H]
 \caption{Greedy allocation}
 {\small
\begin{algorithmic}[1]
\STATE {\bf Input} $\{w_1,\ldots, w_m\}$ \hfill {\color{gray}\%  weights to be allocated}
\STATE $b_1=0, \ldots, b_P=0$ \hfill {\color{gray}\% empty buckets}
\STATE ${\bf w}$ = sorted stack of the weights (descending order)
\FOR{$i = 1:m$}
\STATE $w = \text{pop}({\bf w})$
\STATE{add $w$ to bucket $b_i$ with least sum of weights}
\ENDFOR
\end{algorithmic}
 }
   \label{alg:greedyalloc}
 \end{algorithm}
\end{minipage}
 \end{wrapfigure}

To proceed with characterizing the maximum load among the $P$ cores, we assume that the cost of a single update $U_i$ is proportional to the out-degree of that update ---according to the update-variable graph $G_u$--- times a constant cost which we shall refer to as $\kappa$. Hence, $w_i = O(d_{L,i} \cdot \kappa)$, where $d_{L,i}$ is the  degree of the $i$-th left vertex of $G_u$.

Observe that the total cost of computing the updates in a single sampled subgraph $G_u^i$ is proportional to $\mathcal{O}(E_u^i \cdot \kappa)$.
Moreover, observe that the maximum weight among all CCs cannot be more than $O(\Delta_L  \log n \kappa)$ where $\Delta_L$ is the max left degree of the bipartite update-variable graph $G_u$.

\begin{lem}
We can allocate CCs such that the maximum load among cores is 
$\mathcal{O}\left(
\max\left\{ \frac{E^i_u}{P},  \Delta_L \log n\right\} \cdot \kappa  \right), $
with high probability, where $\kappa$ is the per edge cost for computing one of the $n$ updates defined on $G_u$.
\end{lem}
If $ P = O\left(\frac{n}{\Delta\cdot  \Delta_{\text{L}}}\right) $
then the average weight will be larger than the maximum divided by a $\log n$ factor, and a near-uniform allocation of CCs according to their weights possible.
Since, we are interested in the overall running time for $n_b$ batches, we can see that the above lemma simply boils down to the following:
\begin{cor}
For any number of cores $P = O(\frac{n}{\Delta\cdot \Delta_{\text{L}}} )$, computing the stochastic updates of the allocated connected component for all sampled graphs (i.e., batches) $G_u^1, \ldots, G_u^{n_b}$ can be performed in time
$\mathcal{O}(\frac{E\log^2 n}{P} \cdot \kappa )$.
\end{cor}

\section{Robustness against High-degree Outliers}
\label{app:outliers}
Here, we discuss how \cyc{} can guarantee nearly linear speedups when there is a sublinear $O(n^\delta)$ number of high-conflict updates, as long as the remaining updates have small degree.

Assume that our conflict graph $G_c$ defined between the $n$ update functions has a very high maximum degree $\Delta_o$.
However, consider the case where there are only $O(n^\delta)$ nodes that are of that high-degree, while the rest of the vertices have degree much smaller (on the induced subgraph by the latter vertices), say $\Delta$.
According to our main analysis, our prescribed batch sizes cannot be greater than $B=(1-\epsilon)\frac{(1-\epsilon)n}{\Delta_o}$. However, if say $\Delta_o = \Theta(n)$, then that would imply that $B=O(1)$, hence there is not room for parallelization by \cyc{}.
What we will show, is that by sampling according to $B=(1-\epsilon)\frac{n-O(n^\delta)}{\Delta}$, we can on average expect a parallelization that is similar to the case where the outliers are not present in the conflict graph. For a toy example see Figure~\ref{fig:outliers}.

 \begin{figure}[h] 
\centerline{ \includegraphics[width=0.37\columnwidth]{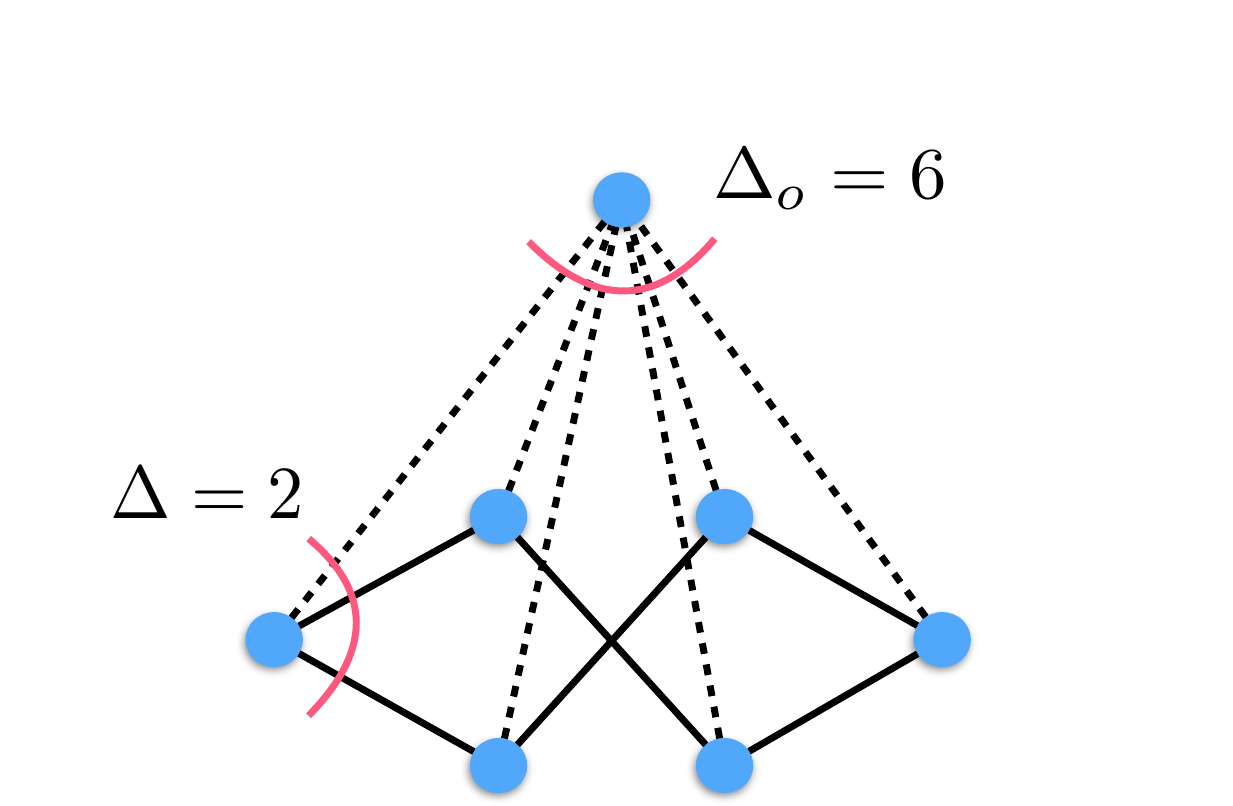}}
\caption{{\small 
The above conflict graph has a vertex with high degree (i.e., $\Delta_o = 6$), and the remaining of the graph has maximum induced degree $\Delta=2$.
In this toy case, when we sample roughly $\frac{n-1}{\Delta} = 3$ vertices, more often than not, the large degree vertex will not be part of the sampled batch.
This implies that when parallelizing with \cyc{} these cases will be as much parallelizable as if the high degree vertex was not part of the graph.
Each time we happen to sample a batch that includes the max. degree vertex, then essentially we lose all flexibility to parallelize, and we have to run the serial algorithm.
What we establish rigorously is that ``on average" the parallelization will be as good as one would hope for even in the case where the outliers are not present.}}
\label{fig:outliers}
\end{figure}

Our main result for the outlier case follows:

\begin{lem}
Let us assume that there are $O(n^\delta)$ outlier vertices in the original conflict graph $G$ with degree at most $\Delta_o$, and let the remaining vertices have degree (induced on the remaining graph) at most $\Delta$.
Let the induced update-variable graph on these low degree vertices abide to the same graph assumptions as those of Theorem \ref{theo:main}.
Moreover, let the batch size be bounded as
$$B   \le \min\left\{(1-\epsilon)\frac{n-O(n^\delta)}{\Delta},\;\;  O\left(\frac{n^{1-\delta}}{P}\right) \right\} \nonumber. $$
Then, the expected runtime of \cyc{} will be
$ O\left( \frac{E_u \cdot \kappa}{P}\cdot  \log^2 n \right).$ 
\label{lem:outliers}
\end{lem}

\begin{proof}
Let $w^i_s$ denote the total work required for batch $i$ if that batch contains no outlier notes, and $w^i_o$ otherwise.
It is not hard to see that 
$w_s=\sum_i w^i_s = O\left( \frac{E_u \cdot \kappa}{P}\cdot  \log^2 n \right)$
and 
$w_o=\sum_i w^i_s = O\left( E_u \cdot \kappa\cdot  \log^2 n \right)$
Hence, the expected computational effort by \cyc{} will be
\begin{align*}
w_s\cdot \Pr\{\text{a random batch contains no outliers}\} + w_o \Pr\{\text{a random batch contains outliers}\}
\end{align*}
where
\begin{equation}
\Pr\{\text{a random  batch contains no outliers}\} = \Omega\left( \left(1-\frac{1}{n^{1-\delta}}\right)^{B}\right) \ge 1-O\left(\frac{B}{n^{1-\delta}}\right)
\end{equation}
Hence the expected running time will be proportional to 
$ O\left( \frac{E_u \cdot \kappa}{P}\cdot  \log^2 n \right)$, if 
$O( \frac{E_u \cdot \kappa}{P}\cdot  \log^2 n ) = O( E_u \cdot \kappa\cdot  \log^2 n\cdot \frac{B}{n^{1-\delta}}),$
which holds when
$B= O\left(\frac{n^{1-\delta}}{P}\right).$
\end{proof}
\section{Complete Experiment Results}
\label{app:exptresults}
In this section, we present the remaining experimental results that were left out for brevity from our main experimental section.
In Figures~\ref{appfig:expt_converge} and \ref{appfig:expt_convergegd}, we show the convergence behaviour of our algorithms, as a function of the overall time, and then as a function of the time that it takes to perform only the stochastic updates (i.e., in Fig.~\ref{appfig:expt_convergegd} we factor out the graph partitioning, and allocation time).
In Figure~\ref{appfig:expt_speedups}, we provide the complete set of speedup results for all algorithms and data sets we tested, in terms of the number of cores.
In Figure~\ref{appfig:expt_speedupsgd}, we provide the speedups in terms of the the computation of the stochastic updates, as a function of the number of cores.
Then, in Figures~\ref{appfig:expt_url_converge}\;--\;\ref{appfig:expt_url_speedupsgd}, we present the convergence, and speedups of the overal computation, and then of the stochastic updates part, for our dense feature URL data set.
Finally, in Figure \ref{appfig:expt_hdiverge} we show the divergent behavior of \HW{} for the least square experiments with SAGA, on  the NH2010 and  DBLP datasets.

Our overall observations here are similar to the main text.
One additional note to make is that when we take a closer look to the figures relative to the times and speedups of the stochastic updates part of \cyc{} (i.e., when we factor out the time of the graph partitioning part), we see that \cyc{} is able to perform stochastic updates faster than \hog{} due to its superior spatial and temporal access locality patterns.
If the coordination overheads for \cyc{} are excluded, we are able to improve speedups, in some cases by up to 20-70\% (Table \ref{tab:speedups}).
This suggests that by further optimizing the computation of connected components, we can hope for better overall speedups of \cyc{}.

\begin{table*}[h]
\vskip 0.15in
\begin{center}
\begin{small}
\begin{tabular}{|l|c|c|c|c|c|c|c|}
\hline
& Mat. Comp. & Mat. comp & Word2Vec & Graph Eig. & Graph Eig. & Least Squares & Least Squares\\
& SGD & $\ell_2$-SGD & SGD & SVRG & SVRG & SAGA & SAGA \\
& MovieLens & MovieLens & EN-Wiki & NH2010 & DBLP & NH2010 & DBLP\\
\hline\hline
Overall &&& &&&&
\\Speedup &    8.8010 &    7.6876 &   10.4299 &    2.9134 &    4.7927 &    4.4790 & 4.6405 \\\hline
Speedup of &&& &&&&\\
Updates &    9.0453 &    7.9226 &   11.4610 &    3.4802 &    5.5533 &    4.6998 & 8.1133 \\
\hline\hline
\% change & 2.7759\% &    3.0580\% &    9.8866\% &   19.4551\% &   15.8718\% &    4.9285\% & 74.8408\%\\
\hline
\end{tabular}
\end{small}
\end{center}
\caption{
  Speedups of \cyc{} at 16 threads.
  Two versions speedups are given for each problem: (1) with the overall running time, including the coordination overheads, and (2) using only the running time for stochastic updates.
  Speedups using only stochastic updates are up to 20\% better, which suggests we could potentially observe larger speedups by further optimizing the computation of connected components.}
  \label{tab:speedups}
\vskip -0.1in
\end{table*}

\newpage
\begin{figure}[H]
\hrule
\vspace{0.1cm}
\begin{center}
\includegraphics[width=0.5\textwidth]{images/matplotlib_plots/time_loss_legend_48_3}\\
    \subfigure[LS, NH2010, SAGA]{\label{appfig:ls_nh2010_may_saga:converge} \includegraphics[width=0.24\textwidth]{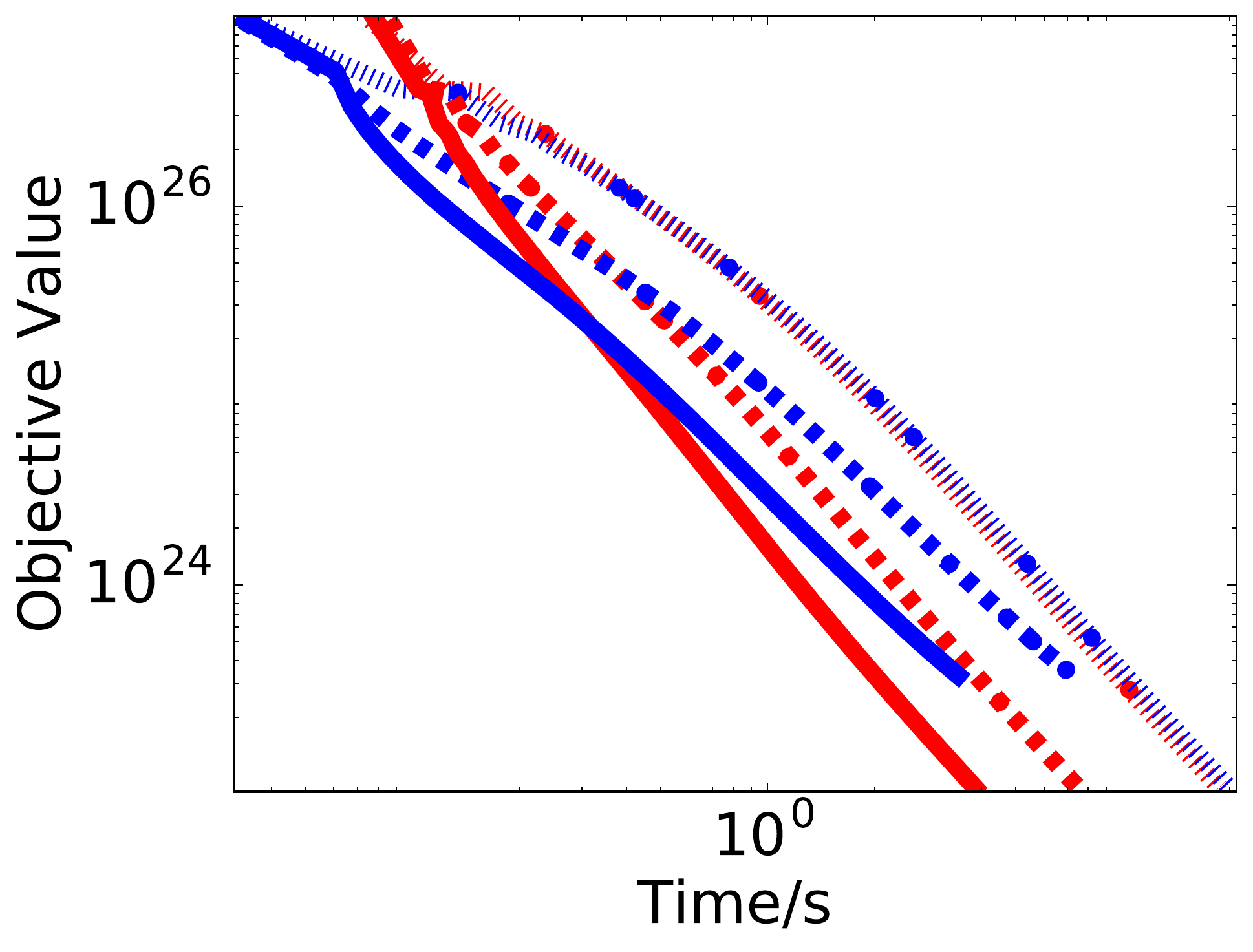}}
    \subfigure[LS, DBLP, SAGA]{\label{appfig:ls_DBLP_may_saga:converge} \includegraphics[width=0.24\textwidth]{images/matplotlib_plots/cyc_least_squares_dblp__saga_100_10000_0_converge}}
    \subfigure[Graph Eig., NH2010, SVRG]{\label{appfig:ge_nh2010_may_svrg:converge} \includegraphics[width=0.24\textwidth]{images/matplotlib_plots/cyc_matrix_inverse_nh2010__svrg_100_1000_0_converge}}  
    \subfigure[Graph Eig., DBLP, SVRG]{\label{appfig:ge_dblp_may_svrg:converge} \includegraphics[width=0.24\textwidth]{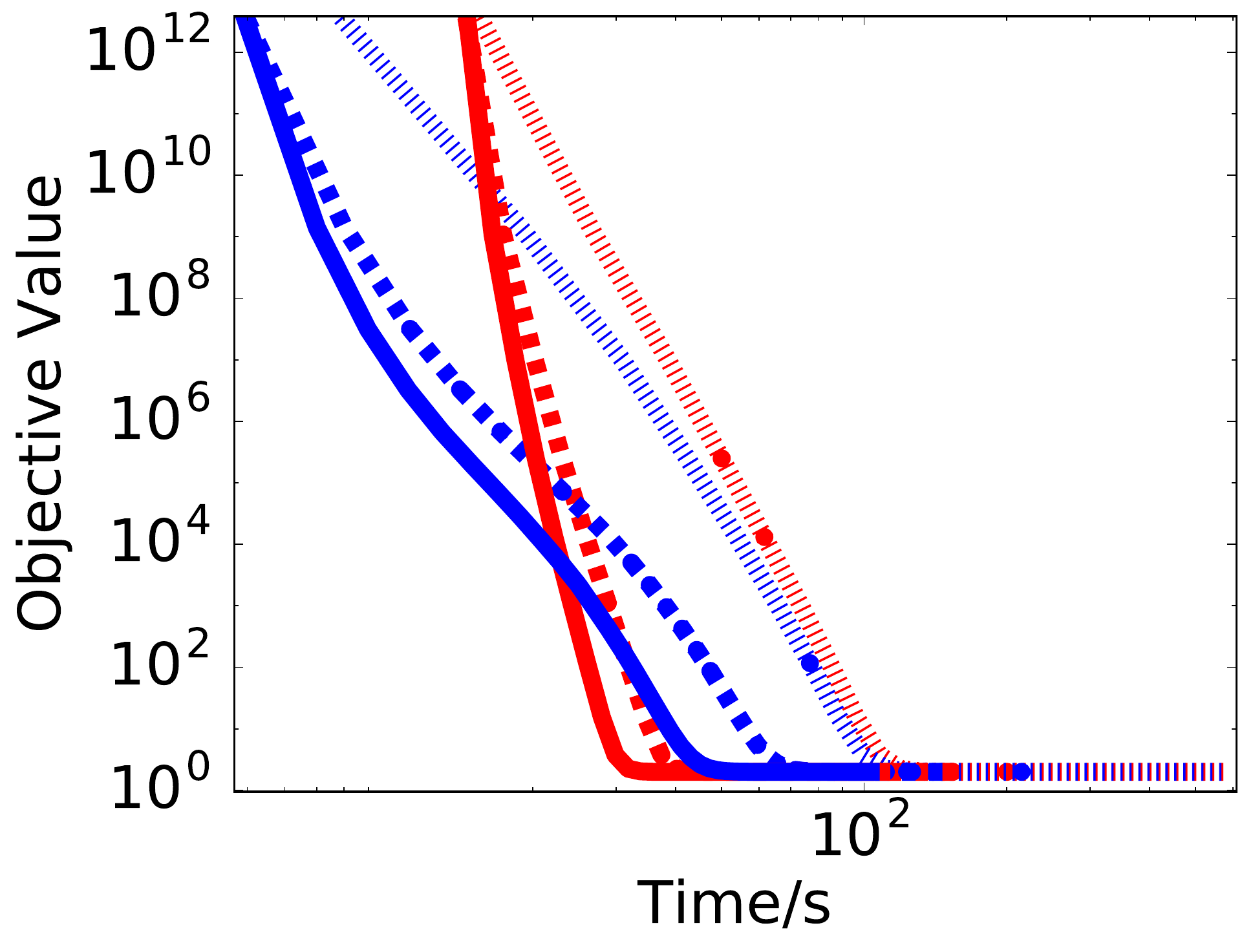}}\\

    \subfigure[Mat. Comp., 10M, $\ell_2$-SGD]{\label{appfig:mc_10m_may_rsg:converge} \includegraphics[width=0.24\textwidth]{images/matplotlib_plots/cyc_movielens__regularize_200_5000_100_converge}}
    \subfigure[Mat. Comp., 10M, SGD]{\label{appfig:mc_10m_may_sgd:converge} \includegraphics[width=0.24\textwidth]{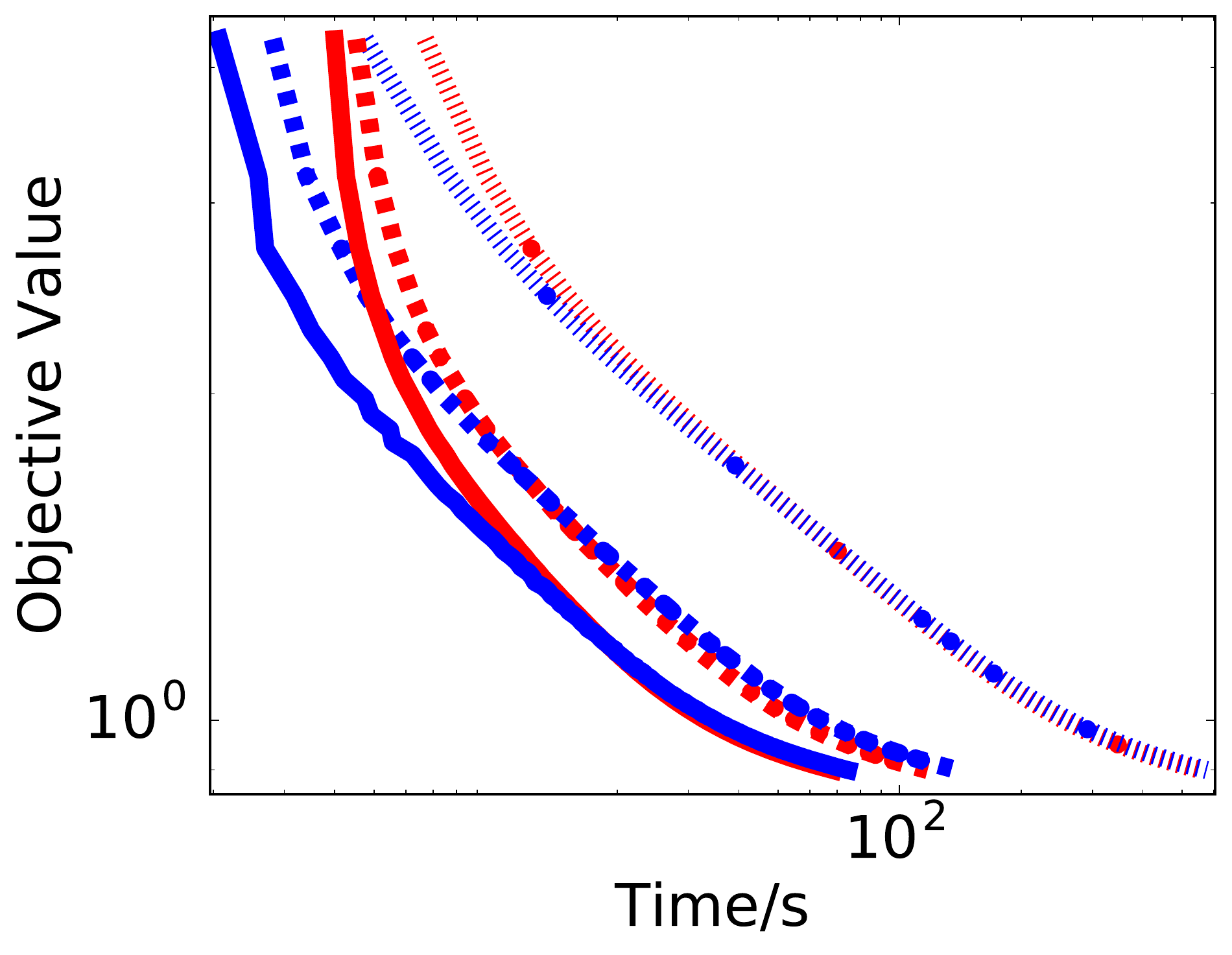}}
    \subfigure[Word2Vec, EN-Wiki, SGD]{\label{appfig:we_ewk_may_sgd:converge} \includegraphics[width=0.24\textwidth]{images/matplotlib_plots/cyc_word_embeddings_copt__sgd_200_4250_100_converge}}\\
\end{center}
      \caption{Convergence of \cyc{} and \hog{} on various problems, using 1, 8, 16 threads, in terms of overall running time.
  \cyc{} is initially slower, but ultimately reaches convergence faster than \hog{}.
  }
  \label{appfig:expt_converge}
\end{figure}

\vspace{-0.7cm}

\begin{figure}[H]
\hrule
\vspace{0.1cm}
\begin{center}
\includegraphics[width=0.5\textwidth]{images/matplotlib_plots/time_loss_legend_48_3}\\
    \subfigure[LS, NH2010, SAGA]{\label{appfig:ls_nh2010_may_saga:convergegd} \includegraphics[width=0.24\textwidth]{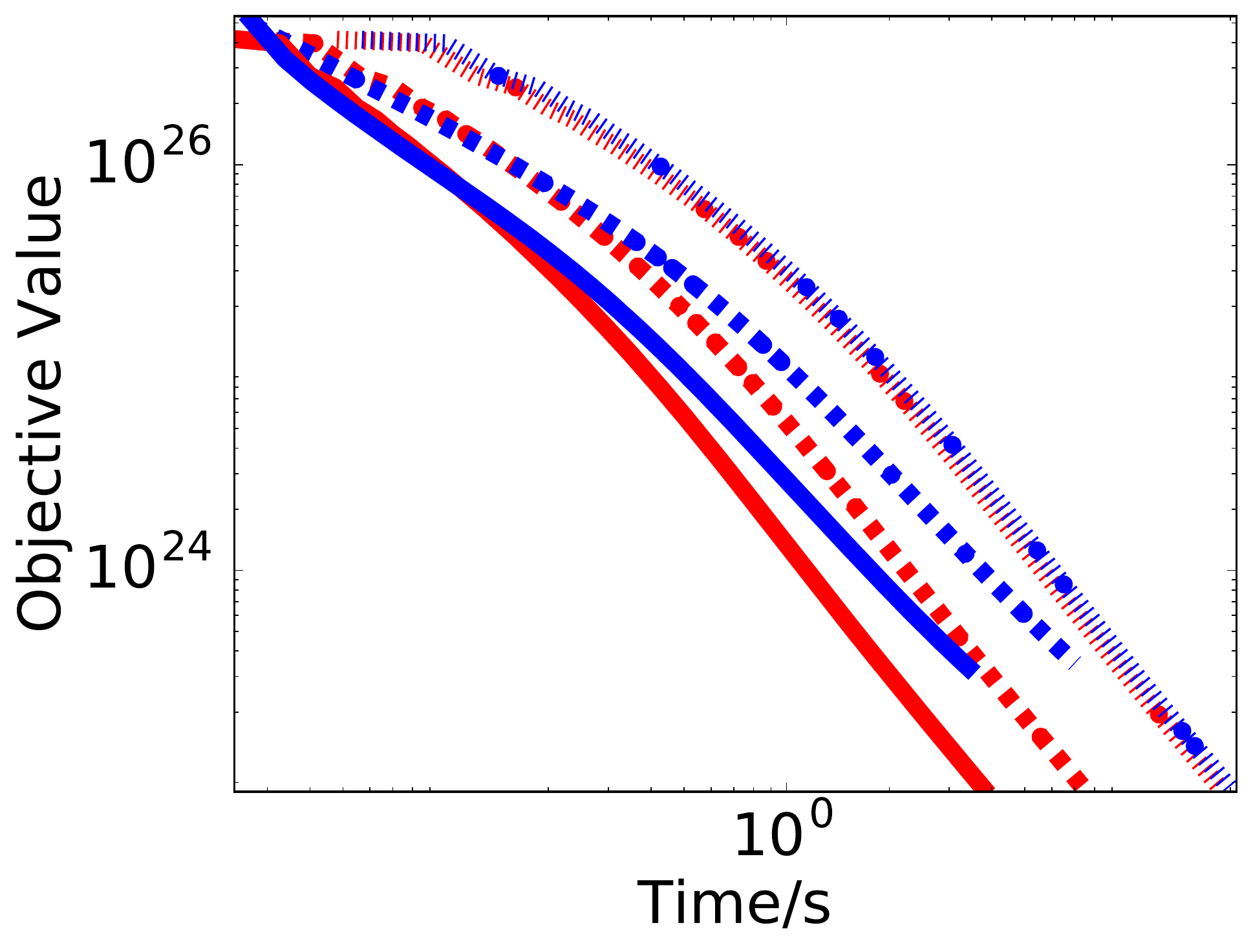}}
    \subfigure[LS, DBLP, SAGA]{\label{appfig:ls_DBLP_may_saga:convergegd} \includegraphics[width=0.24\textwidth]{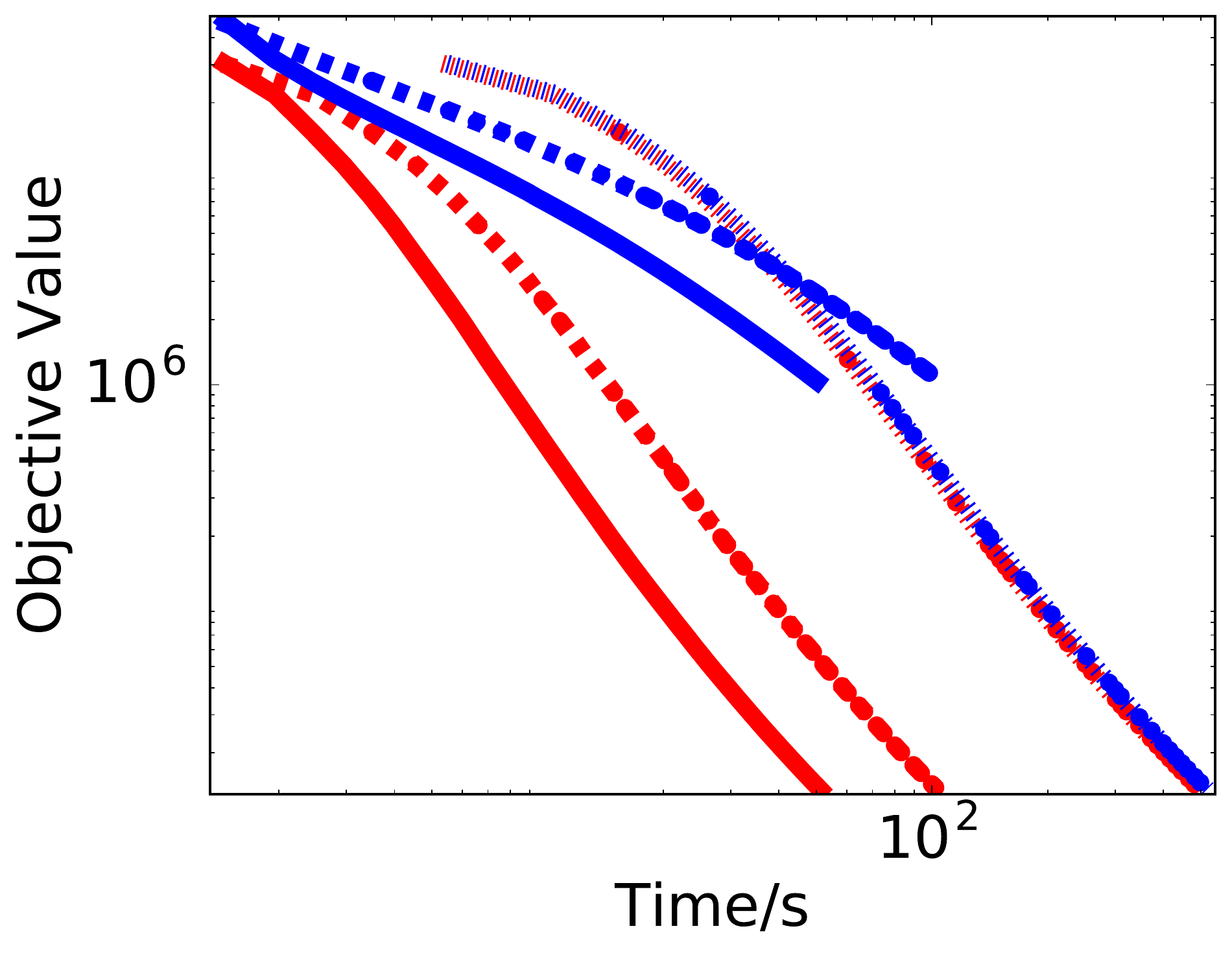}}
    \subfigure[Graph eigenvector, NH2010, SVRG]{\label{appfig:ge_nh2010_may_svrg:convergegd} \includegraphics[width=0.24\textwidth]{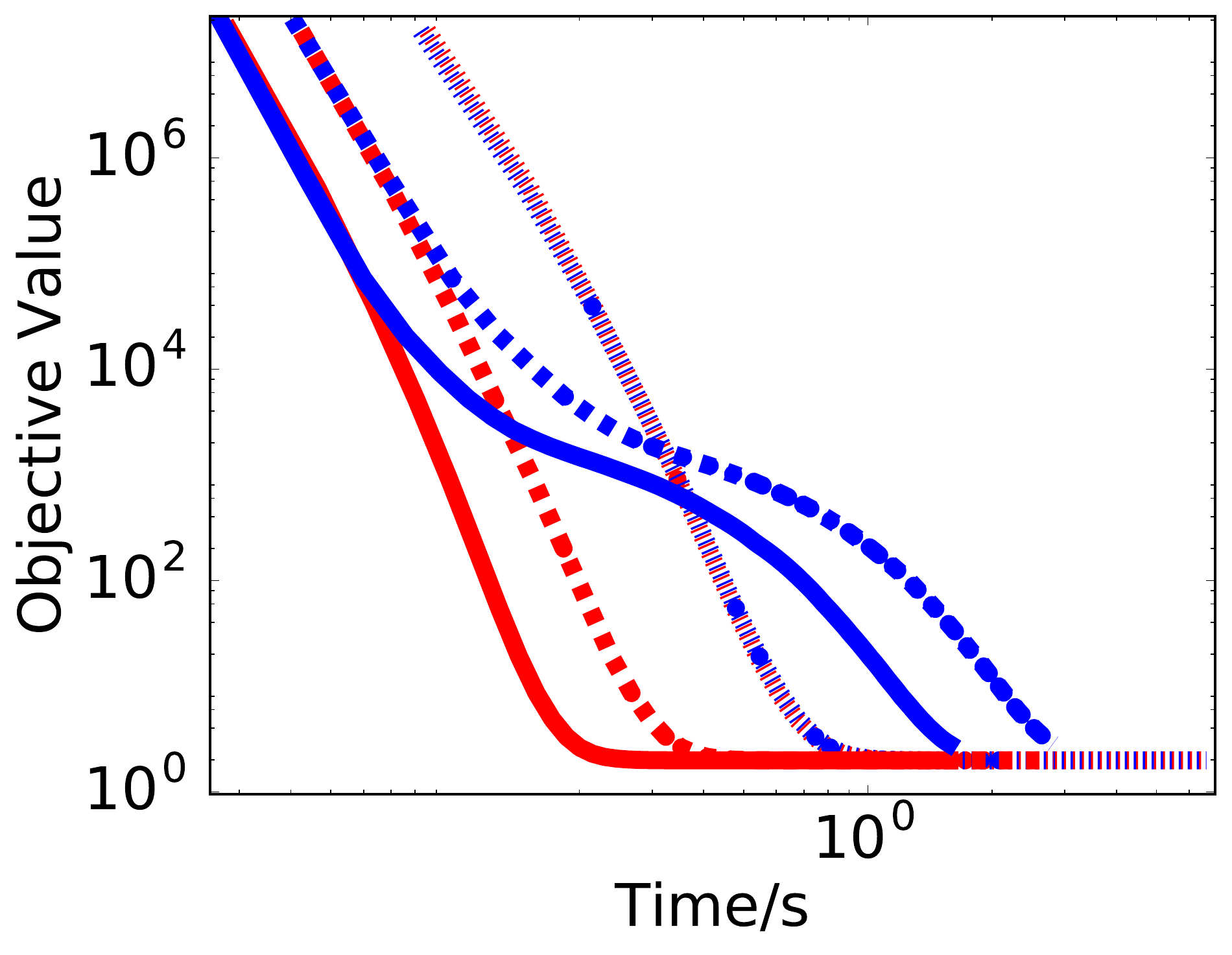}}
    \subfigure[Graph eigenvector, DBLP, SVRG]{\label{appfig:ge_dblp_may_svrg:convergegd} \includegraphics[width=0.24\textwidth]{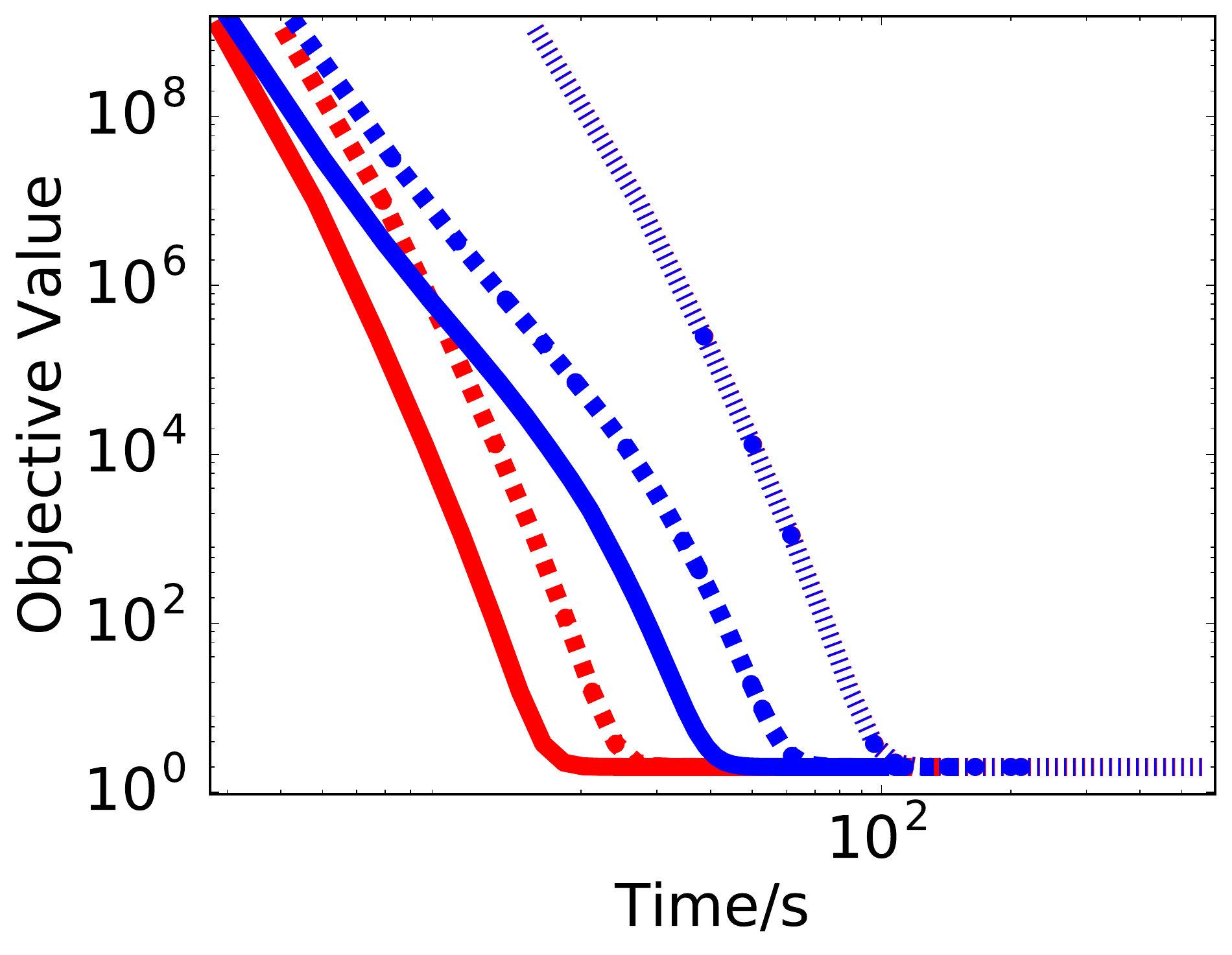}}\\

    \subfigure[Matrix completion, MovieLens 10M, weighted SGD]{\label{appfig:mc_10m_may_rsg:convergegd} \includegraphics[width=0.24\textwidth]{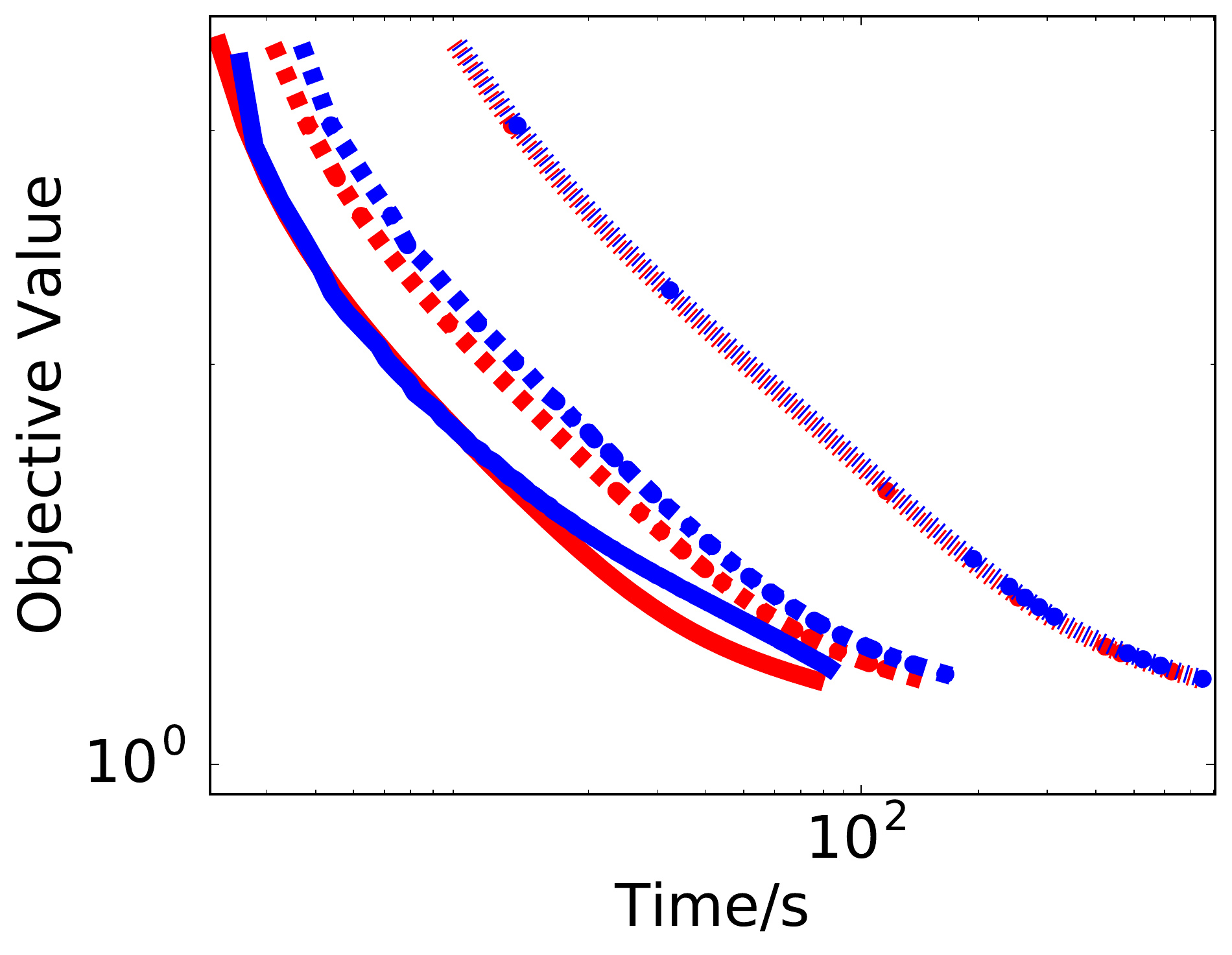}}
    \subfigure[Matrix completion, MovieLens 10M, SGD]{\label{appfig:mc_10m_may_sgd:convergegd} \includegraphics[width=0.24\textwidth]{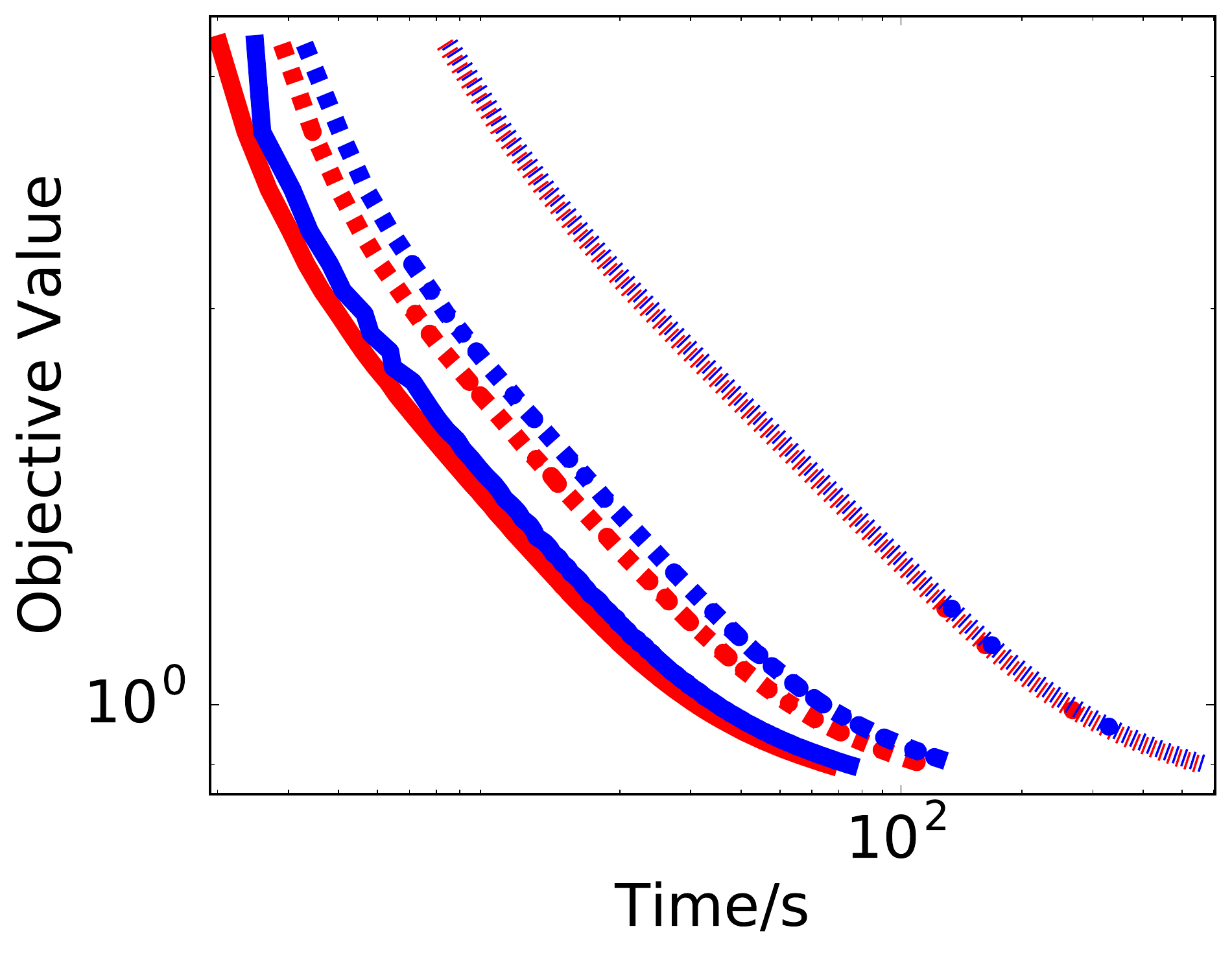}}
    \subfigure[Word embeddings, EN-Wiki, SGD]{\label{appfig:we_ewk_may_sgd:convergegd} \includegraphics[width=0.24\textwidth]{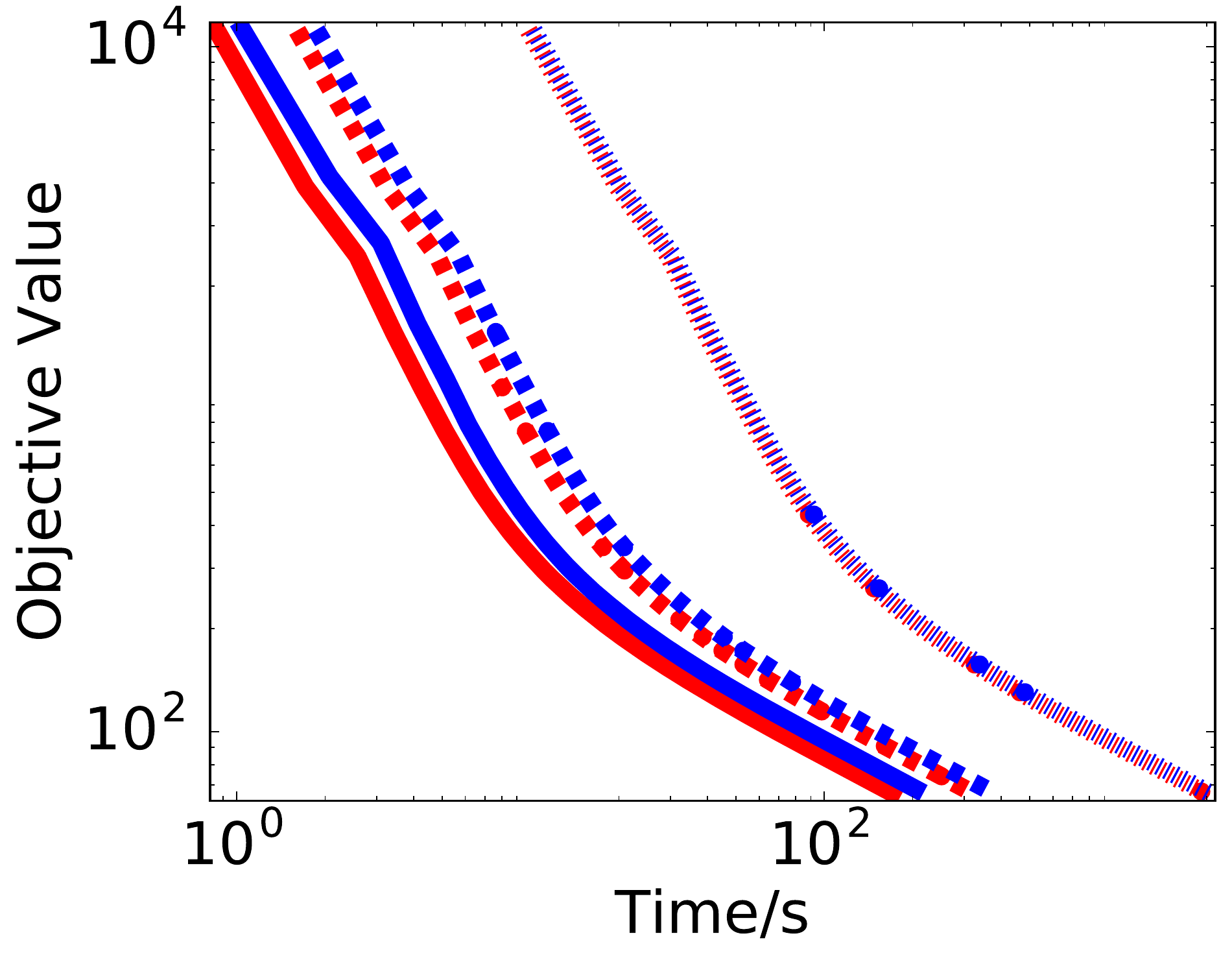}}\end{center}

  \caption{Convergence of \cyc{} and \hog{} on various problems, using 1, 8, 16 threads, in terms of running time for stochastic updates.
  }
  \label{appfig:expt_convergegd}
  \vspace{0.1cm}
  \hrule
\vspace{0.1cm}
\end{figure}

\newpage

\begin{figure}[H]
\begin{center}
  \vspace{0.1cm}
  \hrule
\vspace{0.1cm}
    \subfigure[LS, NH2010, SAGA]{\label{appfig:ls_nh2010_may_saga:speedups} \includegraphics[width=0.24\textwidth]{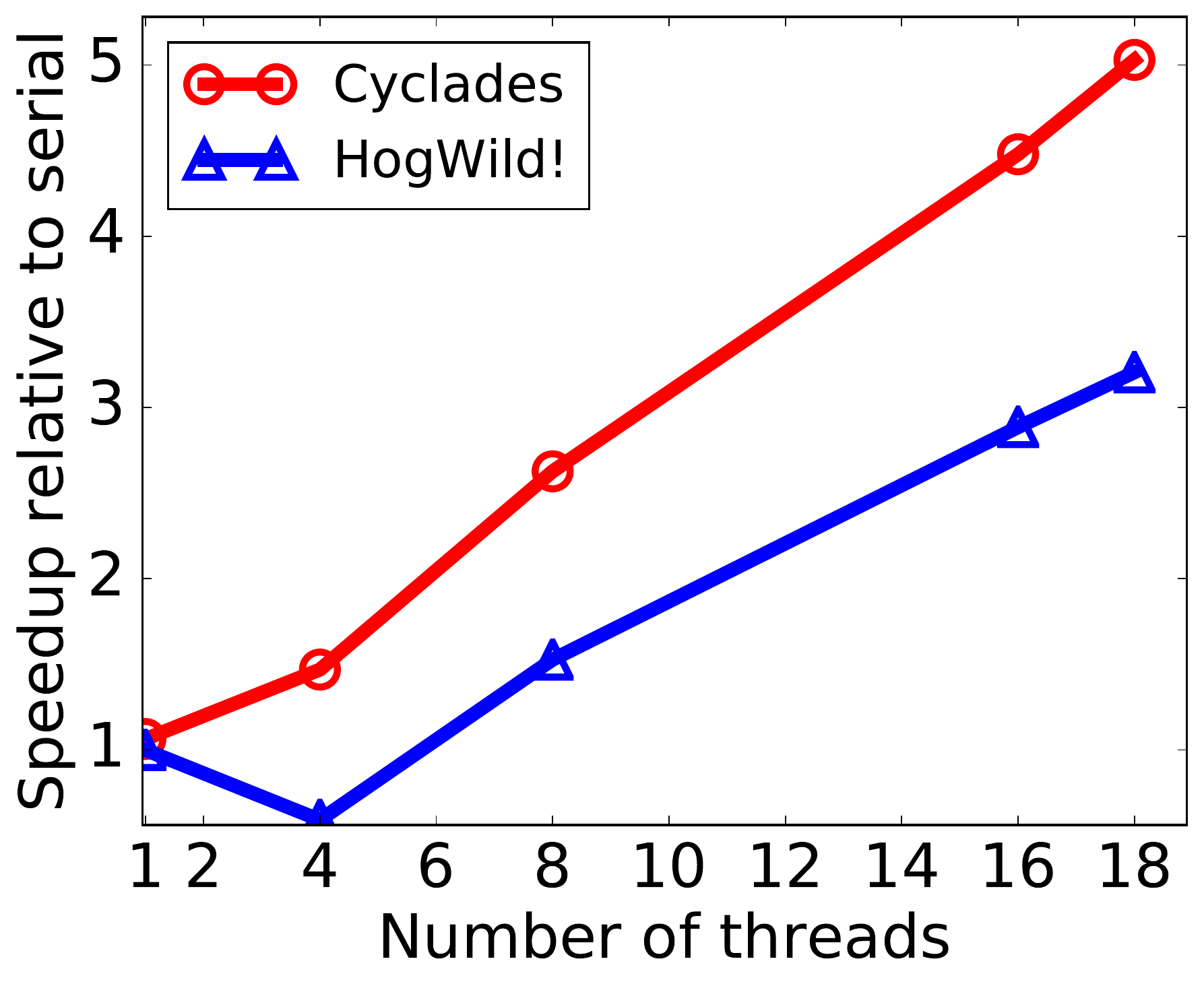}}
    \subfigure[LS, DBLP, SAGA]{\label{appfig:ls_DBLP_may_saga:speedups} \includegraphics[width=0.24\textwidth]{images/matplotlib_plots/cyc_least_squares_dblp__saga_100_10000_0_speedup}}
    \subfigure[Graph Eig., NH2010, SVRG]{\label{appfig:ge_nh2010_may_svrg:speedups} \includegraphics[width=0.24\textwidth]{images/matplotlib_plots/cyc_matrix_inverse_nh2010__svrg_100_1000_0_speedup}}
    \subfigure[Graph Eig., DBLP, SVRG]{\label{appfig:ge_dblp_may_svrg:speedups} \includegraphics[width=0.24\textwidth]{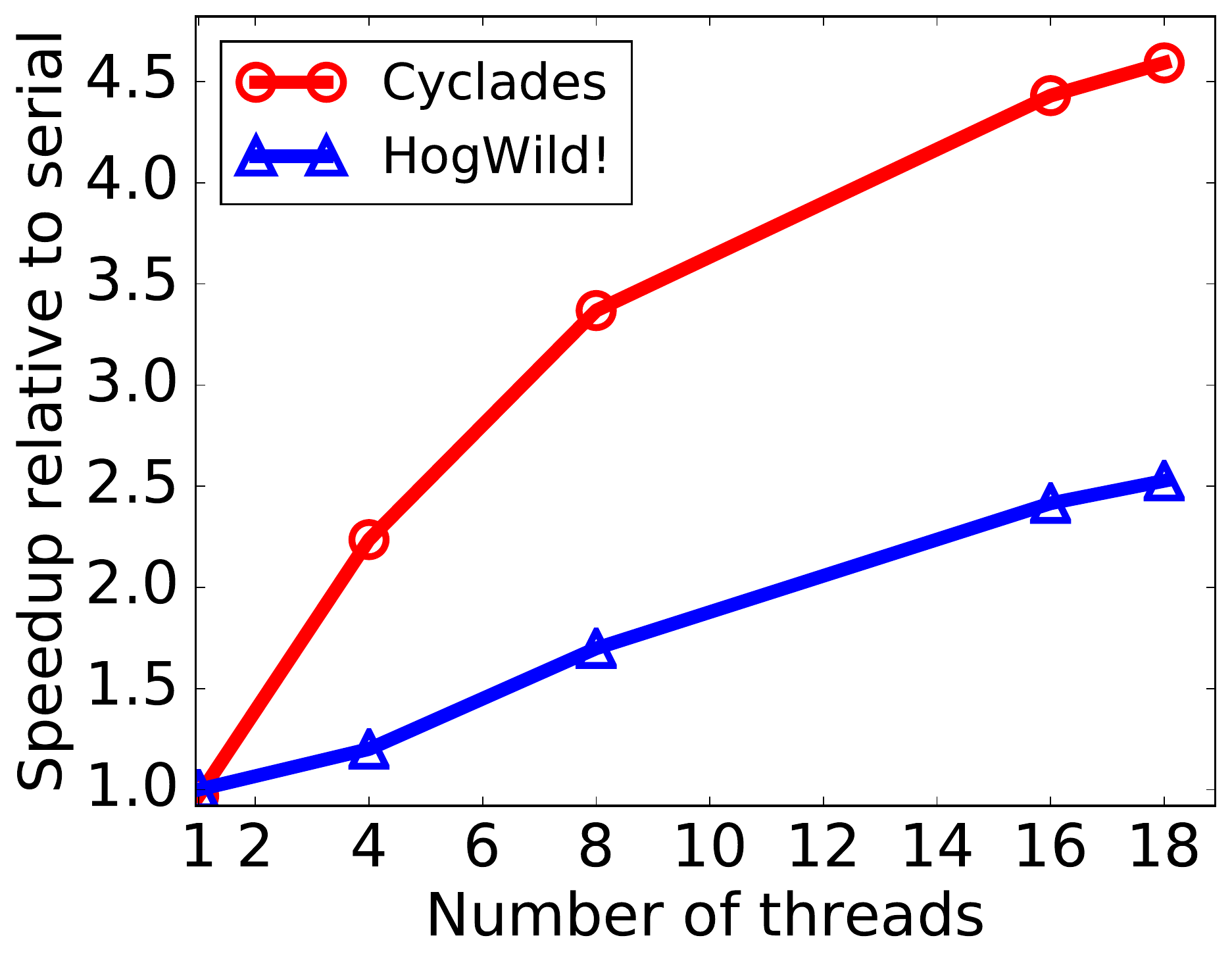}}\\

    \subfigure[Mat. Comp., 10M, $\ell_2$-SGD]{\label{appfig:mc_10m_may_rsg:speedups} \includegraphics[width=0.24\textwidth]{images/matplotlib_plots/cyc_movielens__regularize_200_5000_100_speedup}}
    \subfigure[Mat. Comp., 10M, SGD]{\label{appfig:mc_10m_may_sgd:speedups} \includegraphics[width=0.24\textwidth]{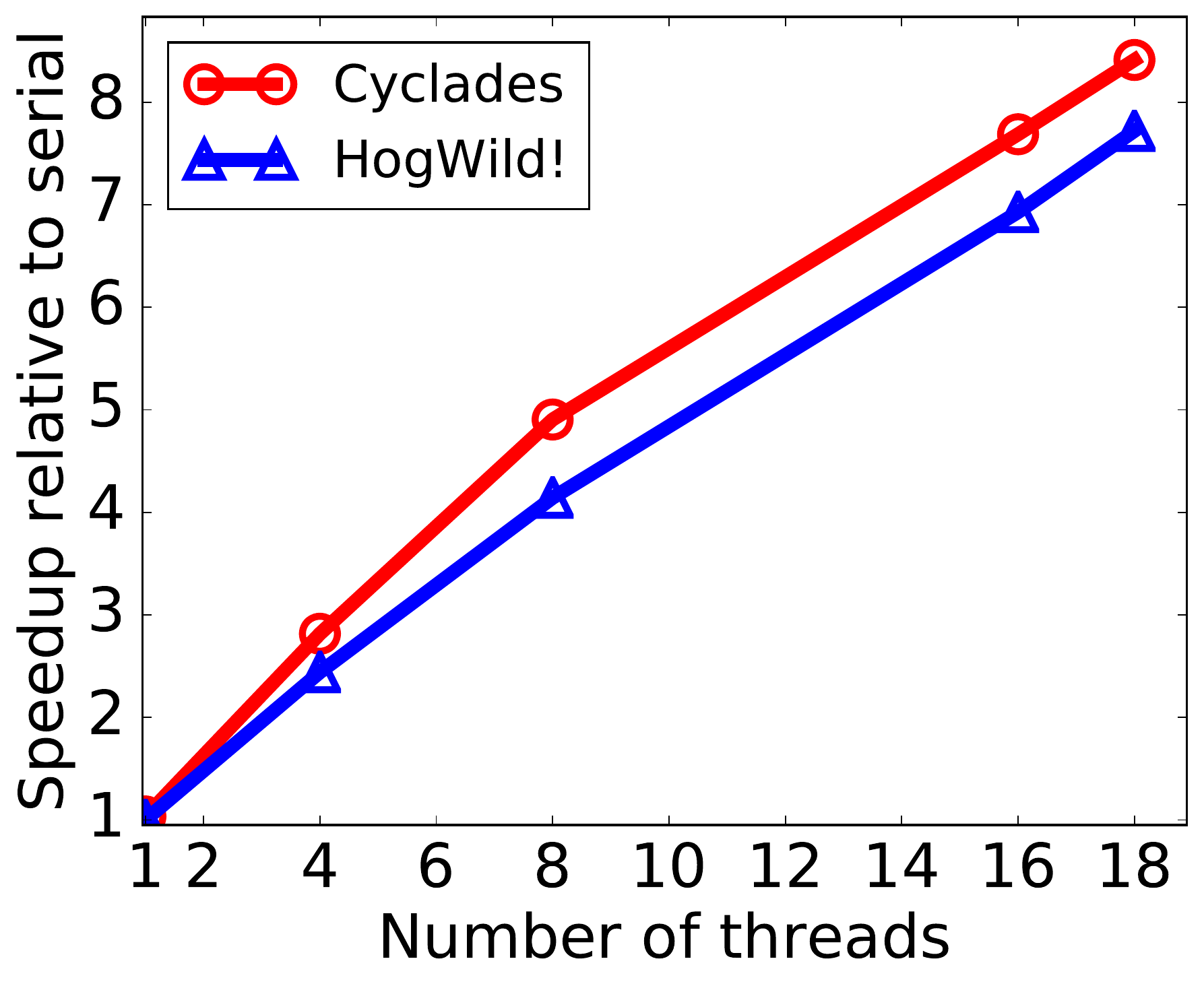}}
    \subfigure[Word2Vec, EN-Wiki, SGD]{\label{appfig:we_ewk_may_sgd:speedups} \includegraphics[width=0.24\textwidth]{images/matplotlib_plots/cyc_word_embeddings_copt__sgd_200_4250_100_speedup}}\\
\end{center}
  \caption{Speedup of \cyc{} and \hog{} on various problems, using 1, 4, 8, 16 threads, in terms of overall running time.
  On multiple threads, \cyc{} always reaches $\epsilon$ objective faster than \hog{}.
  In some cases (\ref{appfig:ls_nh2010_may_saga:speedups}, \ref{appfig:mc_10m_may_rsg:speedups}, \ref{appfig:we_ewk_may_sgd:speedups}), \cyc{} is faster than \hog{} on even 1 thread, as \cyc{} has better cache locality.
  }
  \label{appfig:expt_speedups}
  \vspace{0.1cm}
  \hrule
\vspace{0.1cm}
\end{figure}

\vspace{-0.7cm}

\begin{figure}[H]
\hrule
\vspace{0.1cm}
\begin{center}
    \subfigure[LS, NH2010, SAGA]{\label{appfig:ls_nh2010_may_saga:speedupsgd} \includegraphics[width=0.24\textwidth]{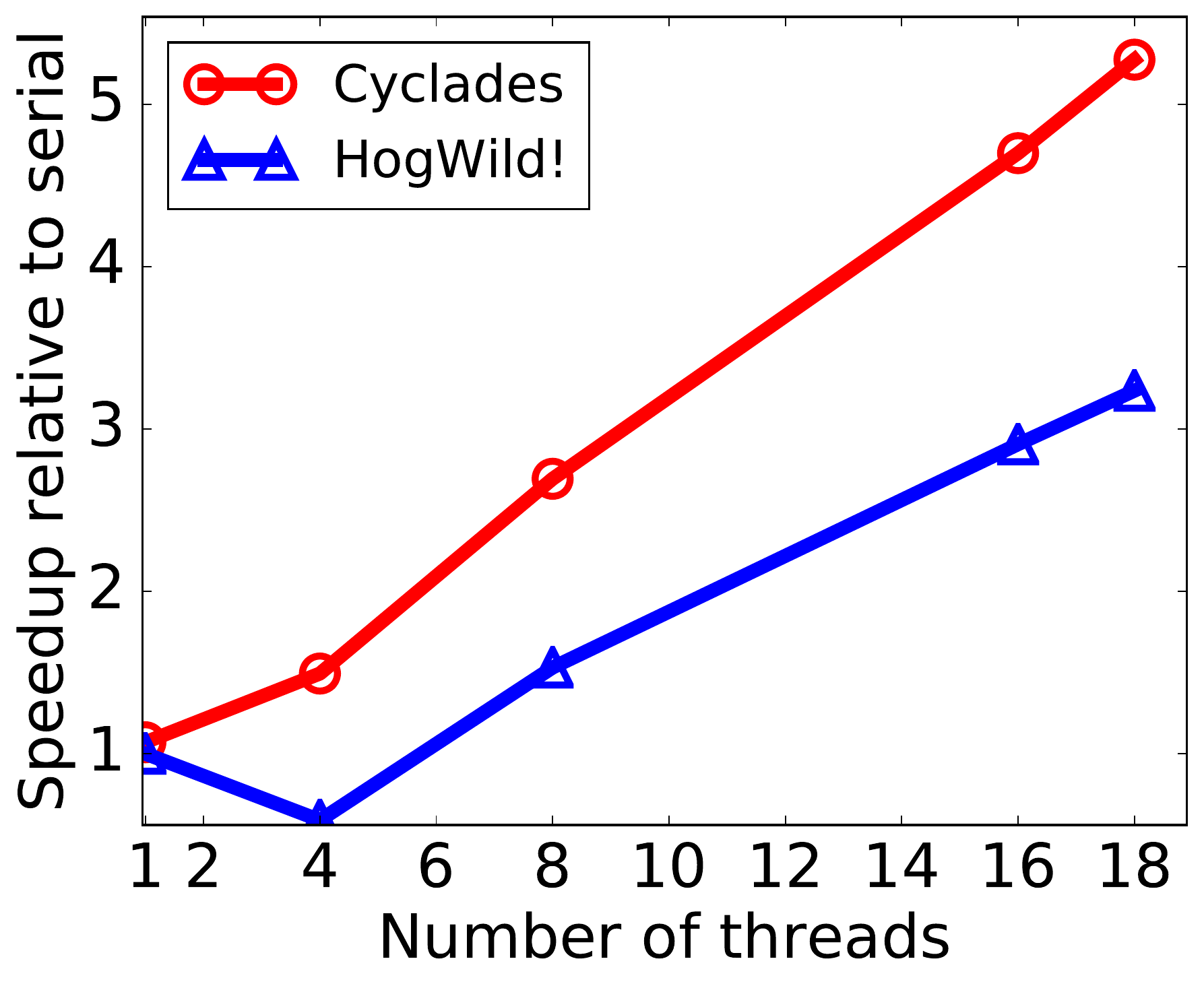}}
    \subfigure[LS, DBLP, SAGA]{\label{appfig:ls_DBLP_may_saga:speedupsgd} \includegraphics[width=0.24\textwidth]{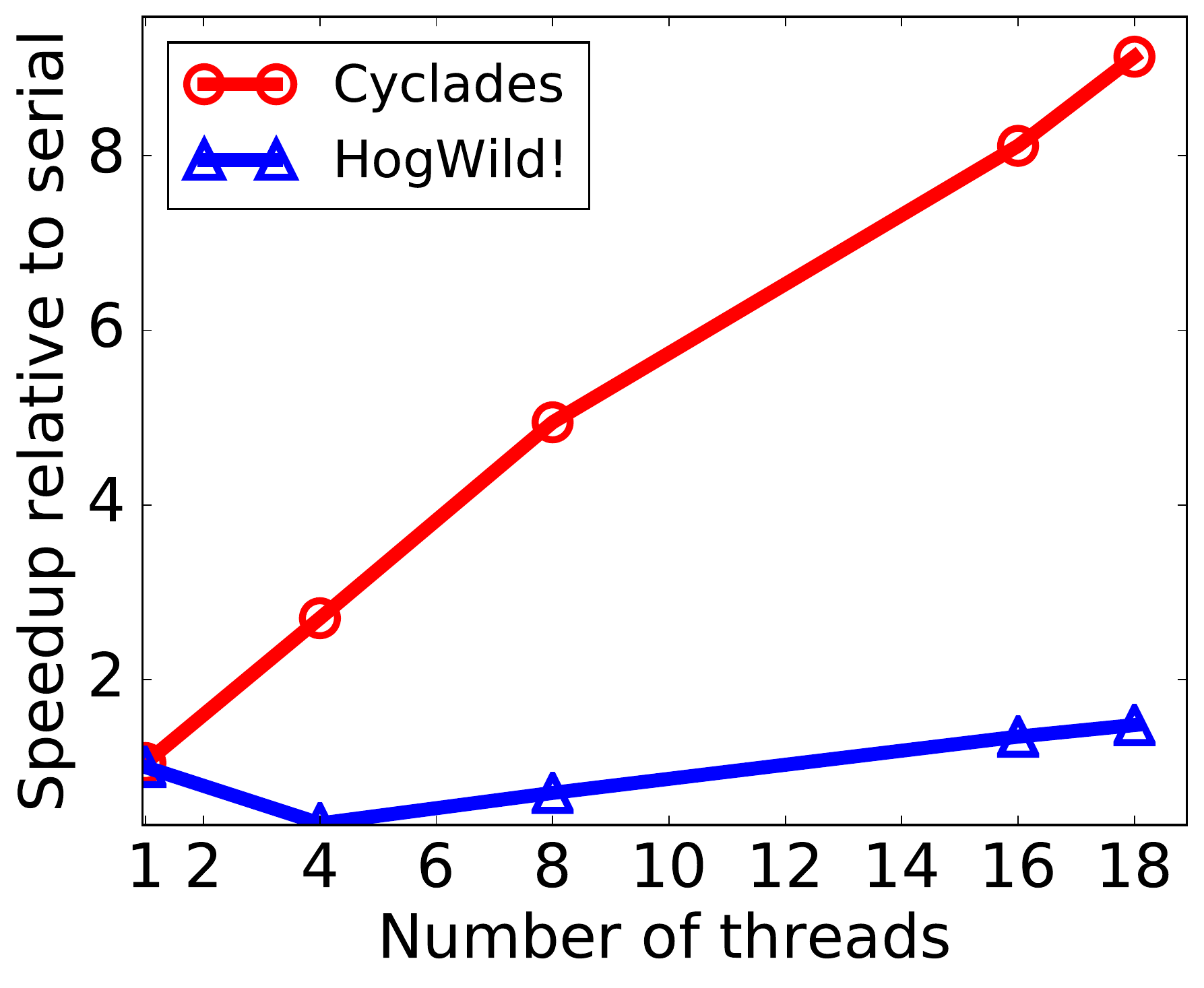}}
    \subfigure[Graph Eig., NH2010, SVRG]{\label{appfig:ge_nh2010_may_svrg:speedupsgd} \includegraphics[width=0.24\textwidth]{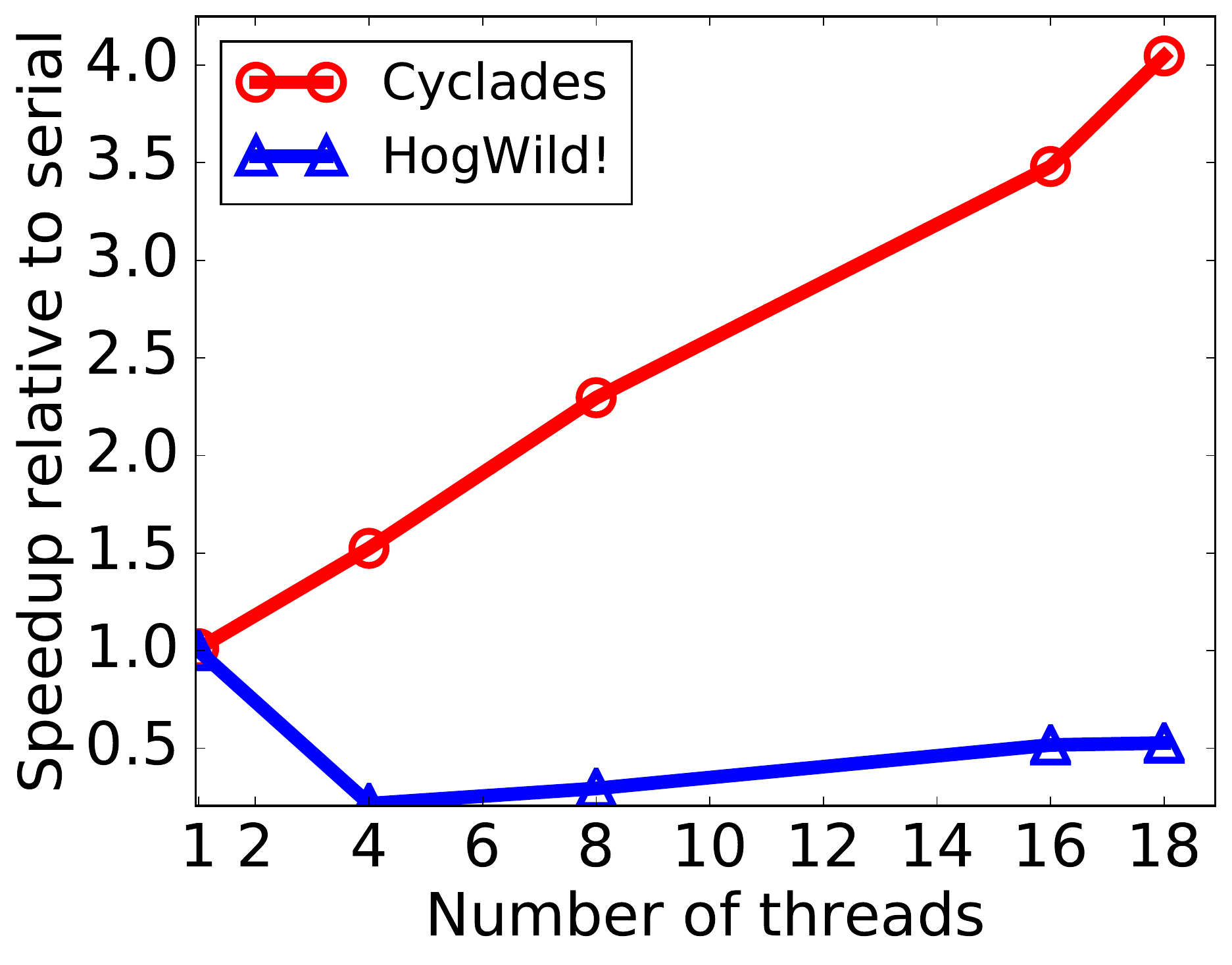}}
    \subfigure[Graph Eig., DBLP, SVRG]{\label{appfig:ge_dblp_may_svrg:speedupsgd} \includegraphics[width=0.24\textwidth]{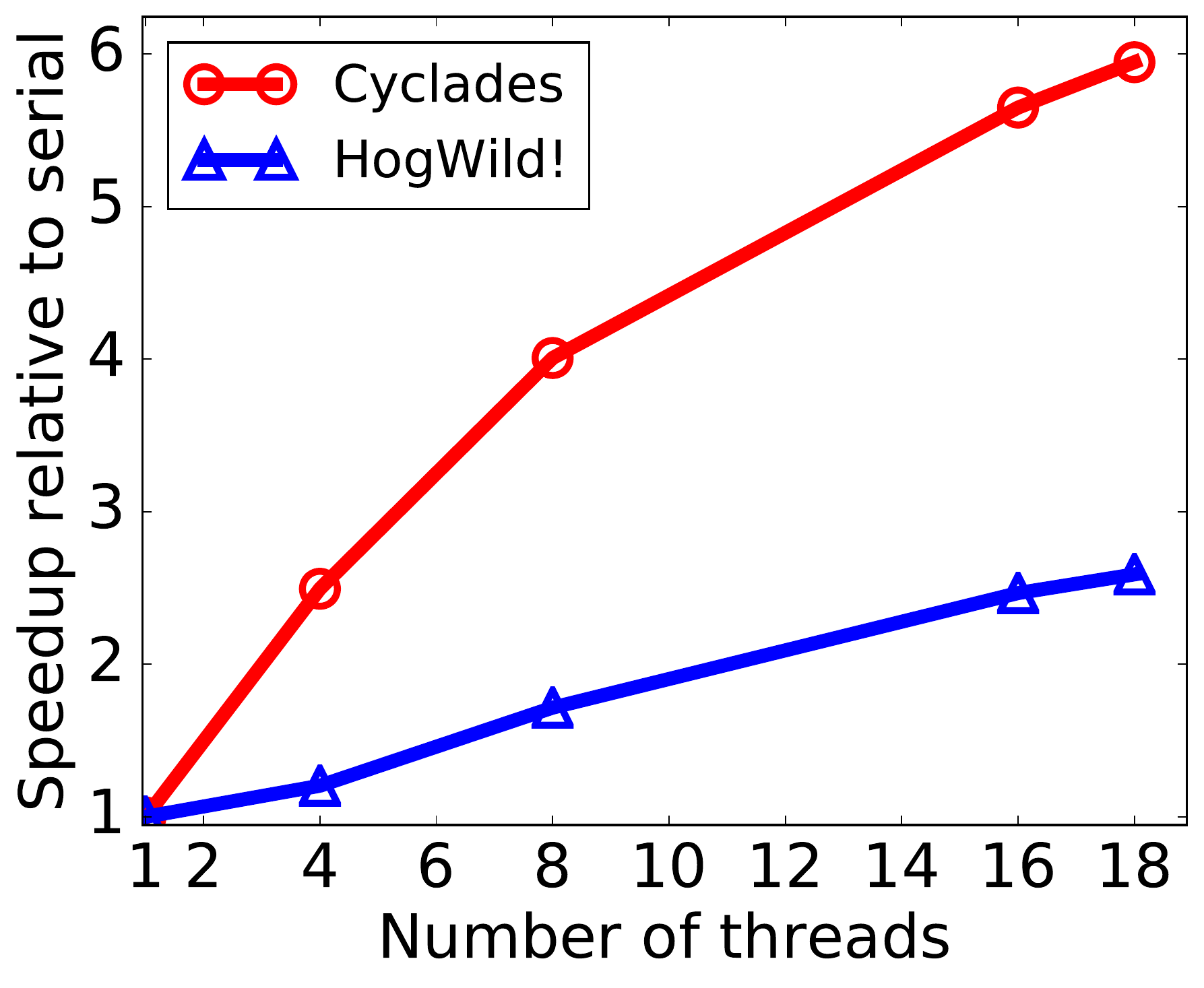}}\\

    \subfigure[Mat. Comp., 10M, $\ell_2$-SGD]{\label{appfig:mc_10m_may_rsg:speedupsgd} \includegraphics[width=0.24\textwidth]{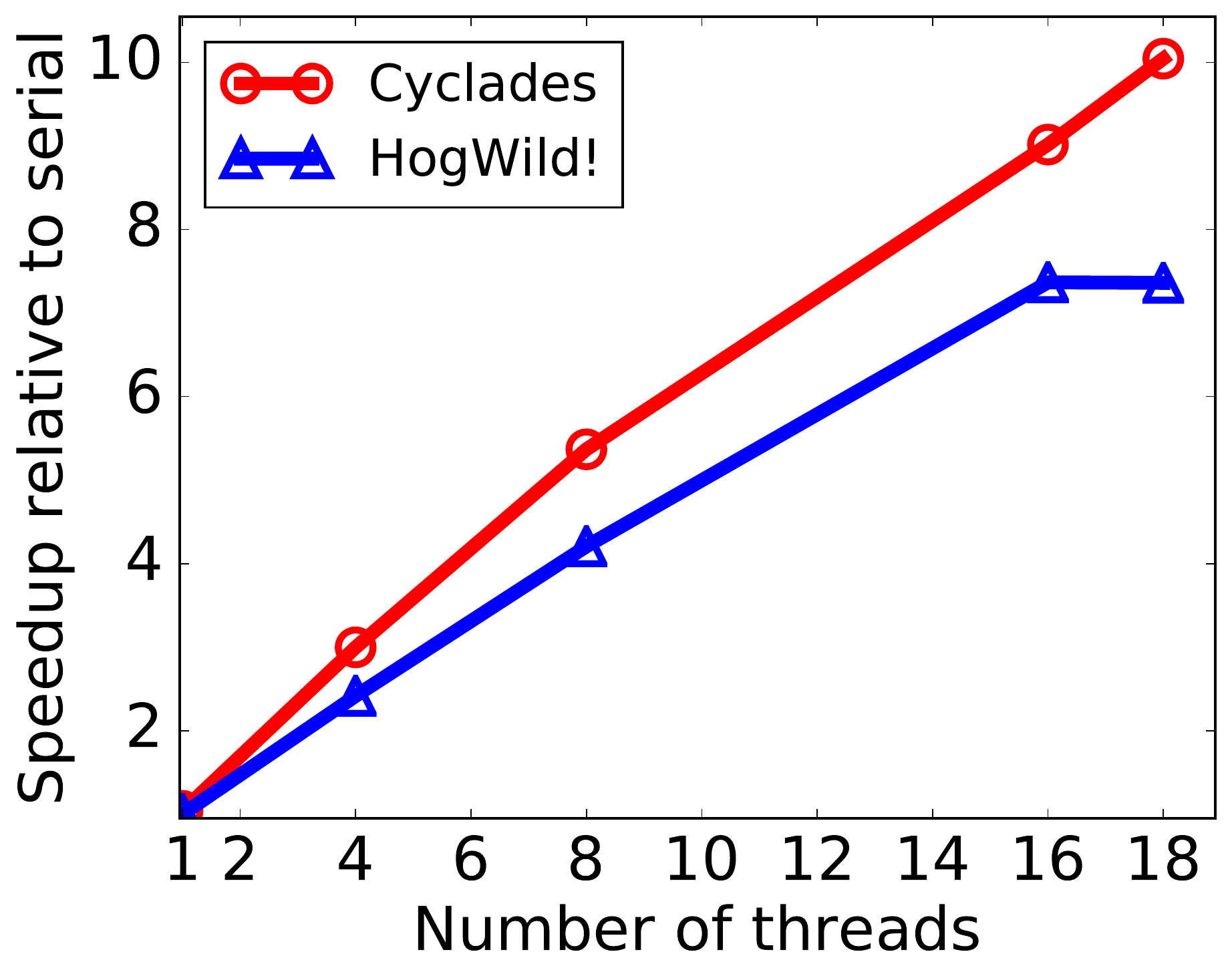}}
    \subfigure[Mat. Comp., 10M, SGD]{\label{appfig:mc_10m_may_sgd:speedupsgd} \includegraphics[width=0.24\textwidth]{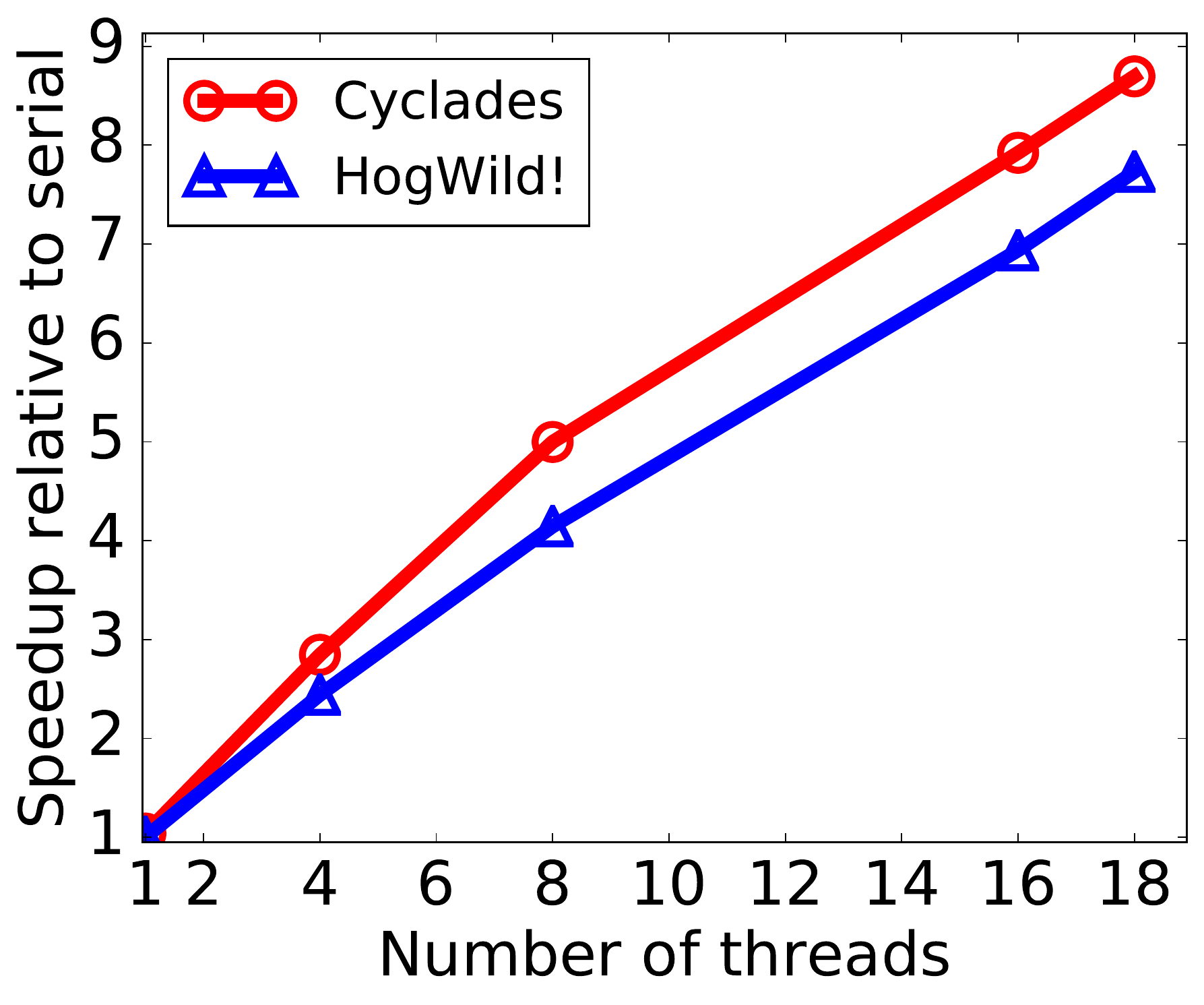}}
    \subfigure[Word2Vec, EN-Wiki, SGD]{\label{appfig:we_ewk_may_sgd:speedupsgd} \includegraphics[width=0.24\textwidth]{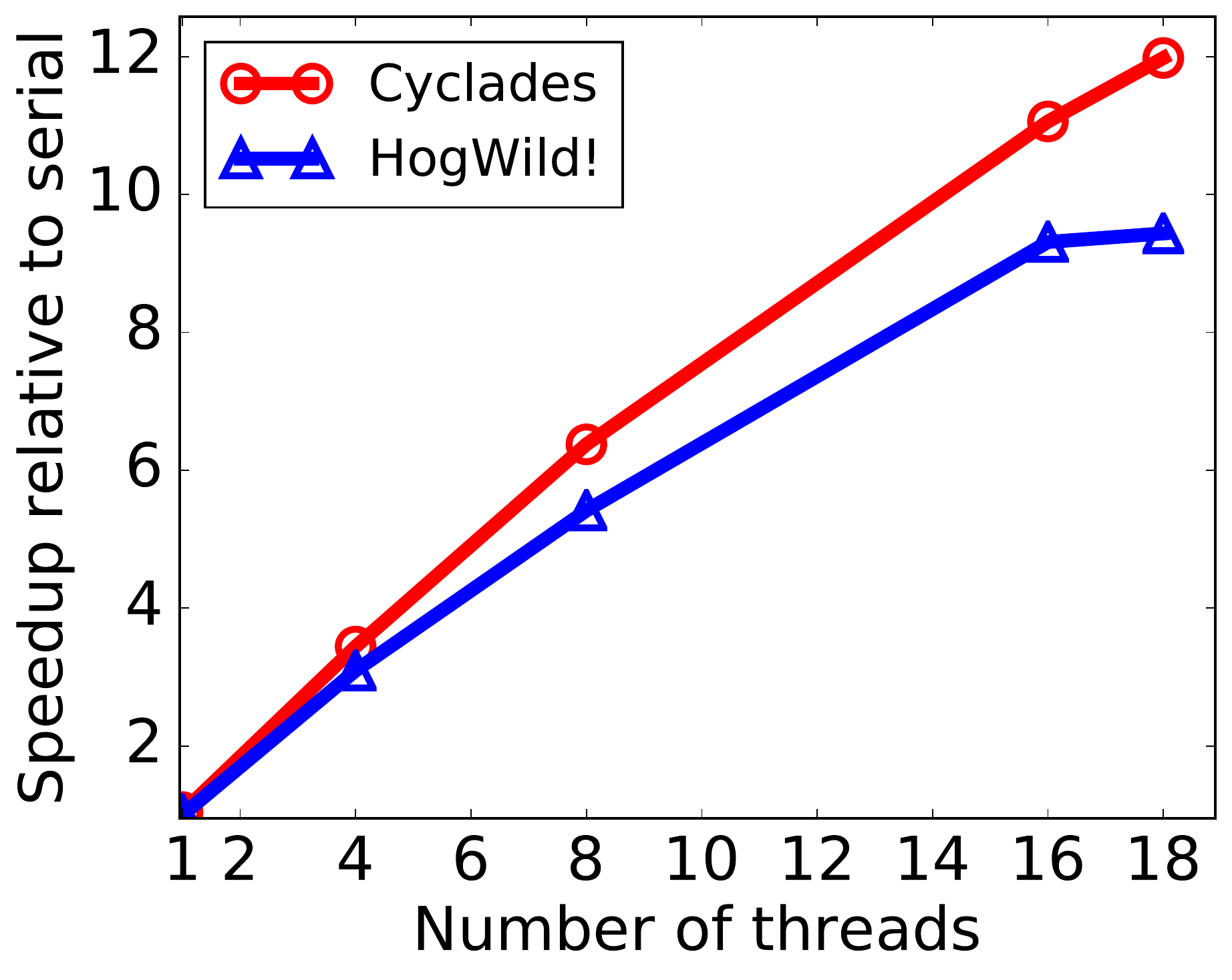}}\\
\end{center}
  \caption{Speedup of \cyc{} and \hog{} on various problems, using 1, 4, 8, 16 threads, in terms of running time for stochastic updates.
  }
  \vspace{0.1cm}
  \hrule
  \label{appfig:expt_speedupsgd}
\end{figure}


\newpage

\begin{figure}[H]
  \vspace{0.1cm}
  \hrule
\vspace{0.1cm}
\begin{center}
\includegraphics[width=0.5\textwidth]{images/matplotlib_plots/time_loss_legend_48_3}\\   
    \subfigure[0.016\%]{\label{appfig:url_0016:converge} \includegraphics[width=0.24\textwidth]{images/matplotlib_plots/cyc_text_classification_url_01__sgd_200_400_0_converge}}
    \subfigure[0.028\%]{\label{appfig:url_0028:converge} \includegraphics[width=0.24\textwidth]{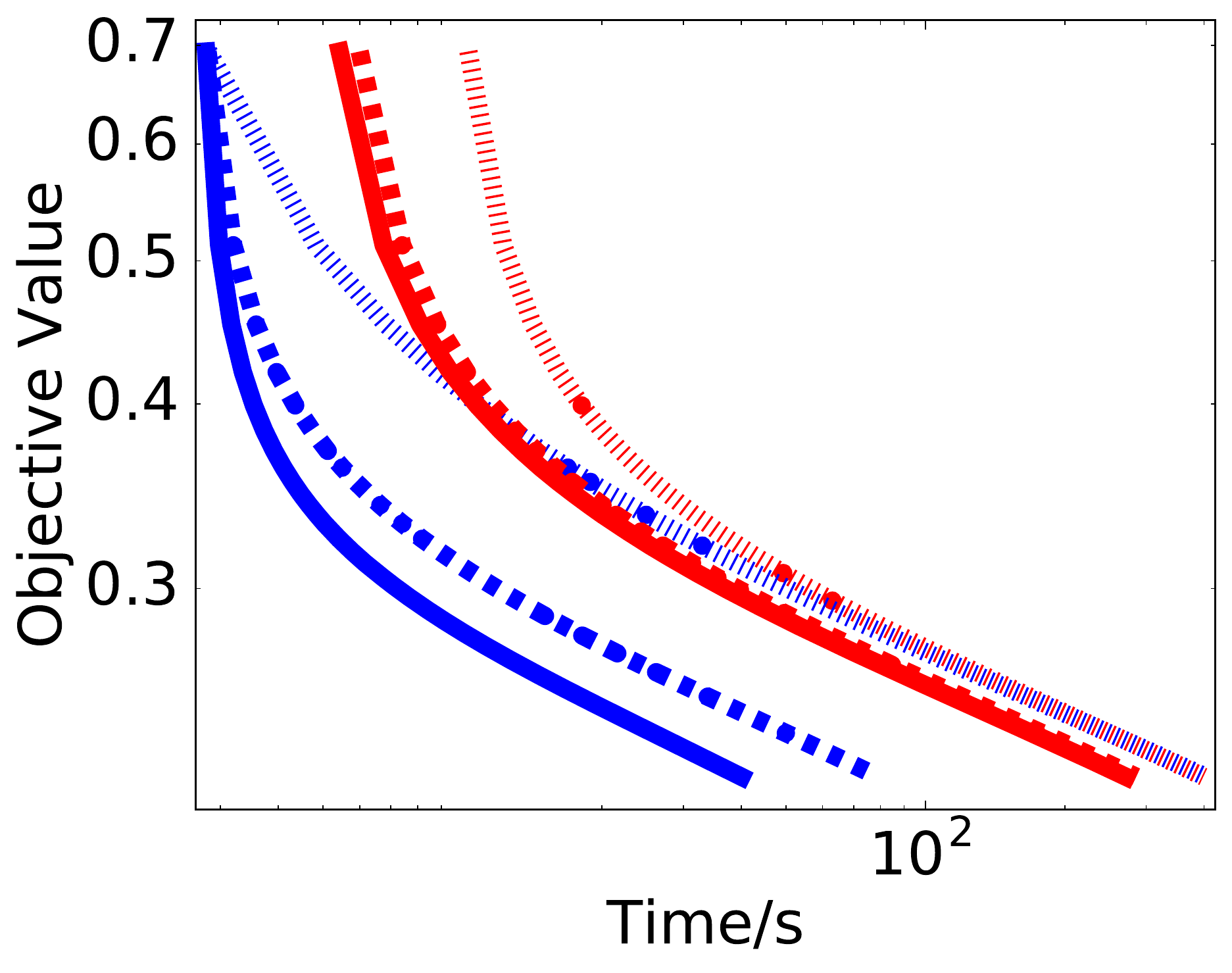}}
    \subfigure[0.047\%]{\label{appfig:url_0047:converge} \includegraphics[width=0.24\textwidth]{images/matplotlib_plots/cyc_text_classification_url_0025__sgd_200_400_0_converge}}
    \subfigure[0.048\%]{\label{appfig:url_0048:converge} \includegraphics[width=0.24\textwidth]{images/matplotlib_plots/cyc_text_classification_url_001__sgd_200_400_0_converge}}
    \end{center}
    \vspace{-0.7cm}
  \caption{Convergence of \cyc{} and \hog{} on the malicious URL detection problem, using 1, 4, 8, 16 threads, in terms of overall running time, for different percentage of features filtered.
  }
  \label{appfig:expt_url_converge}
    \vspace{0.1cm}
  \hrule
\vspace{0.1cm}
\end{figure}
\begin{figure}[H]
     \vspace{-0.7cm}
\begin{center}
	\includegraphics[width=0.5\textwidth]{images/matplotlib_plots/time_loss_legend_48_3}\\
    \subfigure[0.016\%]{\label{appfig:url_0016:convergegd} \includegraphics[width=0.24\textwidth]{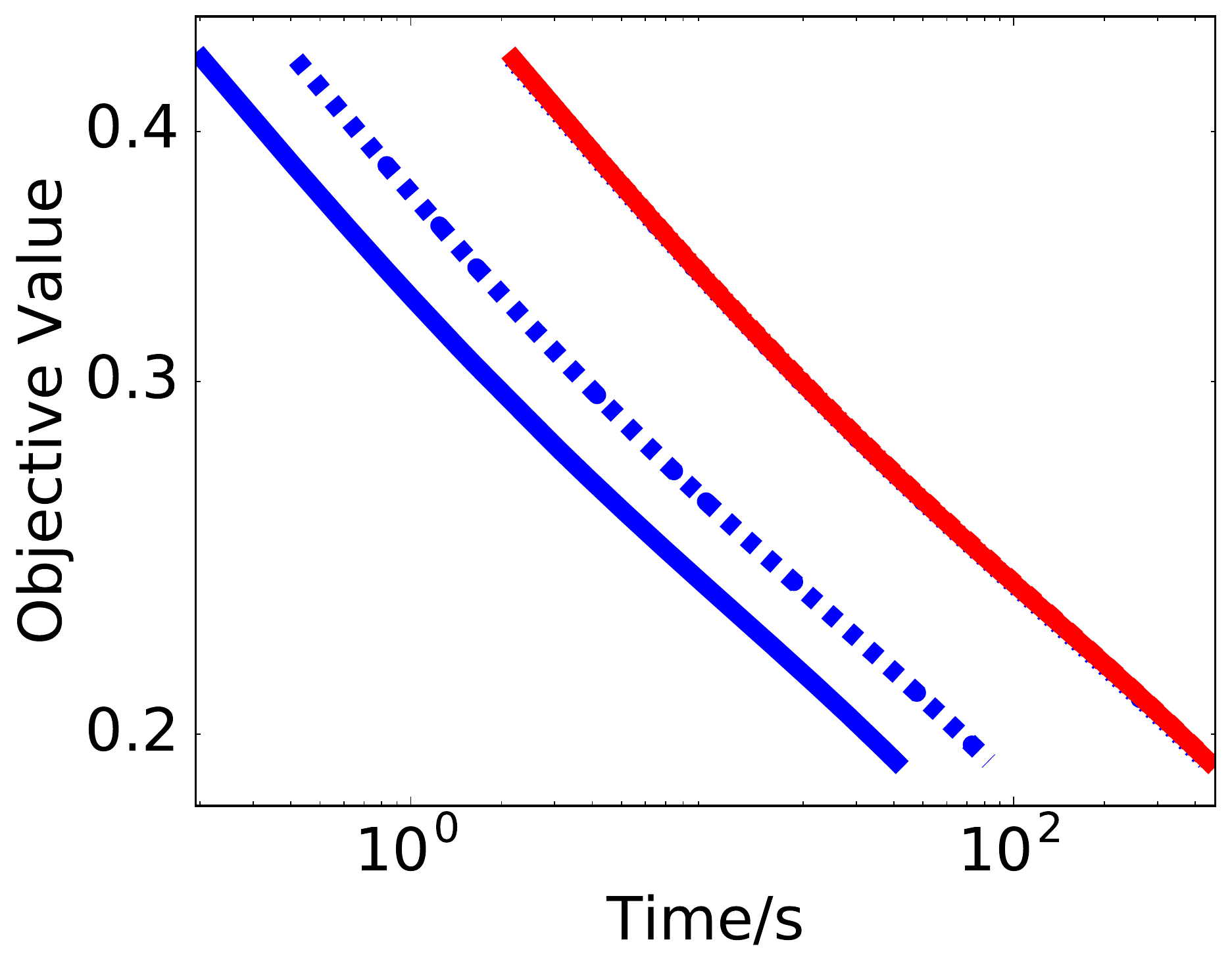}}
    \subfigure[0.028\%]{\label{appfig:url_0028:convergegd} \includegraphics[width=0.24\textwidth]{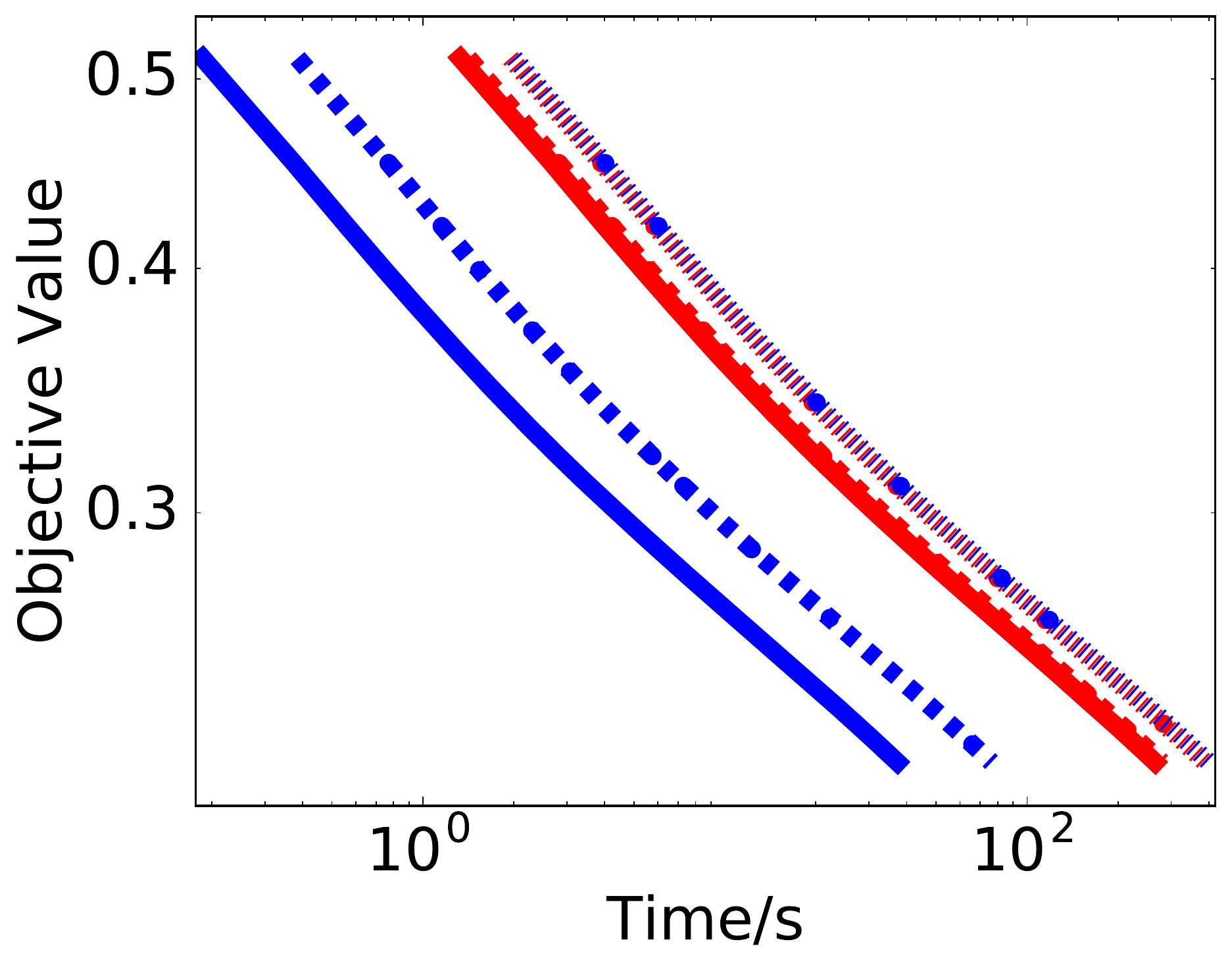}}
    \subfigure[0.047\%]{\label{appfig:url_0047:convergegd} \includegraphics[width=0.24\textwidth]{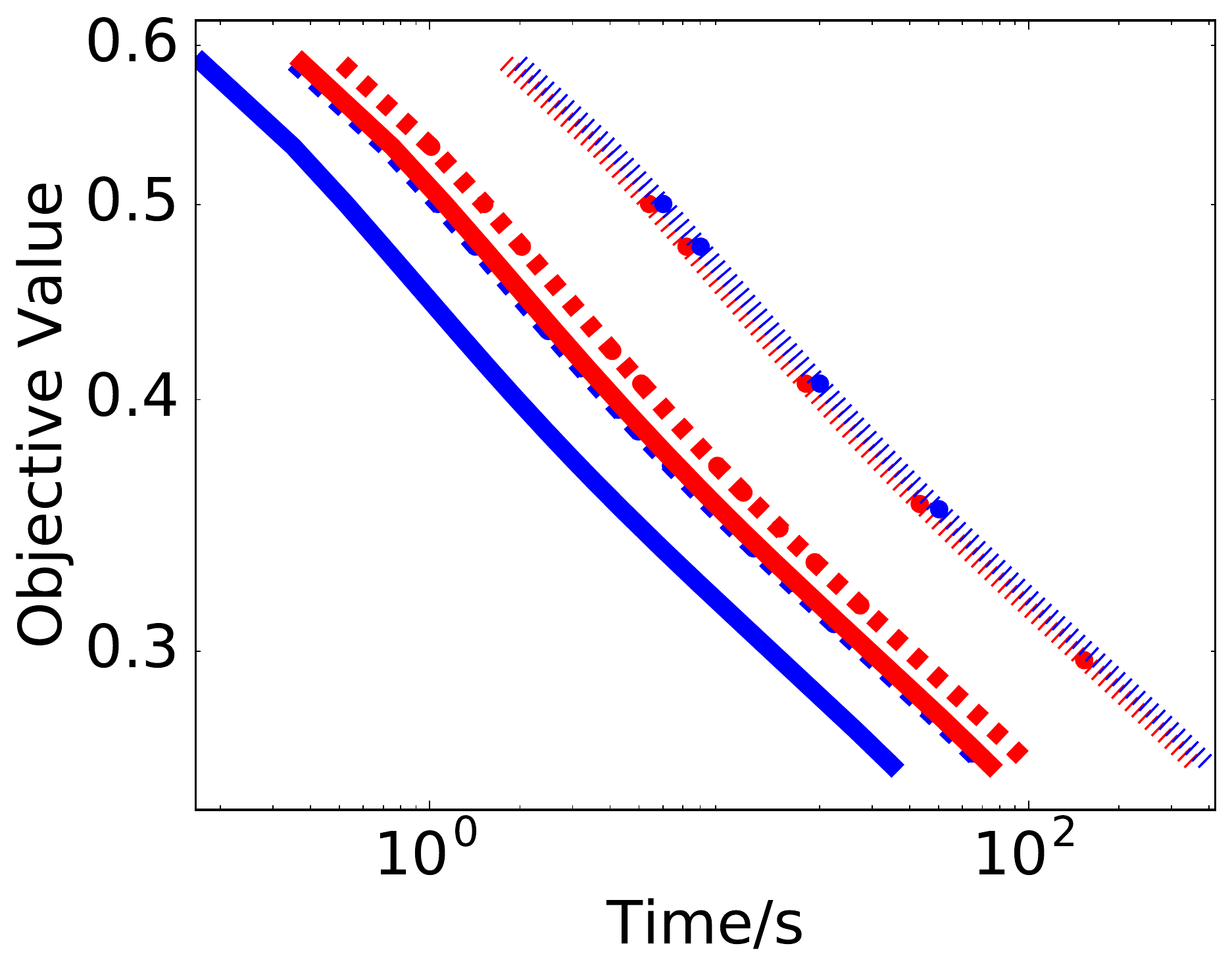}}
    \subfigure[0.048\%]{\label{appfig:url_0048:convergegd} \includegraphics[width=0.24\textwidth]{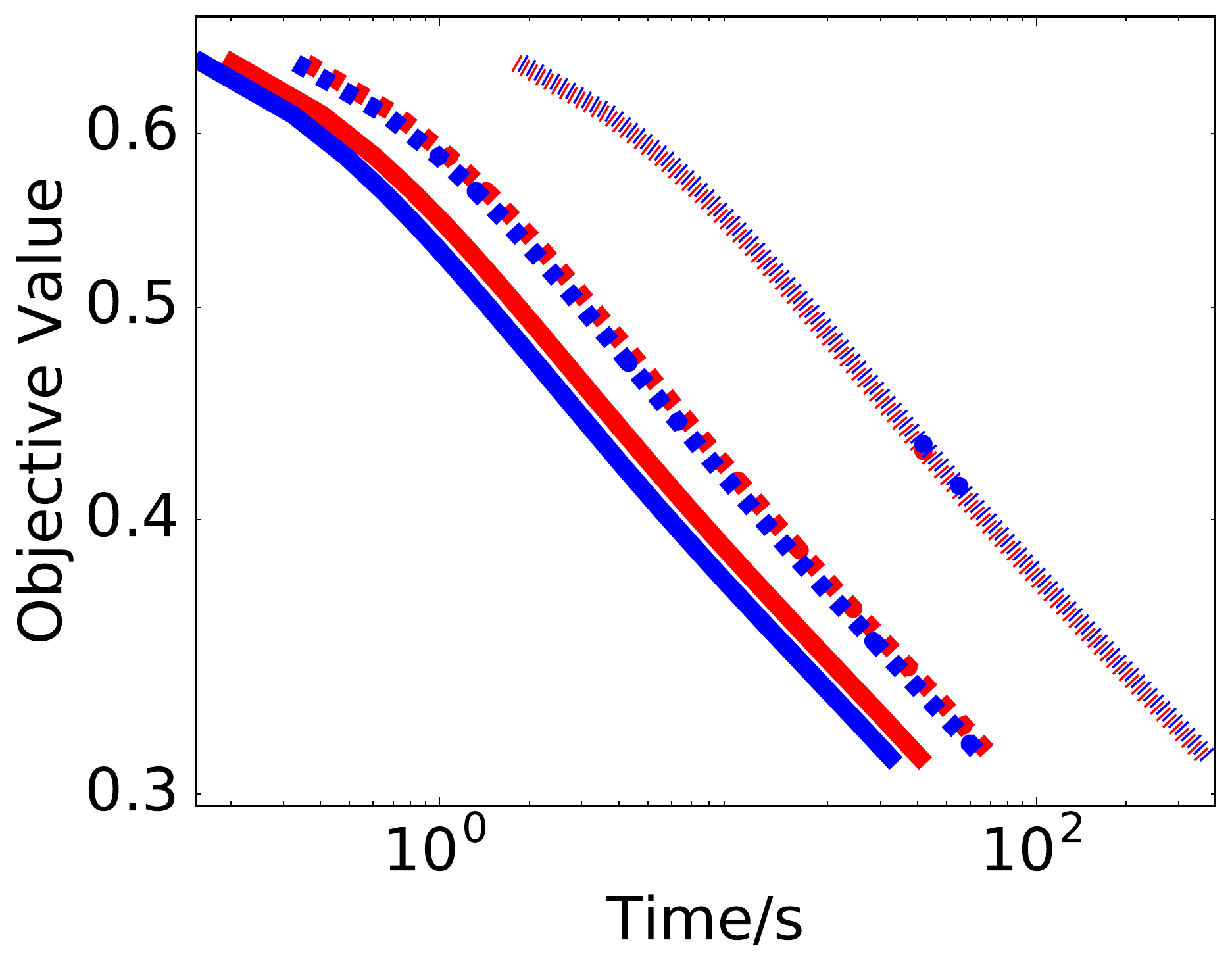}}
\end{center}
\vspace{-0.7cm}
  \caption{Convergence of \cyc{} and \hog{} on the malicious URL detection problem, in terms of running time for stochastic updates, for different percentage of features filtered.
  }
  \label{appfig:expt_url_convergegd}
\end{figure}

\vspace{-0.7cm}
\begin{figure}[H]
    \vspace{0.1cm}
  \hrule
\vspace{0.1cm}
  \centering
    \subfigure[0.016\%]{\label{appfig:url_0016:speedups} \includegraphics[width=0.24\textwidth]{images/matplotlib_plots/cyc_text_classification_url_01__sgd_200_400_0_speedup}}
    \subfigure[0.028\%]{\label{appfig:url_0028:speedups} \includegraphics[width=0.24\textwidth]{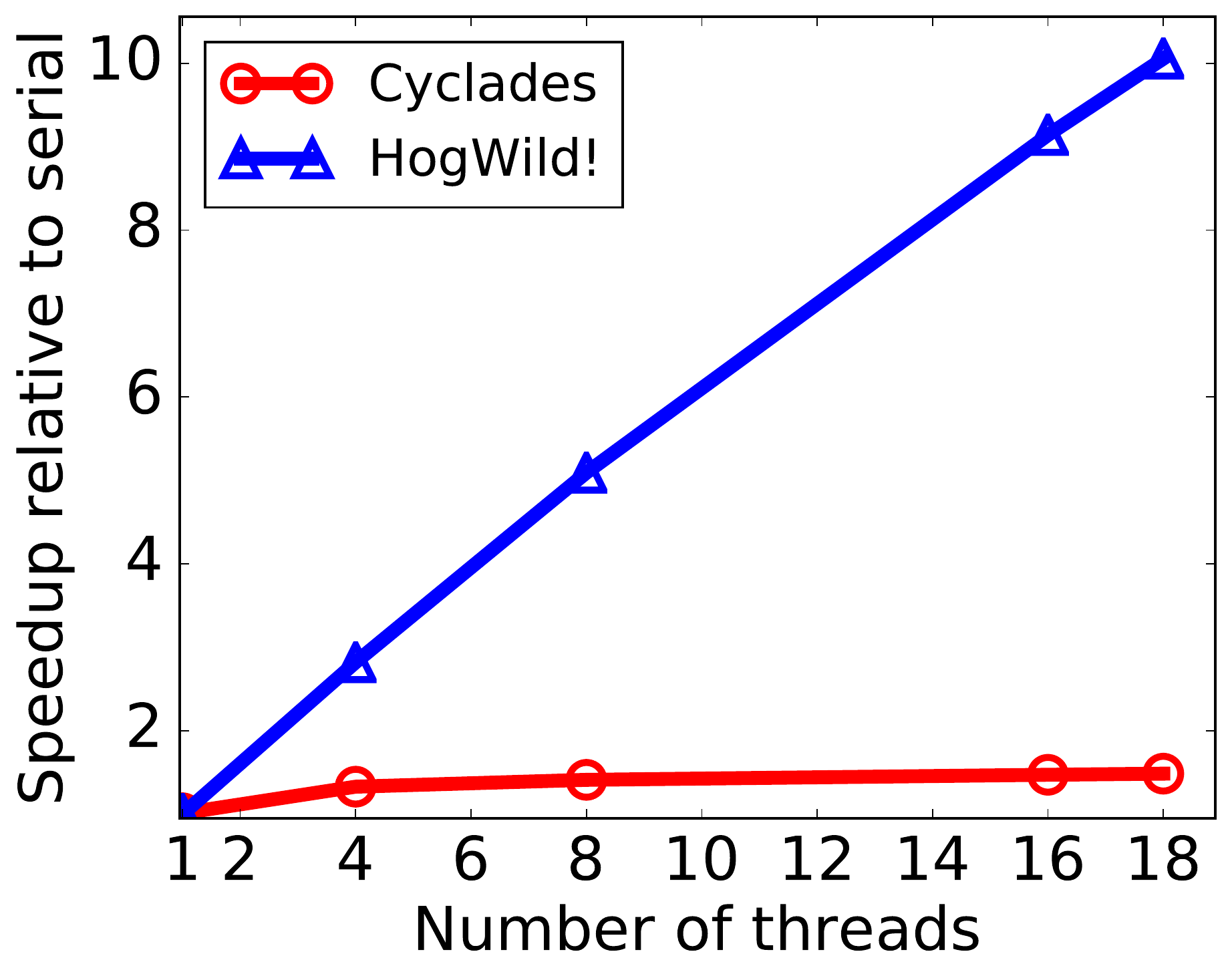}}
    \subfigure[0.047\%]{\label{appfig:url_0047:speedups} \includegraphics[width=0.24\textwidth]{images/matplotlib_plots/cyc_text_classification_url_0025__sgd_200_400_0_speedup}}
    \subfigure[0.048\%]{\label{appfig:url_0048:speedups} \includegraphics[width=0.24\textwidth]{images/matplotlib_plots/cyc_text_classification_url_001__sgd_200_400_0_speedup}}
    \vspace{-0.6cm}
  \caption{Speedup of \cyc{} and \hog{} on the malicious URL detection problem, using 1, 4, 8, 16 threads, in terms of overall running time, for different percentage of features filtered.
  }
  \label{appfig:expt_url_speedups}
       \vspace{0.1cm}
  \hrule
\vspace{0.1cm}
\end{figure}
\begin{figure}[H]
\vspace{-0.7cm}
\begin{center}
    \subfigure[0.016\%]{\label{appfig:url_0016:speedupsgd} \includegraphics[width=0.24\textwidth]{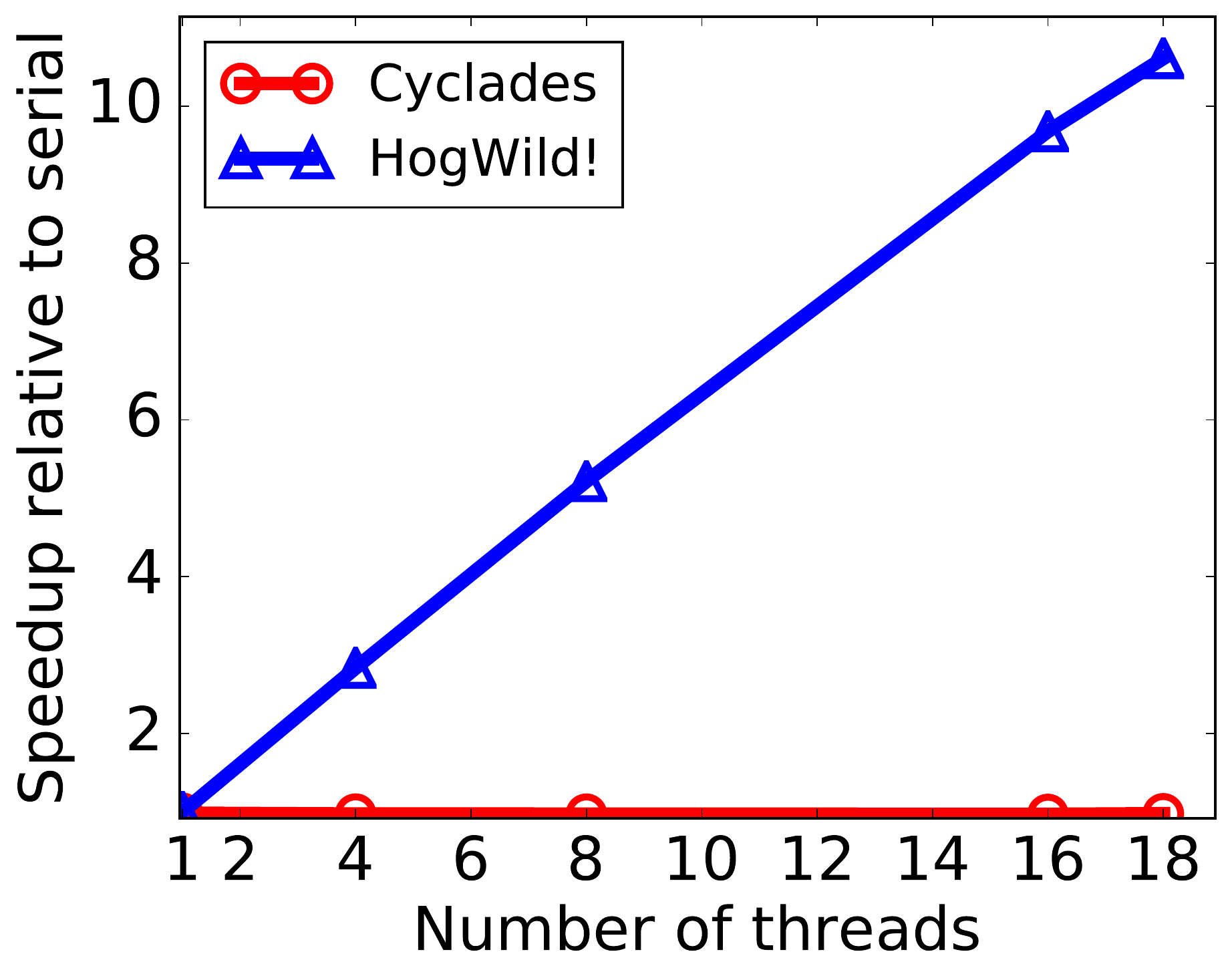}}
    \subfigure[0.028\%]{\label{appfig:url_0028:speedupsgd} \includegraphics[width=0.24\textwidth]{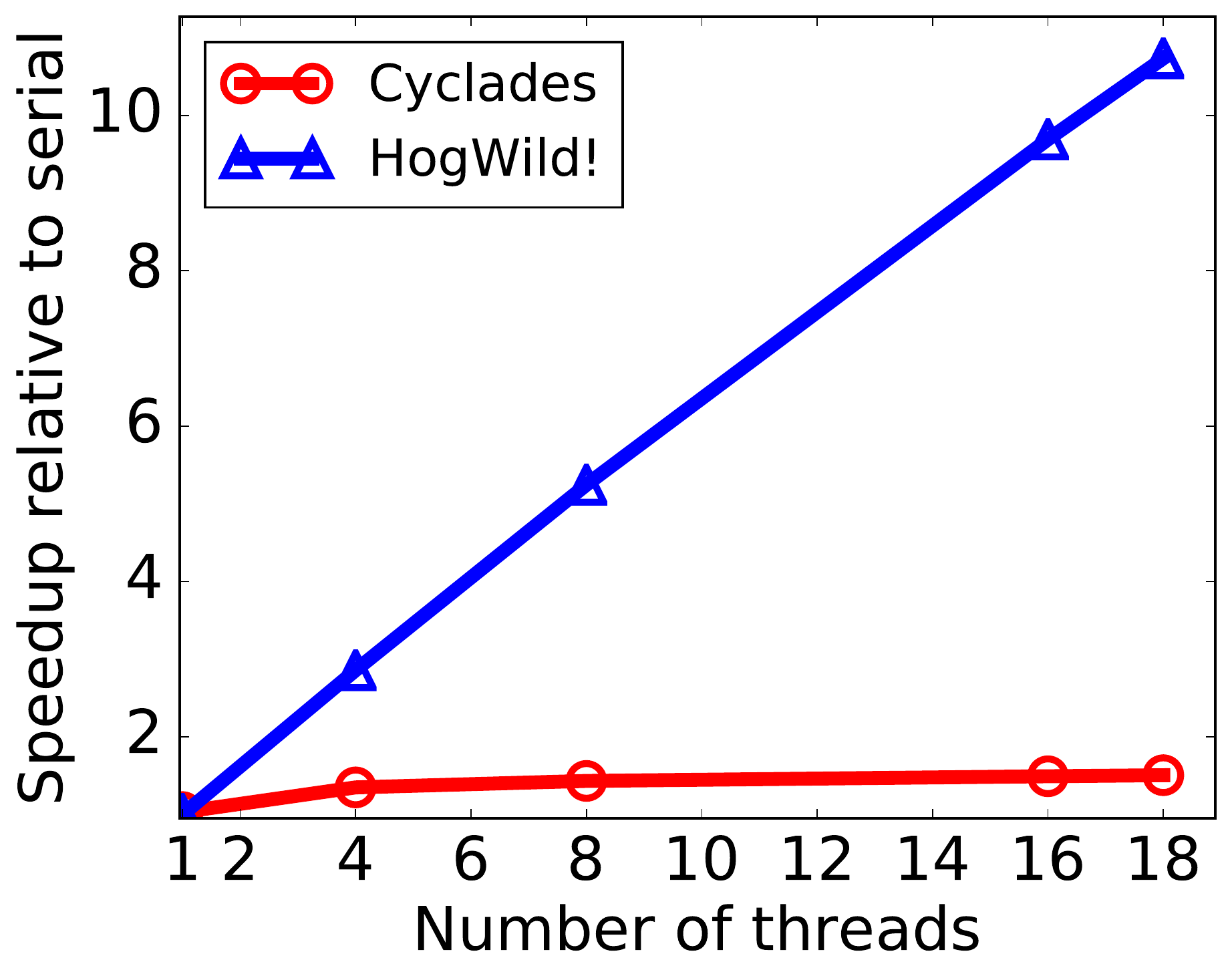}}
    \subfigure[0.047\%]{\label{appfig:url_0047:speedupsgd} \includegraphics[width=0.24\textwidth]{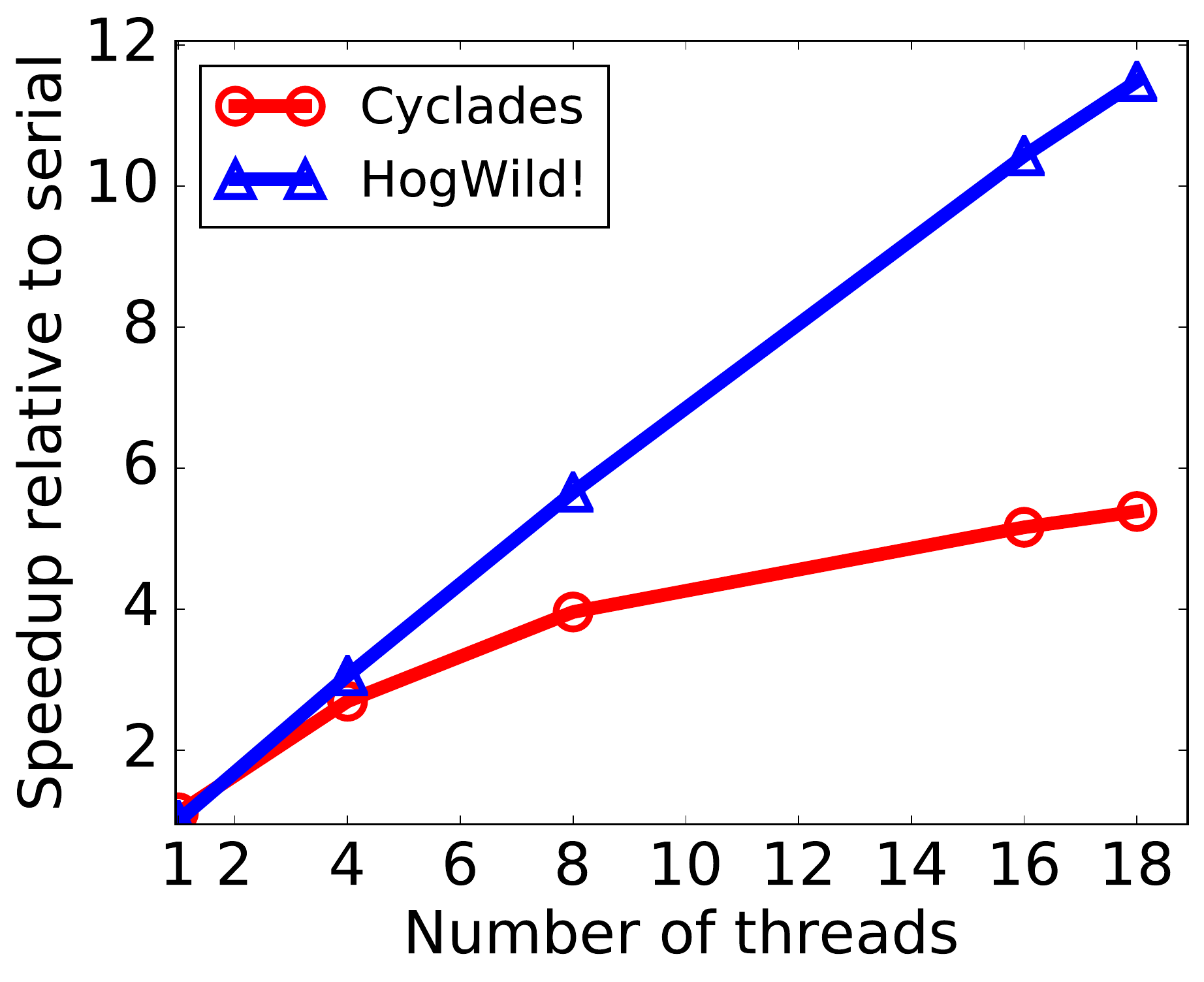}}
    \subfigure[0.048\%]{\label{appfig:url_0048:speedupsgd} \includegraphics[width=0.24\textwidth]{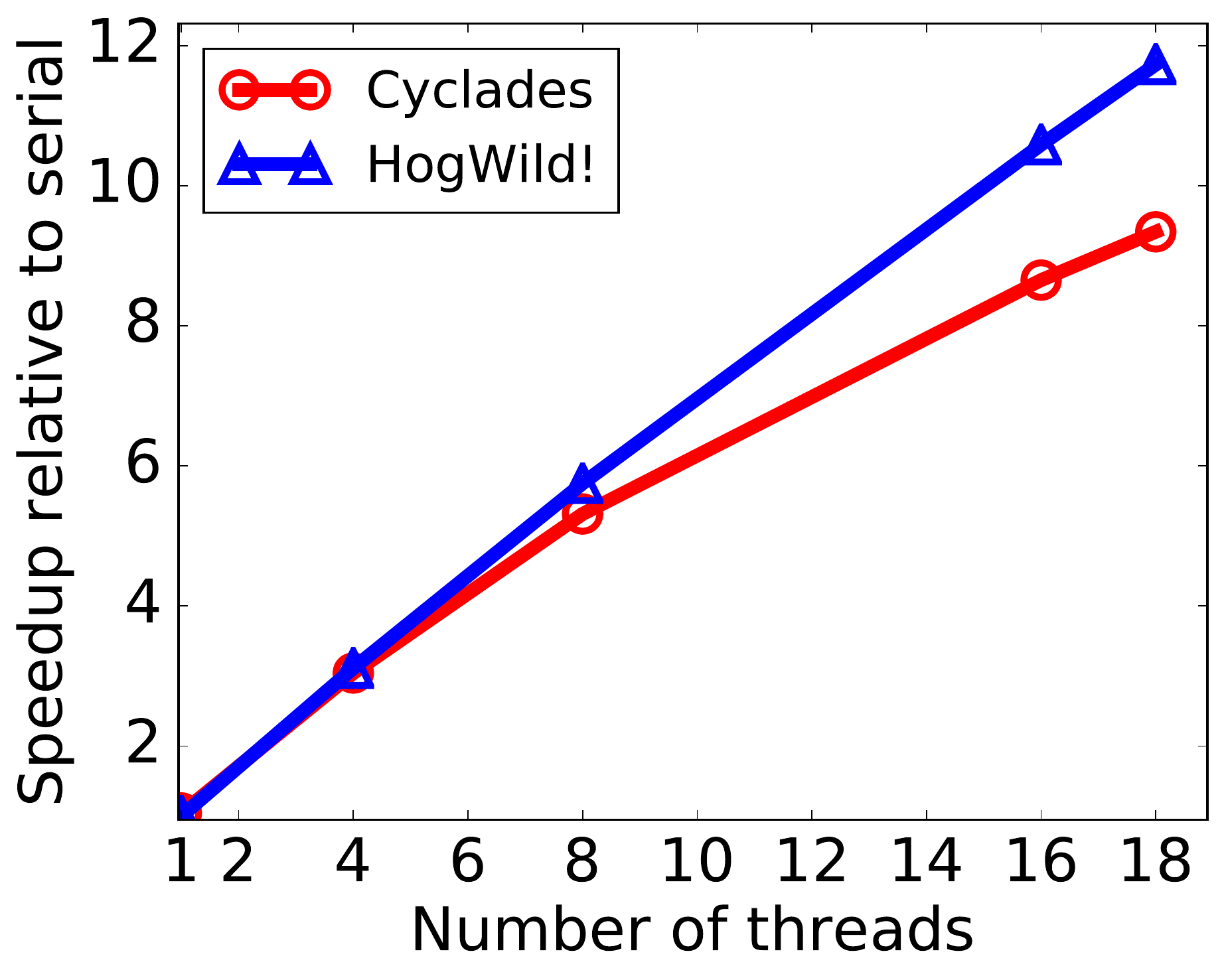}}
\end{center}
\vspace{-0.7cm}
  \caption{Speedup of \cyc{} and \hog{} on the malicious URL detection problem, in terms of running time for stochastic updates, for different percentage of features filtered.
  }
  \label{appfig:expt_url_speedupsgd}
    \vspace{0.1cm}
  \hrule
\vspace{0.1cm}
\end{figure}


\newpage

\begin{center}
\begin{figure}[H]
\vspace{0.1cm}
  \centering
    \subfigure[Least squares, NH2010, SAGA]{\label{appfig:ls_nh2010_may_saga:hdiverge} \includegraphics[width=0.45\textwidth]{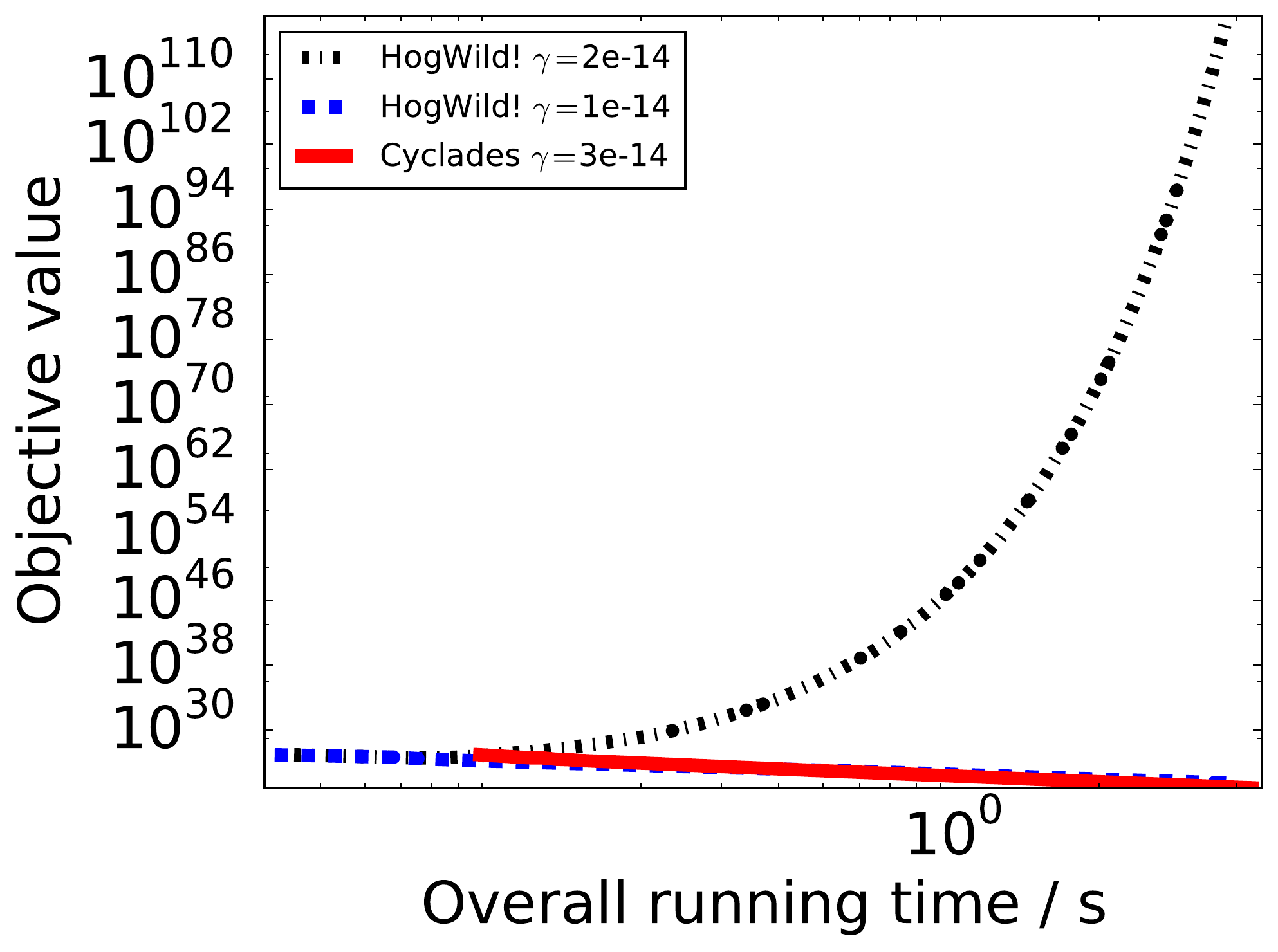}}
    \subfigure[Least squares, DBLP, SAGA]{\label{appfig:ls_DBLP_may_saga:hdiverge} \includegraphics[width=0.45\textwidth]{images/matplotlib_plots/cyc_least_squares_dblp__saga_hdiverge}}
  \caption{
  Convergence of \cyc{} and \hog{} on least squares using SAGA, with 16 threads, on the NH2010 and DBLP datasets.
  \cyc{} was able to converge using larger stepsizes, but \hog{} often diverged with the same large stepsize.
  Thus, we were only able to use smaller stepsizes for \hog{} in the multi-threaded setting.
  }
  \label{appfig:expt_hdiverge}
\end{figure}
\end{center}

\end{document}